\documentclass[jair,twoside,11pt,theapa]{article}
\usepackage{graphicx}
\usepackage{latexsym}
\usepackage{float}
\usepackage{helvet}
\usepackage{courier}
\usepackage{blindtext, graphicx}
\usepackage[utf8]{inputenc}
\usepackage{listings}
\usepackage{graphicx}
\usepackage{subcaption}
\usepackage{amsthm}
\usepackage{hyperref}
\usepackage{enumitem}
\usepackage{algorithm, algpseudocode}
\usepackage{caption}
\usepackage{amsmath}
\usepackage{lipsum}
\usepackage[obeyDraft]{todonotes}
\usepackage{todonotes}
\usepackage{refstyle}
\usepackage{xfrac}

\usepackage{booktabs}
\usepackage{multirow}
\usepackage{colortbl}

\usepackage{pdflscape}
\usepackage{afterpage}
\usepackage{capt-of}

\usepackage{jair, theapa, rawfonts}

\theoremstyle{plain}
\newtheorem{definition}{Definition}
\newtheorem{example}{Example}
\newtheorem{proposition}{Proposition}
\newtheorem{corollary}{Corollary}
\newtheorem{theorem}{Theorem}

\newcommand\pred[1]{\texttt{\selectfont{#1}}}

\def\gets{:=}

\newcommand{\timeout}{{\it Timeout}}

\newcommand{\idest}{{i.e.}}
\newcommand{\exemp}{{e.g.}}

\newcommand{\etal}{{et al.}}

\usepackage{listings}
\ShortHeadings{Landmark-Based Approaches for Goal Recognition as Planning}
{Pereira, Oren \& Meneguzzi}
\firstpageno{1}

\begin{document}

\title{Landmark-Based Approaches for \\ Goal Recognition as Planning}

\author{\name Ramon Fraga Pereira \email ramon.pereira@acad.pucrs.br \\
       \addr Pontifical Catholic University of Rio Grande do Sul \\
       School of Technology \\
       Porto Alegre, RS, Brazil
       \AND
       \name Nir Oren \email n.oren@abdn.ac.uk \\
       \addr University of Aberdeen\\
       Department of Computing Science\\
       Aberdeen, Scotland
       \AND
       \name Felipe Meneguzzi \email felipe.meneguzzi@pucrs.br \\
       \addr Pontifical Catholic University of Rio Grande do Sul \\
       School of Technology\\
       Porto Alegre, RS, Brazil}


\maketitle

\begin{abstract}

The task of recognizing goals and plans from missing and full observations can be done efficiently by using automated planning techniques. 
In many applications, it is important to recognize goals and plans not only accurately, but also quickly. 
To address this challenge, we develop novel goal recognition approaches based on planning techniques that rely on planning landmarks. 
In automated planning, landmarks are properties (or actions) that cannot be avoided to achieve a goal.
We show the applicability of a number of planning techniques with an emphasis on landmarks for goal and plan recognition tasks in two settings: (1) we use the concept of landmarks to develop goal recognition heuristics; and (2) we develop a landmark-based filtering method to refine existing planning-based goal and plan recognition approaches.
These recognition approaches are empirically evaluated in experiments over several classical planning domains. 
We show that our goal recognition approaches yield not only accuracy comparable to (and often higher than) other state-of-the-art techniques, but also substantially faster recognition time over such techniques. 

\end{abstract}

\section{Introduction}

As more computer systems require reasoning about what agents (both human and artificial) other than themselves are doing, the ability to accurately and efficiently recognize goals and plans from agent behavior becomes increasingly important. Goal and plan recognition is the task of recognizing goals and plans based on often incomplete observations that include actions executed by agents and properties of agent behavior in an environment~\cite{ActivityIntentPlanRecogition_Book2014}. 
Most goal and plan recognition approaches~\cite{Geib_ProbabilisticPlanRecognition_MOO2005,AvrahamiZilberbrandK_IJCAI2005,PRGeib_AI_2009,PR_Mirsky_2016,Mirsky2017plan} employ plan libraries to represent agent behavior, \idest, a plan library with plans for achieving goals, 
resulting in approaches to recognize plans that are analogous to parsing. 
Recent work~\cite{RamirezG_IJCAI2009,RamirezG_AAAI2010,PattisonGoalRecognition_2010,GoalRecognitionDesign_Keren2014,NASA_GoalRecognition_IJCAI2015,Sohrabi_IJCAI2016,PereiraMeneguzzi_ECAI2016,RamonNirMeneguzzi_AAAI2017} use a planning domain definition (a domain theory) to represent potential agent behavior, bringing goal and plan recognition closer to planning algorithms. 
These approaches allow techniques used in planning algorithms to be employed for recognizing goals and plans requiring less domain information. 
Recognizing goals and plans are important in applications for monitoring and anticipating agent behavior in an environment, including crime detection and prevention~\cite{GeibPlanRecognitionIntrusionDect_DARPA2001}, monitoring activities in elderly-care~\cite{ProblemsWithElderCare_AAAI2002}, recognizing plans in educational environments~\cite{UzanDSG_PR_2015} and exploratory domains~\cite{PR_EXP_Mirsky2017}, and traffic monitoring~\cite{WellmanTraffic_2013}, among others~\cite{GeibPlanRecognitionIntrusionDect_DARPA2001,Granada2017_PAIR,Mirsky_UISP17}.

We develop recognition approaches that are based on planning techniques (without pre-defined static plan libraries) that rely on planning landmarks~\cite{Hoffmann2004_OrderedLandmarks}, namely, landmark-based approaches for goal recognition.
In automated planning, landmarks are properties (or actions) that every plan must satisfy (or execute) at some point in every plan execution to achieve a goal. 
Whereas in planning algorithms landmarks are used to focus search, in this work, landmarks allow our recognition approaches to reason about what cannot be avoided for achieving goals, substantially speeding up recognition time. 
Thus, we provide novel contributions to efficiently solve goal recognition problems, as follows.
First, we provide two contributions for goal recognition techniques.  
We develop two novel recognition heuristics that rely on landmarks and obviate the need to execute a planner multiple times yielding substantial runtime gains. 
Our initial heuristic estimates goal completion by considering the ratio between achieved and extracted landmarks of a candidate goal. 
We expand this heuristic to use a \textit{landmark uniqueness value}, representing the information value of the landmark for some specific candidate goal when compared to landmarks for all candidate goals. 
Second, we also develop a filtering method that rules out candidate goals by estimating how many landmarks required by every goal in the set of candidate goals have been reached within a sequence of observations. 
This filtering method can be applied to other planning-based goal and plan recognition approaches, such as the approaches from Ram{\'{\i}}rez and Geffner~(\citeyear{RamirezG_IJCAI2009,RamirezG_AAAI2010}) (with a probabilistic ranking), as well as from Sohrabi et al.~(\citeyear{Sohrabi_IJCAI2016}). 

Our use of landmarks to drive goal recognition stems from properties of landmarks in classical planning. 
First, they are necessary conditions to achieving goals, and thus provide very strong evidence that certain observations are tied to specific goals.
Second, although their computation is, in theory, expensive, in practice, we can efficiently compute very informative sets of ordered landmarks, and critically, only once per goal recognition problem, resulting in a very efficient overall algorithm. 

We prove key properties of our recognition heuristics and their use as a filtering mechanism, and evaluate empirically our approaches using a set of well-known domains from the International Planning Competition (IPC), as well as a number of domains we developed specifically to measure the scalability of goal and plan recognition algorithms. 
For all domains, we evaluate the algorithms using datasets with varying degrees of observability (missing observations) and noise (spurious observations). 
We compare our heuristics for goal recognition against the current state-of-the-art~\cite{RamirezG_IJCAI2009,RamirezG_AAAI2010,NASA_GoalRecognition_IJCAI2015,Sohrabi_IJCAI2016} by using a dataset developed by 
Ram{\'{\i}}rez and Geffner~(\citeyear{RamirezG_IJCAI2009,RamirezG_AAAI2010}), and a new dataset we generated for other planning domains with larger and more complex problems, as well as problems with missing and noisy observations. 
Experiments show that our recognition heuristics are substantially faster and more accurate than the state-of-the-art for datasets that contain several domains and problems where recognizing the intended goal is non-trivial. 

The remainder of this article is organized as follows.
Section~\ref{section:Background} provides background on planning, domain-independent heuristics, and landmarks. 
We proceed to describe how we extract useful information from planning domain definitions
in Section~\ref{section:InferringStructures}, which we use throughout the article. 
In Section~\ref{section:LandmarkGoalRecognitionApproaches}, we develop our goal recognition approaches using landmarks. 
We empirically evaluate our approaches in Section~\ref{section:ExperimentsAndEvaluation}, which shows the results of the experiments for our goal recognition approaches against the state-of-the-art.
In Section~\ref{section:RelatedWork}, we survey related work and compare the state-of-the-art with our approaches. 
Finally, in Section~\ref{section:Conclusions}, we conclude this article by discussing limitations, advantages and future directions of our approaches. 

\section{Background}\label{section:Background}

In this section, we review essential background on planning terminology and landmarks. 
Finally, we define the task of goal recognition over planning domain definitions.

\subsection{Planning}
\label{sec:planning}

Planning is the problem of finding a sequence of actions (\idest, plan) that achieves a particular goal from an initial state. 
In this work, we use the terminology from Ghallab~\etal~(\citeyear{AutomatedPlanning_Book2016}) to represent planning domains and problems. 
First, we define a \textit{state} in the environment by the following Definition~\ref{def:state}. 

\begin{definition} [\textbf{Predicates and State}]\label{def:state}
A predicate is denoted by an n-ary predicate symbol $p$ applied to a sequence of zero or more terms ($\tau_1$, $\tau_2$, ..., $\tau_n$) -- terms are either constants or variables. 
We refer to grounded predicates that represent logical values according to some interpretation as facts, which are divided into two types: positive and negated facts, as well as constants for truth ($\top$) and falsehood ($\bot$).
A state $S$ is a finite set of positive facts $f$ that follows the closed world assumption so that if $f \in S$, then $f$ is true in $S$. 
We assume a simple inference relation $\models$ such that $S \models f$ iff $f \in S$, $S \not\models f$ iff $f \not\in S$, and $S \models f_1 \land ... \land f_n$ iff $\{f_1, ..., f_n\} \subseteq S$.
\end{definition}

Planning domains describe the environment dynamics through operators, which use a limited first-order logic representation to define schemata for state-modification actions according to Definition~\ref{def:operator}.

\begin{definition} [\textbf{Operator and Action}]\label{def:operator}
An operator $a$ is represented by a triple $\langle$name($a$), pre($a$), eff($a$)$\rangle$: name($a$) represents the description or signature of $a$; pre($a$) describes the preconditions of $a$, a set of predicates that must exist in the current state for $a$ to be executed; eff($a$) represents the effects of $a$. 
These effects are divided into eff($a$)$^+$ (\idest, an add-list of positive predicates) and eff($a$)$^-$ (\idest, a delete-list of negated predicates).
An action is a ground operator instantiated over its free variables. 
\end{definition}


 
We say an action $a$ is applicable to a state $S$ if and only if $S \models \mathit{pre}(a)$, and generates a new state $S'$ such that $S' \gets (S \cup \mathit{eff}(a)^{-})/\mathit{eff}(a)^{+}$. 

\begin{definition}[\textbf{Planning Domain}]\label{def:planningDomain}
A planning domain definition $\Xi$ is represented by a pair $\langle \Sigma, \mathcal{A} \rangle$, which specifies the knowledge of the domain, and consists of a finite set of facts $\Sigma$ (\exemp, environment properties) and a finite set of actions $\mathcal{A}$.
\end{definition}

A \textit{planning instance}, comprises both a \textit{planning domain} and the elements of a \textit{planning problem}, describing a finite set of \textit{objects} of the environment, the \textit{initial state}, and the \textit{goal state} which an agent wishes to achieve, as formalized in Definition~\ref{def:planningInstance}.

\begin{definition} [\textbf{Planning Instance}]\label{def:planningInstance}
A planning instance $\Pi$ is represented by a triple $\langle \Xi, \mathcal{I}, G\rangle$. 
\begin{itemize}
	\item $\Xi =  \langle \Sigma, \mathcal{A}\rangle$ is the domain definition; 
	\item $\mathcal{I} \subseteq \Sigma$ is the initial state specification, which is defined by specifying the value for all facts in the initial state; and 
	\item $G \subseteq \Sigma$ is the goal state specification, which represents a desired state to be achieved.
\end{itemize}
\end{definition}

Classical planning representations often separate the definition of $\mathcal{I}$ and $G$ as part of a planning problem to be used together with a domain $\Xi$, such as STRIPS~\cite{STRIPSFikes1971} and PDDL~\cite{PDDLMcdermott1998}. In this work, we use the STRIPS fragment of PDDL to formalize planning domains and problems.
Finally, a \textit{plan} is the solution of a planning instance, as formalized in Definition~\ref{def:plan}.

\begin{definition} [\textbf{Plan}]\label{def:plan}
A plan $\pi$ for a planning instance $\Pi = \langle \Xi, \mathcal{I}, G\rangle$ is a sequence of actions $\langle$$a_1$, $a_2$, ..., $a_n$$\rangle$ that modifies the initial state $\mathcal{I}$ into a state $S\models G$ in which the goal state $G$ holds by the successive execution of actions in a plan $\pi$. 
A plan $\pi^{*}$ with length $|\pi^{*}|$ is optimal if there exists no other plan $\pi'$ for $\Pi$ such that $\pi' < \pi^{*}$.
\end{definition}

While actions have an associated cost, we take the assumption from classical planning that this cost is 1 for all instantiated actions. 
A plan $\pi$ is considered optimal if its cost, and thus length, is minimal.

Finally, modern classical planners use a variety of heuristics to efficiently explore the search space of planning domains by estimating the cost to achieve a specific goal~\cite{AutomatedPlanning_Book2016}. 
In classical planning, this estimate is often the number of actions to achieve the goal state from a particular state, so we describe all techniques assuming a uniform action cost $c(a) = 1$ for all $a \in \mathcal{A}$.  
Thus, the cost for a plan $\pi =[ a_1, a_2, ..., a_n ]$ is $c(\pi) = \Sigma c(a_{i})$.
Heuristics make no guarantees of the accuracy of their estimations, however, when a heuristic never overestimates the cost to achieve a goal, it is called admissible and guarantees optimal plans for certain search algorithms. 
A heuristic $h$($s$) is admissible if $h$($s$) $\leq$ $h$*($s$) for all states, where $h$*($s$) is the optimal cost to the goal from state $s$. Heuristics that overestimate the cost to achieve a goal are called inadmissible. 

\subsection{Landmarks}
\label{sec:landmarks}

In the planning literature~\cite{LandmarksRichter_2008}, \textit{landmarks} are defined as necessary properties (alternatively, actions) that must be true (alternatively, executed) at some point in every valid plan (see Definition~\ref{def:plan}) to achieve a particular goal, being often partially ordered following the sequence in which they must be achieved. 
Hoffman~\etal~(\citeyear{Hoffmann2004_OrderedLandmarks}) define fact landmarks, and Vidal and Geffner~(\citeyear{Vicent_ActionLandmarks_2005}) define action landmarks, as follows:

\begin{definition}[\textbf{Fact Landmark}]\label{def:planLandmark}
Given a planning instance $\Pi = \langle \Xi, \mathcal{I}, G\rangle$, a formula $F_{l}$ is a landmark in $\Pi$ iff $F_{l}$ is true at some point along all valid plans that achieve $G$ from $\mathcal{I}$. 
In other words, a landmark is a type of formula (\exemp, conjunctive formula or disjunctive formula) over a set of facts that must be satisfied (or achieved) at some point along all valid plan executions.
\end{definition}

\begin{definition}[\textbf{Action Landmark}]
Given a planning instance $\Pi = \langle \Xi, \mathcal{I}, G\rangle$, an action $A_{l}$ is a landmark in $\Pi$ iff $A_{l}$ is a necessary action that must be executed at some point along all valid plans that achieve $G$ from $\mathcal{I}$.
\end{definition}

From the concept of fact landmarks, Hoffmann~\etal~(\citeyear{Hoffmann2004_OrderedLandmarks}) introduce two types of landmarks as formulas: \textit{conjunctive} and \textit{disjunctive landmarks}. 
A \textit{conjunctive landmark} is a set of facts that must be true together at some point in every valid plan to achieve a goal. 
A \textit{disjunctive landmark} is a set of facts such that at least one of the facts must be true at some point in every valid plan to achieve a goal. Figure~\ref{fig:Landmarks-BlocksWorld} shows an example that illustrates a set of landmarks for a \textsc{Blocks-World}\footnote{\textsc{Blocks-World} is a classical planning domain where a set of stackable blocks must be re-assembled on a table~\cite{AutomatedPlanning_Book2011}.} problem instance. This example shows a set of conjunctive ordered landmarks (connected boxes) that must be true to achieve the goal state \pred{(on A B)}. For instance, to achieve the fact landmark \pred{(on A B)} which is also the goal state, the conjunctive landmark \pred{(and (holding A) (clear B))} must be true immediately before, and so on, as shown in Figure~\ref{fig:Landmarks-BlocksWorld}.

\begin{figure}[tb!]
  \centering
  \includegraphics[width=0.7\linewidth]{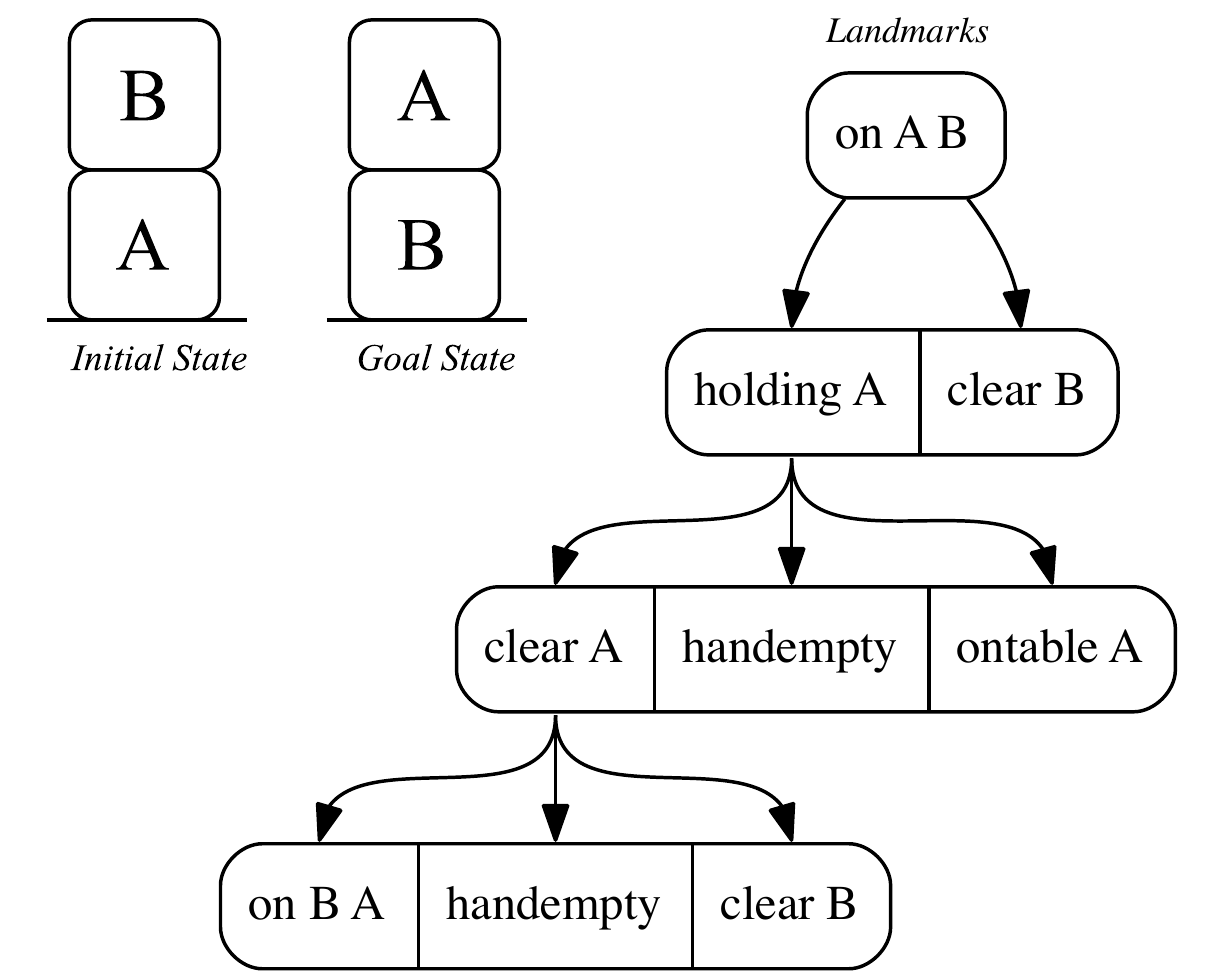}
  \caption{Ordered landmarks for a \textsc{Blocks-World} problem instance.}
  \label{fig:Landmarks-BlocksWorld}
\end{figure}

Whereas in planning the concept of landmarks is used to build heuristics~\cite{LandmarksRichter_2008} and planning algorithms~\cite{RichterLPG_2010}, in this work, we propose a novel use for landmarks: to monitor an agent's plan execution. 
Intuitively, we use landmarks as waypoints (or stepping stones) in order to monitor what an observed agent cannot avoid to achieve its goals.

\subsection{Goal Recognition}

Goal recognition is the task of recognizing agents' goals by observing their interactions in an environment~\cite{ActivityIntentPlanRecogition_Book2014}. 
In goal recognition, such observed interactions (\idest, observations) comprise the evidence available to recognize goals. 
Definition~\ref{def:observationSequence} follows the formalism proposed by Ram{\'{\i}}rez and Geffner in~(\citeyear{RamirezG_IJCAI2009,RamirezG_AAAI2010}) characterizing an observation sequence as the result of action sequence.

\begin{definition}[\textbf{Observation Sequence}]\label{def:observationSequence}
An observation sequence $O = \langle o_1, o_2, ..., o_n\rangle$ is said to be satisfied by a plan $\pi = \langle a_1, a_2, ..., a_m\rangle$, if there is a monotonic function $f$ that maps the observation indices $j = 1, ..., n$ into action indices $i = 1, ..., n$, such that $a_{f(j)} = o_{j}$.
\end{definition}

By combining the various notions of planning problem and an observation sequences, we formally define a goal recognition problem over a planning domain definition following Ram{\'{\i}}rez and Geffner~(\citeyear{RamirezG_IJCAI2009})\footnote{Unlike the probabilistic approach developed by Ram{\'{\i}}rez and Geffner~(\citeyear{RamirezG_AAAI2010}), our heuristic approaches do not use any prior probabilities to perform the goal recognition process.} in Definition~\ref{def:goalRecognition}, and a weak notion of solution to that problem in Definition~\ref{def:goalRecognitionSolution}. 

\begin{definition}[\textbf{Goal Recognition Problem}]\label{def:goalRecognition}
A goal recognition problem is a tuple $T_{GR} = \langle\Xi,\mathcal{I} ,\mathcal{G}, O\rangle$, where: 
\begin{itemize}
	\item $\Xi = \langle\Sigma, \mathcal{A}\rangle$ is a planning domain definition; 
	\item $\mathcal{I}$ is the initial state; 
	\item $\mathcal{G}$ is the set of possible goals, which include a correct hidden goal $G^{*}$ (\idest, $G^{*} \in \mathcal{G}$); and 
	\item $O = \langle o_1, o_2, ..., o_n\rangle$ is an observation sequence of executed actions, with each observation $o_i \in \mathcal{A}$, and the corresponding action being part of a valid plan $\pi$ (from Definition~\ref{def:plan}) that transitions $\mathcal{I}$ into $G^{*}$ through the sequential execution of actions in $\pi$. 
\end{itemize}
\end{definition}

\begin{definition}[\textbf{Solution to a Goal Recognition Problem}]\label{def:goalRecognitionSolution}
	A solution to a goal recognition problem $T_{GR} = \langle\Xi,\mathcal{I} ,\mathcal{G}, O\rangle$ is a nonempty subset of the set of possible goals $\mathbf{G} \subseteq \mathcal{G}$ such that $\forall G \in \mathbf{G}$ there exists a plan $\pi_{G}$ generated from a planning instance $\langle\Xi,\mathcal{I},G\rangle$ and $\pi_{G}$ is consistent with $O$.
\end{definition}

Thus, the ideal solution for a goal recognition problem is a set containing only the correct hidden goal $G^{*} \in \mathcal{G}$ that the observation sequence $O$ of a plan execution achieves. 
As an example of how the goal recognition process works, consider the Example~\ref{exemp:goalRecognition}, as follows.

\begin{example}\label{exemp:goalRecognition}
To exemplify the goal recognition process, let us consider the \textsc{Blocks-World} example in {\normalfont Figure~\ref{fig:blocksExample}}. 
The initial state represents an initial configuration of stackable blocks, while the set of candidate goals is composed by the following stacked ``words'': {\normalfont \pred{RED}, \pred{BED}, \textit{and} \pred{SAD}}. 
Consider an observation sequence for a hidden goal {\normalfont \pred{RED}} consisting of the following action sequence: 
{\normalfont $[$\pred{(unstack D B)}, \pred{(putdown D)}, \pred{(unstack E A)}, \pred{(stack E D)}, \pred{(pickup R)}, \pred{(stack R E)}$]$}. 
By following the full plan, we can easily infer that the hidden goal is indeed {\normalfont \pred{RED}}. 
However, if the we cannot observe action {\normalfont \pred{(stack R E)}}, it is not trivial to infer that {\normalfont \pred{RED}} is indeed the goal the observation sequence aims to achieve. 
\end{example}

\begin{figure}[tb]
  \centering
  \includegraphics[width=0.65\linewidth]{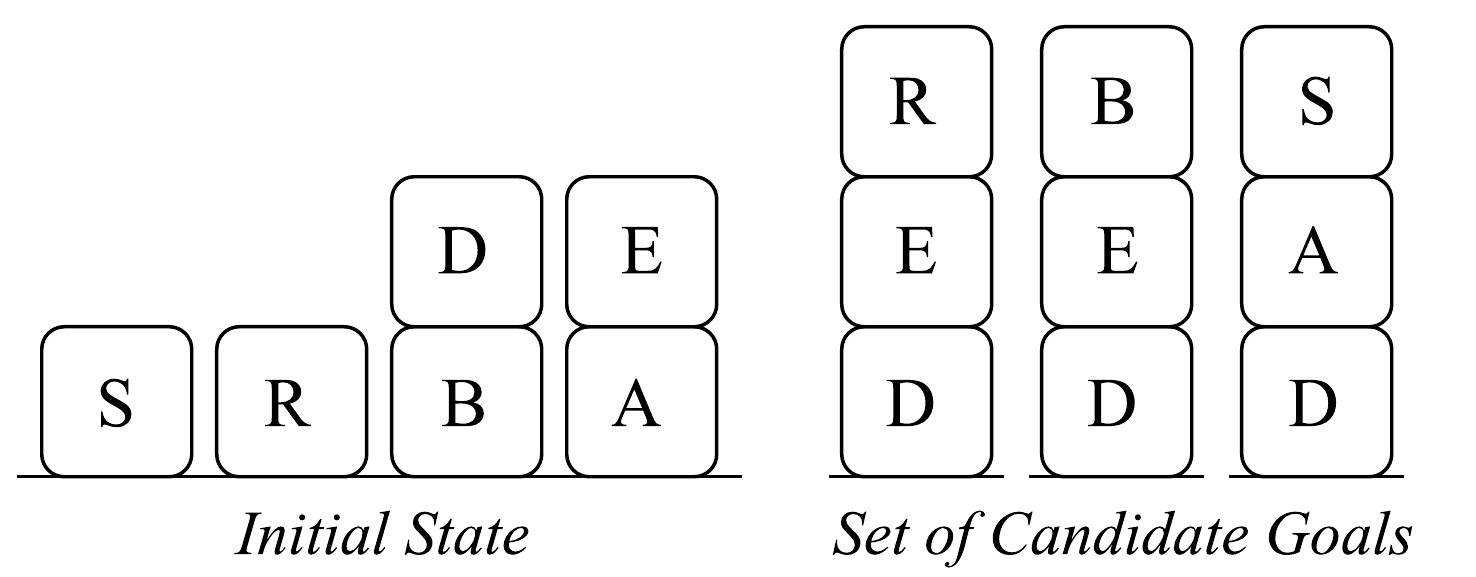}
  \caption{\textsc{Blocks-World} example.}
  \label{fig:blocksExample}
\end{figure}

Like~\cite{Sohrabi_IJCAI2016}, we also deal with missing observations during the goal recognition process. 
We differ from~\cite{Sohrabi_IJCAI2016} in that we do not deal with noisy (unreliable) observations explicitly. 
Nevertheless our technique proves to be robust against noise by relying exclusively on necessary conditions for the plans leading to each goal as our empirical analysis corroborates. 
In a partial observation sequence, we observe only a sub-sequence of actions of a plan that achieves a particular goal because some actions are missing or obfuscated. 
A noisy observation sequence contains one or more actions (or a set of facts) that might not be part of a plan that achieves a particular goal, \exemp, when a sensor fails and generates abnormal or spurious readings. We formalize the way in which an environment generates observations of agent plans in Definition~\ref{def:plan-projection}. 

\begin{definition}[\bf Observation Sequence Generation]\label{def:plan-projection}
	Let $\pi = \langle a_1, a_2, \\ \dots, a_n\rangle$ be a plan generated by a planning instance $\Pi = \langle \Xi, \mathcal{I}, G\rangle$. 
	An action projection function $\mathit{ap}(a):\mathcal{A} \mapsto \vec{\mathcal{A}}$ is a function that maps actions to sequences of zero or more actions.
	An observation sequence generation function $\mathit{os}(\pi)$ is a function that maps a plan $\pi$ into an observation sequence $O$ as follows:
	$$os(\pi) = \begin{cases}
		\langle \rangle & \text{if}~\pi=\langle\rangle \\
		\langle \mathit{ap}(a_1)\rangle \cdot \mathit{os}(\langle a_2, \dots, a_n \rangle) & \text{if}~\pi=\langle a_2, \dots, a_n \rangle
	\end{cases}$$
\end{definition}

The key to generating such sequences is how the rules for function $\mathit{ap}$ to translate actions. 
Following our Example~\ref{exemp:goalRecognition}, we could define an action projection function that never generates observations for unstack actions, and generates noise for all stack actions as follows. 
$$\mathit{ap}_1(a) = \begin{cases}
	\langle \pred{(pickup X)}\rangle & \text{if}~a=\pred{(pickup X)}\\
	\langle \pred{(putdown X)}\rangle & \text{if}~a=\pred{(putdown X)}\\
	\langle \rangle & \text{if}~a=\pred{(unstack X Y)}\\
	\langle \pred{(stack X Y)}, \pred{(unstack X Y)} \rangle & \text{if}~a=\pred{(stack X Y)}\\
\end{cases}$$
We formally define missing and noisy observations in Definitions~\ref{def:missingObservation}~and~\ref{def:noisyObservation} and both types of observation are exemplified in Example~\ref{exemp:missingNoisy}.


\begin{definition}[\textbf{Missing Observation}]\label{def:missingObservation}
Let $\Pi = \langle \Xi, \mathcal{I}, G\rangle$ be a planning instance, $\pi$ be a valid plan that achieves $G$ from $\mathcal{I}$, and $O$ is an observation sequence induced by an observation generation function $\mathit{os}$ with an action projection function $\mathit{ap}$. 
An observation sequence $O$ misses observations (is a partial or incomplete observation sequence) with respect to the plan $\pi$ that achieves the goal $G$ from $\mathcal{I}$ if the $\mathit{ap}$ function maps any action into the empty sequence $\langle \rangle$.
\end{definition}


\begin{definition}[\textbf{Noisy Observation}]\label{def:noisyObservation}
Let $\Pi = \langle \Xi, \mathcal{I}, G\rangle$ be a planning instance, $\pi$ be a valid plan that achieves $G$ from $\mathcal{I}$, and $O$ is an observation sequence induced by an observation generation function $\mathit{os}$ with an action projection function $\mathit{ap}$. 
An observation sequence $O$ contains noisy observations with respect to the plan $\pi$ that achieve the goal $G$ from $\mathcal{I}$ if the $\mathit{ap}$ function maps any action $a$ into a non-empty sequence containing one or more action $a'\neq a$.
\end{definition}


\begin{example}\label{exemp:missingNoisy}
Let us consider that a valid plan to achieve a goal $G$ is $\pi = [a,b,c,d,e]$. Consider the following observation sequences $O_{m1}$, $O_{m2}$, and $O_{m3}$:
\begin{itemize}
	\item $O_{m1} = [a,d]$;
	\item $O_{m2} = [b,e]$; and
	\item $O_{m3} = [d,a,c]$
\end{itemize} 
Observation sequences $O_{m1}$ and $O_{m2}$ satisfy {\normalfont Definition~\ref{def:missingObservation}}, and therefore, they are partial observation sequences and contain missing observed actions. 
$O_{m3}$ is not a partial observation sequence because it does not satisfy {\normalfont Definition~\ref{def:missingObservation}} as the observation sequence $[d,a,c]$ is not a strict subset of ordered actions of the plan $\pi$. 

Now, consider the following observation sequences $O_{n1}$ and $O_{n2}$: 
\begin{itemize}
	\item $O_{n1} = [a,b,c,d,e,g]$; and
	\item $O_{n2} = [b,d,h]$
\end{itemize}
It is possible to see that both observation sequences $O_{n1}$ and $O_{n2}$ contain noisy observations (g and h respectively) and satisfy {\normalfont Definition~\ref{lst:FactLandmarksUniquenessValue}}. 
However, note that $O_{n2}$ contains not only noisy observations but it also misses observations, \idest, $O_{n2}$ is partial observation sequence with noisy observations.
\end{example}

Although we define missing and noisy observations with actions as observations, our goal recognition approaches can also deal with facts (or fluents) as observations, like~\cite{Sohrabi_IJCAI2016}. 
Indeed, as we see in Section~\ref{section:LandmarkGoalRecognitionApproaches}, using states as observations makes goal recognition much easier for our heuristic approaches, since we can use the observations directly to compute achieved landmarks. 
In Section~\ref{section:LandmarkGoalRecognitionApproaches}, we show that what matters for our goal recognition approaches is the evidence of fact landmarks during the observations, and it is irrelevant whether this evidence is provided by either an observed action or a set of facts.

\section{Extracting Recognition Information from Planning Definition}\label{section:InferringStructures}

In this section, we describe the process through which we extract useful information for goal recognition from planning domain definition. 
First, we describe landmark extraction algorithms from the literature, and how we use these algorithms in our approaches in Section~\ref{subsec:extractinglandmarks}. 
Second, we show how we classify facts into partitions from planning action descriptions and how we use them during the goal recognition process in Section~\ref{subsec:factparitioning}. 

\subsection{Extracting Landmarks}\label{subsec:extractinglandmarks}


Hoffman~\etal~(\citeyear{Hoffmann2004_OrderedLandmarks}) proves that the process of generating exactly all landmarks and deciding about their ordering is PSPACE-complete, which is exactly the same complexity of deciding plan existence~\cite{PlanningComplexity_Bylander1994}. 
Nevertheless, most landmark extraction algorithms extract only a subset of landmarks for a given planning instance for efficiency. 
While there are several algorithms to extract landmarks and their orderings in the literature that we could use~\cite{ICAPS03_DC_ZhGi,LandmarksRichter_2008,KeyderRH_ECAI10}, we chose the landmark extraction algorithm from Hoffmann~\etal~(\citeyear{Hoffmann2004_OrderedLandmarks}) to extract landmarks from planning instances due to its speed and simplicity. 
This algorithm can extract both conjunctive and disjunctive landmarks, but we use the conjunctive landmarks to build heuristics for our goal recognition approaches.

To represent landmarks and their ordering, the algorithm of Hoffmann~\etal~(\citeyear{Hoffmann2004_OrderedLandmarks}) uses a tree in which nodes represent landmarks and edges represent necessary prerequisites between landmarks. 
Each node in the tree represents a conjunction of facts that must be true simultaneously at some point during plan execution, and the root node is a landmark representing the goal state. 
This algorithm uses a Relaxed Planning Graph (RPG)~\cite{FFHoffmann_2001}, which is a leveled graph that ignores the delete-list effects of all actions, thus containing no mutex relations. 
Once the RPG is built, the algorithm extracts \textit{landmark candidates} by back-chaining from the RPG level in which all facts of the goal state $G$ are possible, and, for each fact $g$ in $G$, checks which facts must be true until the first level of the RPG. 
For example, if fact $B$ is a landmark and all actions that achieve $B$ share $A$ as precondition, then $A$ is a landmark candidate. 
To confirm that a landmark candidate is indeed a landmark, the algorithm builds a new RPG structure by removing actions that achieve this landmark candidate and checks the solvability over this modified problem\footnote{Deciding the solvability of a relaxed planning problem using an RPG structure can be done in polynomial time~\cite{BlumFastPlanning_95}.}, and, if the modified problem is unsolvable, then the landmark candidate is a necessary landmark. This means that the actions that achieve the landmark candidate are necessary to solve the original planning problem. 

\begin{figure}[t!]
  \centering
  \includegraphics[width=0.65\linewidth]{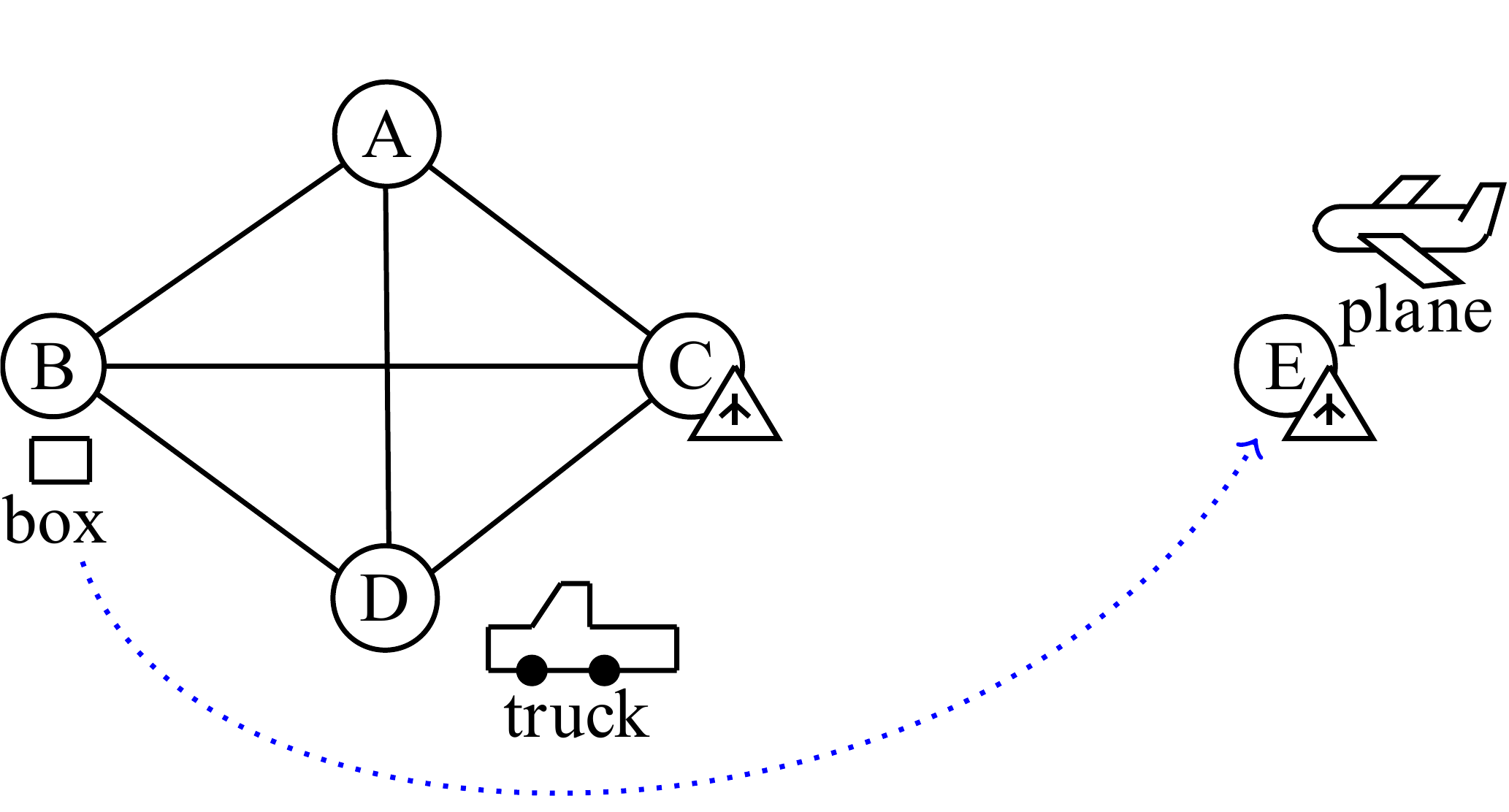}
  \caption{\textsc{Logistics} problem example.}
  \label{fig:logisticsExample}
\end{figure}

To exemplify the output of the landmark extraction algorithm from~\cite{Hoffmann2004_OrderedLandmarks}, consider the \textsc{Logistics}\footnote{The \textsc{Logistics} domain consists of airplanes and trucks transporting packages between locations (\exemp, airports and cities).} problem example in Figure~\ref{fig:logisticsExample}. 
Fact landmarks extracted for this example are shown respectively in Listing~\ref{lst:logisticsFactLandmarks} and Figure~\ref{fig:logisticsLandmarks}. 
While Listing~\ref{lst:logisticsFactLandmarks} shows one possible serialization of the landmarks, Figure~\ref{fig:logisticsLandmarks} represents the same landmarks ordered from bottom-up by facts that must be true together. 
These landmarks allow us to monitor way-points during a plan execution to determine which goals this plan is going to achieve. 

\begin{figure}[ht]
  \centering
  \includegraphics[width=1\linewidth]{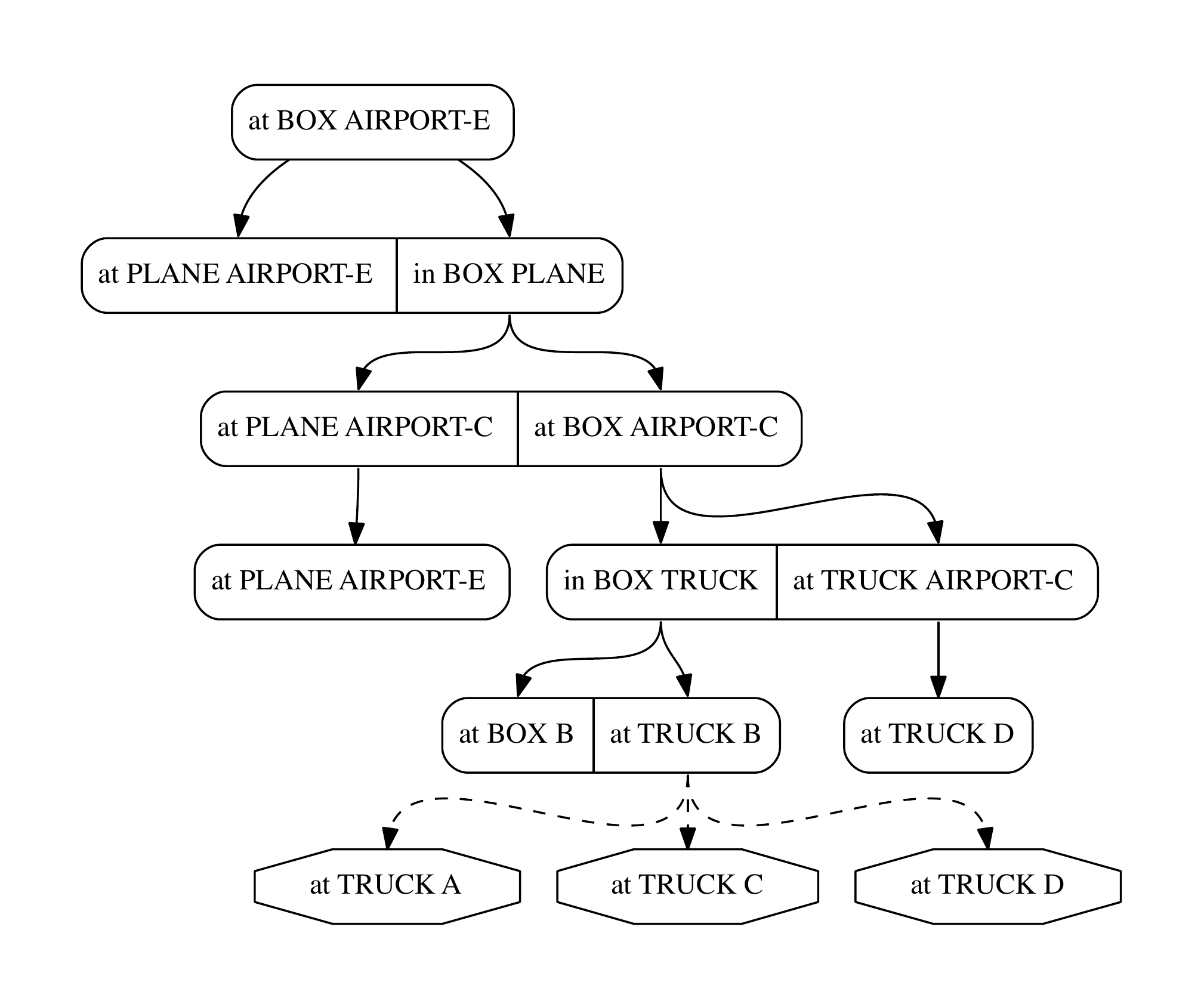}
  \caption{Ordered fact landmarks extracted from \textsc{Logistics} problem example shown in Figure~\ref{fig:logisticsExample}. Fact landmarks that must be true together are represented by connected boxes, which are conjunctive facts, \idest, representing conjunctive landmarks. Disjunctive landmarks are represented by octagon boxes that are connected by dashed lines.}
  \label{fig:logisticsLandmarks}
\end{figure}

\begin{lstlisting}[float=!tb,caption={Fact landmarks (conjunctive and disjunctive) extracted from the \textsc{Logistics} example in Figure~\ref{fig:logisticsExample}.},label={lst:logisticsFactLandmarks},basicstyle=\ttfamily\footnotesize]
Fact Landmarks:
(and (at BOX AIRPORT-E))
(and (at PLANE AIRPORT-E) (in BOX PLANE))
(and (at PLANE AIRPORT-C) (at BOX AIRPORT-C))
(and (at PLANE AIRPORT-E))
(and (at TRUCK D))
(and (in BOX TRUCK) (at TRUCK AIRPORT-C))
(and (at BOX B) (at TRUCK B))
(or (at TRUCK A) (at TRUCK C) (at TRUCK D))
\end{lstlisting} 

This landmark extraction algorithm is referred to as function \textsc{ExtractLandmarks}, which takes as input a planning domain definition $\Xi = \langle \Sigma, \mathcal{A}\rangle$, an initial state $\mathcal{I}$, and a set of candidate goals $\mathcal{G}$ or a single goal $G$. In case the input is a set of candidate goals $\mathcal{G}$, this function outputs a map $\mathcal{L}_{\mathcal{G}}$ that associates candidate goals to their respective ordered fact landmarks (\idest, a set of landmarks with an order relation).
Alternatively, in case the input is a single goal $G$, this function outputs a map $\mathcal{L}_{G}$ that associates the goal $G$ to its respective ordered fact landmarks.

We note that many landmark extraction techniques, including that of Hoffmann~et al.~\cite{Hoffmann2004_OrderedLandmarks}, might infer incorrect landmark orderings, which can lead to problems if the goal recognition process relies on the ordering information to make inferences. 
Nevertheless, our empirical evaluation shows that landmark orderings do not affect detection performance in the experimental datasets. 
We discuss landmark orderings later (Section~\ref{subsec:computingAchievedLandmarks}) in the paper.

\subsection{Classifying Facts into Partitions}\label{subsec:factparitioning}

Pattison and Long~(\citeyear{PattisonGoalRecognition_2010}) classify facts into mutually exclusive partitions in order to infer whether certain observations are likely to be goals for goal recognition. 
Their classification relies on the fact that, in some planning domains, predicates may provide additional information that can be extracted by analyzing preconditions and effects in operator definitions. 
We use this classification to infer if certain observations are consistent with a particular goal, and if not, we can eliminate a candidate goal. We formally define fact partitions in what follows. 

\begin{definition} [\textbf{Strictly Activating}] \label{def:strictlyActivating}
A fact $f$ is strictly activating if $f \in \mathcal{I}$ and $\forall a \in \mathcal{A}$, $f \notin \textit{eff}(a)^+ \cup \textit{eff}(a)^-$. Furthermore, $\exists a \in \mathcal{A}$, such that $f \in$ \textit{pre}($a$).
\end{definition}

\begin{definition}[\textbf{Unstable Activating}] \label{def:unstableActivating}
A fact $f$ is unstable activating if $f \in \mathcal{I}$ and $\forall a \in \mathcal{A}$, $f \notin \textit{eff}(a)^+$ and $\exists a,b \in \mathcal{A}, f \in \textit{pre}(a)$ and $f \in \textit{eff}(b)^-$.
\end{definition}

\begin{definition}[\textbf{Strictly Terminal}] \label{def:strictlyTerminal}
A fact $f$ is strictly terminal if $\exists a \in \mathcal{A}$, such that $f \in \textit{eff}(a)^+$ and $\forall a \in \mathcal{A}$, $f \notin \textit{pre}(a)$ and $f \notin \textit{eff}(a)^-$.
\end{definition}

A \textit{Strictly Activating} fact (Definition~\ref{def:strictlyActivating}) appears as a precondition, and does not appear as an add or delete effect in an operator definition. 
This means that unless defined in the initial state, this fact can never be added or deleted by an operator.
An \textit{Unstable Activating} fact (Definition~\ref{def:unstableActivating}) appears as both a precondition and a delete effect in two operator definitions, so once deleted, this fact cannot be re-achieved. 
The deletion of an unstable activating fact may prevent a plan execution from achieving a goal.
A \textit{Strictly Terminal} fact (Definition~\ref{def:strictlyTerminal}) does not appear as a precondition of any operator definition, and once added, cannot be deleted. 
For some planning domains, this kind of fact is the most likely to be in the set of goal facts, because once added in the current state, it cannot be deleted, and remains true until the final state.

The fact partitions that we can extract depend on the planning domain definition. 
For example, from the \textsc{Blocks-World} domain, it is not possible to extract any fact partitions.
However, it is possible to extract fact partitions from the \textsc{IPC-Grid}\footnote{\textsc{IPC-Grid} domain consists of an agent that moves in a grid using keys to open locked locations.} domain, such as \textit{Strictly Activating} and \textit{Unstable Activating} facts. 
In this work, we use fact partitions to obtain additional information on fact landmarks during the goal recognition process. 
For example, consider an \textit{Unstable Activating} fact landmark $L_{ua}$, so that if $L_{ua}$ is deleted from the current state, then it cannot be re-achieved. 
We can trivially determine that goals for which this fact is a landmark are unreachable, because there is no available action that achieves $L_{ua}$ again.

\section{Landmark-Based Goal Recognition}\label{section:LandmarkGoalRecognitionApproaches}

We now describe our goal recognition approaches that rely on planning landmarks. 
First, we start with a method to monitor and compute the evidence of landmarks from observations in Section~\ref{subsec:computingAchievedLandmarks}. 
Second, we develop a landmark-based filtering method that can be used with any other planning-based goal and plan recognition approach in Section~\ref{subsec:filteringCandidateGoals}. 
Finally, we describe how we build goal recognition heuristics using landmarks in Sections~\ref{subsec:goalCompletionHeuristic}~and~\ref{subsec:uniquenessHeuristic}.

\subsection{Computing Achieved Landmarks in Observations}\label{subsec:computingAchievedLandmarks}

An essential part of our approaches to goal recognition is the ability to monitor and compute the evidence of achieved fact landmarks in the observations. 
To do so, we compute the evidence of achieved fact landmarks in preconditions and effects of observed actions during a plan execution~\cite{RamonNirMeneguzzi_AAAI2017} using the \textsc{ComputeAchievedLandmarks} function shown in Algorithm~\ref{alg:ComputeAchievedLandmarks}. 
This algorithm takes as input an initial state $\mathcal{I}$, a set of candidate goals $\mathcal{G}$, a sequence of observed actions $O$, and a map $\mathcal{L}_{\mathcal{G}}$ containing candidate goals and their extracted fact landmarks (provided by the \textsc{ExtractLandmarks} function). 
Note that Algorithm~\ref{alg:ComputeAchievedLandmarks} can be easily modified to allow it to deal with observations as states, so instead of analyzing preconditions and effects of actions, we compare the observations directly to computed landmarks.

\floatname{algorithm}{Algorithm}
\begin{algorithm}[h!]
    \caption{Compute Achieved Landmarks in Observations.}
    \textbf{Input:} $\mathcal{I}$ \textit{initial state}, $\mathcal{G}$ \textit{set of candidate goals}, $O$ \textit{observations}, and $\mathcal{L}_{\mathcal{G}}$ \textit{goals and their extracted landmarks}.
    \\\textbf{Output:} \textit{A map of goals to their achieved landmarks.}
	\label{alg:ComputeAchievedLandmarks}
    \begin{algorithmic}[1]
        \Function{ComputeAchievedLandmarks}{$\mathcal{I}, \mathcal{G}, O, \mathcal{L}_{\mathcal{G}}$}
        \State $\Lambda_{\mathcal{G}} \gets \langle \rangle$ \Comment{\textit{Map goals $\mathcal{G}$ to their respective achieved landmarks}.}
        \For{\textbf{each} goal $G$ in $\mathcal{G}$}\label{alg:line:IterateCandidateGoals}
			\State $\mathcal{L}_{G} \gets$ fact landmarks of $G$ s.t $\langle G, \mathcal{L}_{G}\rangle$ in $\mathcal{L}_{\mathcal{G}}$\label{alg:line:GetLandmarks}
			\State $\mathcal{L_{\mathcal{I}}} \gets$ all fact landmarks $L \in \mathcal{I}$\label{alg:line:CheckAchievedLandmarksInitialState}
			\State $\mathcal{L} \gets \emptyset$
			\For{\textbf{each} observed action $o$ in $O$}\label{alg:line:IterateObservations}
				\State $\mathcal{L} \gets$ all fact landmarks $L$ in $\mathcal{L}_{G}$ such that $L$ $\in$ \textit{pre}($o$) $\cup$ \textit{eff}($o$)$^+$ and $L$ $\notin$ $\mathcal{L}$\label{alg:line:CheckAchievedLandmarks}
				\State $\mathcal{L}_{\prec} \gets$ predecessors $L_{\prec}$ of all $L$ in $\mathcal{L}$, such that $L_{\prec} \notin  \mathcal{L}$ \label{alg:line:CheckAchievedLandmarksPredecessors}
				\State $\mathcal{AL}_{G} \gets \mathcal{AL}_{G} \cup \lbrace \mathcal{L_{\mathcal{I}}} \cup \mathcal{L} \cup \mathcal{L}_{\prec}\rbrace$\label{alg:line:AddAchievedLandmarksGoal}
			\EndFor
			\State $\Lambda_{\mathcal{G}}(G) \gets \mathcal{AL}_{G}$ \Comment{\textit{Achieved landmarks of $G$.}}\label{alg:line:AddAchievedLandmarksIntoMap}
		\EndFor
		\State \textbf{return} $\Lambda_{\mathcal{G}}$\label{alg:line:ReturnAchievedLandmarks}
        \EndFunction
    \end{algorithmic}
\end{algorithm} 

Algorithm~\ref{alg:ComputeAchievedLandmarks} iterates over the set of candidate goals $\mathcal{G}$ (Line~\ref{alg:line:IterateCandidateGoals}) selecting the fact landmarks $\mathcal{L}_{G}$ of each goal $G$ in $\mathcal{L}_{\mathcal{G}}$ in Line~\ref{alg:line:GetLandmarks} and computes the fact landmarks that are in the initial state in Line~\ref{alg:line:CheckAchievedLandmarksInitialState}.
With this information, the algorithm iterates over the observed actions $O$ to compute the achieved fact landmarks of $G$ in Lines~\ref{alg:line:IterateObservations} to~\ref{alg:line:AddAchievedLandmarksGoal}. 
For each observed action $o$ in $O$, the algorithm computes all fact landmarks of $G$ that are either in the preconditions and effects of $o$ in Line~\ref{alg:line:CheckAchievedLandmarks}. 
As we deal with partial observations in a plan execution some executed actions may be missing from the observation sequence, thus whenever we identify a fact landmark, we also infer that its predecessors must have been achieved in Line~\ref{alg:line:CheckAchievedLandmarksPredecessors}. 
For example, consider that the set of fact landmarks to achieve a goal from a state is represented by the following ordered facts: \pred{(at A)} $\prec$ \pred{(at B)} $\prec$ \pred{(at C)} $\prec$ \pred{(at D)}, and we observe just one action during a plan execution, and this observed action contains the fact landmark \pred{(at C)} as an effect. 
From this observed action, we can infer that the predecessors of \pred{(at C)} must have been achieved before this observation (\idest, \pred{(at A)} and \pred{(at B)}), and therefore, we also include them as achieved landmarks. 
At the end of each iteration over an observed action $o$, the algorithm stores the set of achieved landmarks of $G$ in $\mathcal{AL}_{G}$ in Line \ref{alg:line:AddAchievedLandmarksGoal}. 
Finally, after computing the evidence of achieved landmarks in the observations for a candidate goal $G$, the algorithm stores the set of achieved landmarks $\mathcal{AL}_{G}$ of $G$ in $\Lambda_{\mathcal{G}}$ (Line~\ref{alg:line:AddAchievedLandmarksIntoMap}) and returns a map $\Lambda_{\mathcal{G}}$ containing all candidate goals and their respective achieved fact landmarks (Line~\ref{alg:line:ReturnAchievedLandmarks}). Example~\ref{exemp:computingAchievedLandmarks} illustrates the execution of Algorithm~\ref{alg:ComputeAchievedLandmarks} to compute achieved landmarks from the observations of our running example.

\begin{example}\label{exemp:computingAchievedLandmarks}
Consider the \textsc{Blocks-World} example from {\normalfont Figure~\ref{fig:blocksExample}}, and the following observed actions: {\normalfont \pred{(unstack E A)}} and {\normalfont \pred{(stack E D)}}. 
Thus, from these observed actions, the candidate goal {\normalfont \pred{RED}}, and the set of fact landmarks of this candidate goal {\normalfont (Figure~\ref{fig:RED-AchievedLandmarks})}, our algorithm computes that the following fact landmarks have been achieved: 

{
\normalfont
\begin{itemize}
	\item $\mathcal{AL}_{\pred{RED}}=\lbrace$\pred{[(clear R)]}, \pred{[(on E D)]}, \\\pred{[(clear R) (ontable R) (handempty)]}, \\ \pred{[(on E A) (clear E) (handempty)]}, \\\pred{[(clear D) (holding E)]}, \\\pred{[(on D B) (clear D) (handempty)]}$\rbrace$
\end{itemize}
}

In the preconditions of {\normalfont \pred{(unstack E A)}} the algorithm computes {\normalfont \pred{[(on E A) (clear E) (handempty)]}}. 
Subsequently, in the preconditions and effects of {\normalfont \pred{(stack E D)}} the algorithm computes {\normalfont \pred{[(clear D) (holding E)]}} and {\normalfont \pred{[(on E D)]}}, while it computes the other achieved landmarks for the word {\normalfont \pred{RED}} from the initial state.
{\normalfont Figure~\ref{fig:RED-AchievedLandmarks}} shows the set of achieved landmarks for the word {\normalfont \pred{RED}} in gray. {\normalfont Listing~\ref{lst:FactLandmarksUniquenessValue}} shows in bold the set of achieved landmarks that our algorithm computes for the set of candidate goals in {\normalfont Figure~\ref{fig:blocksExample}}.
\end{example}

\begin{figure}[th!]
  \centering
  \includegraphics[width=0.8\linewidth]{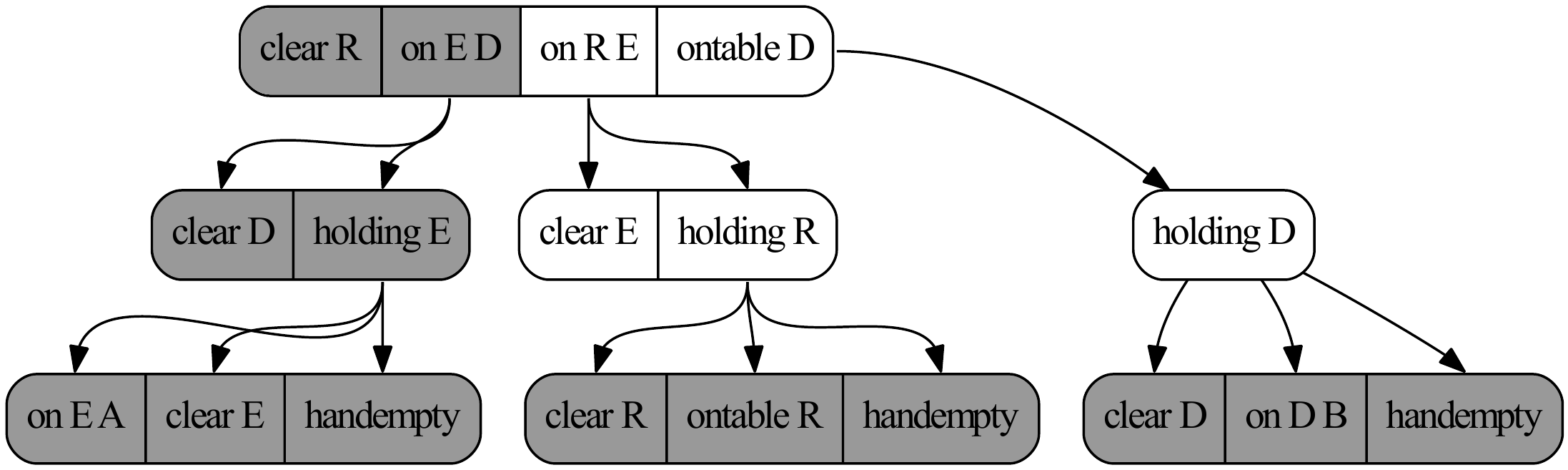}
  \caption{Ordered fact landmarks extracted for the stacked blocks for the word \pred{RED}. Fact landmarks that must be true together are represented by connected boxes. Connected boxes in grey represent achieved fact landmarks. Edges represent prerequisites between landmarks.}
  \label{fig:RED-AchievedLandmarks}
\end{figure}

The complexity of computing achieved landmarks in observations (Algorithm~\ref{alg:ComputeAchievedLandmarks}) with the process of extracting landmarks ($EL$) is: $O(EL + |\mathcal{G}|\cdot|O|\cdot|\mathcal{L}_{\mathcal{G}}|)$, where $\mathcal{G}$ is the set of candidate goals, $O$ is the observation sequence, and $\mathcal{L}_{\mathcal{G}}$ is the extracted landmarks for $\mathcal{G}$.

\subsection{Filtering Candidate Goals from Achieved Landmarks in Observations}
\label{subsec:filteringCandidateGoals}

We now develop an approach to filter candidate goals based on the evidence of fact landmarks and partitioned facts in preconditions and effects of observed actions in a plan execution~\cite{PereiraMeneguzzi_ECAI2016}. 
This filtering method analyzes fact landmarks in preconditions and effects of observed actions, and selects goals, from a set of candidate goals, that have achieved most of their associated landmarks. 

This filtering method is detailed in function \textsc{FilterCandidateGoals} of Algorithm~\ref{alg:filterCandidateGoals}. 
This algorithm takes as input a goal recognition problem $T_{GR}$, which is composed of a planning domain definition $\Xi$, an initial state $\mathcal{I}$, a set of candidate goals $\mathcal{G}$, a set of observed actions $O$, and a filtering threshold $\theta$. 
Our algorithm iterates over the set of candidate goals $\mathcal{G}$, and, for each goal $G$ in $\mathcal{G}$, it extracts and classifies fact landmarks and partitions for $G$ from the initial state $\mathcal{I}$ (Lines~\ref{alg:filter:extractLandmarks}~and~\ref{alg:filter:factPartitioner}). 
Function $\Call{PartitionFacts}{\mathcal{L}_{g},\mathcal{A}}$ takes a set of goals and the actions in the domain and returns the fact partitions induced by the actions in $\mathcal{A}$ into the sets $F_{sa}$ of strictly activating (from Definition~\ref{def:strictlyActivating}), $F_{ua}$ of unstable activating (from Definition~\ref{def:unstableActivating}), and $F_{st}$ of strictly terminal (from Definition~\ref{def:strictlyTerminal}) facts. 
We then check whether the observed actions $O$ contain fact landmarks or partitioned facts in either their preconditions or effects. 
At this point, if any \textit{Strictly Activating} facts for the candidate goal $G$ are not in initial state $\mathcal{I}$, then the candidate goal $G$ is no longer achievable, so we can discard it (Line~\ref{alg:filter:SA}). 
Subsequently, we check for \textit{Unstable Activating} and \textit{Strictly Terminal} facts of goal $G$ in the preconditions and effects of the observed actions $O$, and if we find any, we discard the candidate goal $G$ (Line~\ref{alg:filter:STandUA}).
If we observe no facts from partitions as evidence from the observed actions in $O$, we move on to checking landmarks of $G$ within the observed actions in $O$. 
If we observe any landmarks in the preconditions and positive effects of the observed actions (Line~\ref{alg:filter:identifyFactLandmarks}), we compute the evidence of achieved landmarks for the candidate goal $G$ (Line~\ref{alg:filter:computeFactLandmarks}). 
Like Algorithm~\ref{alg:ComputeAchievedLandmarks}, we deal with missing observations by inferring that the unobserved predecessors of observed fact landmarks must have been achieved in Line~\ref{alg:filter:predecessors}. 
Given the number of achieved fact landmarks of $G$, we then estimate the percentage of fact landmarks that the observed actions $O$ have achieved according to the ratio between the amount of achieved fact landmarks and the total amount of landmarks (Line~\ref{alg:filter:ratio}). 
Finally, after computing the percentage of landmark completion for all candidate goals in $\mathcal{G}$, we return the goals with the highest percentage of achieved landmarks within our filtering threshold $\theta$ (Line~\ref{alg:filter:filterGoals}). 
We follow Definition~\ref{def:planLandmark} of fact landmarks and consider conjunctive landmarks as a single landmark when counting achieved landmarks (Line~\ref{alg:filter:ratio}), except for the sub-goals, where each fact is a separate landmark. 
With respect to the threshold value, note that, if threshold $\theta=0$, the filter returns only the goals with maximum completion, given the observations. 
The threshold gives us flexibility when dealing with missing observations and sub-optimal plans, which, when $\theta=0$, it may cause some potential candidate goals to be filtered out before we get additional observations. 
Example~\ref{exemp:filterCandidateGoals} shows how our filtering method prunes efficiently goals from a set of candidate goals.

\afterpage{
\floatname{algorithm}{Algorithm}
\begin{algorithm}[h!]
    \caption{Filter Candidate Goals in Observations.}
    \textbf{Input:} $\Xi$ $=$ $\langle$$\Sigma$, $\mathcal{A}$$\rangle$ \textit{planning domain}, $\mathcal{I}$ \textit{initial state}, $\mathcal{G}$ \textit{set of candidate goals}, $O$ \textit{observations}, and $\theta$ \emph{threshold}.
    \\\textbf{Output:} \textit{A set of filtered candidate goals $\Lambda_{\mathcal{G}}$ with the highest percentage of achieved landmarks in observations $O$.}
	\label{alg:filterCandidateGoals}
    \begin{algorithmic}[1]
        \Function{FilterCandidateGoals}{$\Xi, \mathcal{I}, \mathcal{G}, O, \theta$}
        \State $\Lambda_{\mathcal{G}} \gets \langle \rangle$ \Comment{\textit{Map goals to \% of achieved landmarks}.}
        \For{each goal $G$ in $\mathcal{G}$}
			\State $\mathcal{L}_{G} \gets \Call{ExtractLandmarks}{\Xi, \mathcal{I}, G}$ \label{alg:filter:extractLandmarks}
			\State $\langle F_{sa}, F_{ua}, F_{st} \rangle \gets \Call{PartitionFacts}{\mathcal{L}_{g},\mathcal{A}}$ \label{alg:filter:factPartitioner} \Comment{\textit{$F_{sa}$: set of Strictly Activating facts, $F_{ua}$: set of Unstable Activating facts, and $F_{st}$: set of Strictly Terminal facts.}}
			\If{$F_{sa} \cap \mathcal{I} = \emptyset$}\label{alg:filter:SA}
				\State \textbf{continue} \Comment{\textit{Goal $G$ is no longer possible.}}
			\EndIf
			\State $\mathcal{AL}_{G} \gets \langle$ $\rangle$ \Comment{\textit{Achieved landmarks for $G$.}}
			\State $\mathcal{L_{\mathcal{I}}} \gets$ all fact landmarks $L \in \mathcal{I}$
			\For{each observed action $o$ in $O$}
				\If{$(F_{ua} \cup  F_{st}) \subseteq (\textit{pre}(o) \cup \textit{eff}(o)^+ \cup \textit{eff}(o)^-)$} \label{alg:filter:STandUA}
					\State $discardG = $ \textbf{true}
					\State \textbf{break}
				\Else
					\State $\mathcal{L} \gets$ remove all fact landmarks $L$ in $\mathcal{L}_{G}$ such that $L$ $\in$ \textit{eff}($o$)$^-$ and $L$ $\in$ $\mathcal{L}$\label{alg:filter:CheckAchievedAndDeletedLandmarks}
					\State $\mathcal{L} \gets$ all fact landmarks $L$ in $\mathcal{L}_{G}$ s.t $L$ $\in$ \textit{pre}($o$) $\cup$ \textit{eff}($o$)$^+$ and $L$ $\notin$ $\mathcal{L}$ \label{alg:filter:identifyFactLandmarks}
					\State $\mathcal{L}_{\prec} \gets$ predecessors $L_{\prec}$ of all $L$ in $\mathcal{L}$, s.t $L_{\prec} \notin  \mathcal{L}$\label{alg:filter:predecessors}
					\State $\mathcal{AL}_{G} \gets \mathcal{AL}_{G} \cup \lbrace \mathcal{L_{\mathcal{I}}} \cup \mathcal{L} \cup \mathcal{L}_{\prec}\rbrace$\label{alg:filter:computeFactLandmarks}
				\EndIf
			\EndFor
			\If{$discardG$} \textbf{break} \Comment{\textit{Avoid computing achieved landmarks for $G$.}}
			\EndIf
			\State $\Lambda_{\mathcal{G}} \gets \Lambda_{\mathcal{G}} \cup \langle G,\left(\frac{\mid\mathcal{AL}_{G}\mid}{\mid\mathcal{L}_{G}\mid}\right)\rangle$ \Comment{\textit{Percentage of achieved landmarks for $G$.}} \label{alg:filter:ratio}
		\EndFor
		\State \textbf{return} {all $G$ s.t $\langle G, v \rangle$ $\in\Lambda_{\mathcal{G}}$ and $v \geq (\max_{v_i}{ \langle G',v_i \rangle \in \Lambda_{\mathcal{G}}})- \theta $}
		\label{alg:filter:filterGoals}
        \EndFunction
    \end{algorithmic}
\end{algorithm}
}
\newpage
\begin{example}\label{exemp:filterCandidateGoals}
Consider the \textsc{Blocks-World} example shown in {\normalfont Figure~\ref{fig:blocksExample}} and that the following actions have been observed in the plan execution: {\normalfont \pred{(unstack E A)}} and {\normalfont \pred{(stack E D)}}. 
Using $\theta=0$, {\normalfont Algorithm~\ref{alg:filterCandidateGoals}} returns just the goal {\normalfont \texttt{RED}} because this goal has achieved 6 out of 10 fact landmarks, so it is the goal in the set of candidate goals with the highest percentage of achieved landmarks in observations. 
From the first observed action {\normalfont \pred{(unstack E A)}}, the algorithm computes in its preconditions the following fact landmark:

{
\normalfont
\begin{itemize}
	\item \textit{fact landmarks in preconditions}: \pred{[(on E A) (clear E) \\ (handempty)]};
\end{itemize}
}
Subsequently, the second observed action {\normalfont \pred{(stack E D)}} has in its preconditions and effects the following fact landmarks:
{
\normalfont
\begin{itemize}
	\item \textit{fact landmarks in preconditions}: \pred{[(clear D) (holding E)]}; and
	\item \textit{fact landmarks in effects}: \pred{[(on E D)]} (\textit{which is also a sub-goal});
\end{itemize}
}
From the initial state, it is also possible to compute the following set of achieved fact landmarks:
{
\normalfont
\begin{itemize}
	\item \pred{[(clear R) (ontable R) (handempty)]};
	\item \pred{[(clear D) (on D B) (handempty)]};
	\item \pred{[(clear R)]} (\textit{which is also a sub-goal});
\end{itemize}
}

Thus, the estimated percentage of achieved fact landmarks for the goal {\normalfont \texttt{RED}} is 60\%, because it has achieved 6 out of 10 fact landmarks. Note that we consider sub-goals like, such as {\normalfont \pred{(clear R)}} and {\normalfont \pred{(on E D)}}, as independent fact landmarks. 
Although all sub-goals of a goal must be true together for achieving the goal, in our filtering method we use them separately to estimate the percentage of achieved landmarks.

By contrast, for goals {\normalfont \pred{BED}} and  {\normalfont \texttt{SAD}}, the observed actions allow the filtering method to conclude that, respectively, 5 out of 10 and 5 out of 11 fact landmarks have been achieved for these goals. 
Thus, the estimated percentage of achieved fact landmarks for the {\normalfont \texttt{BED}} is 50\%, and for {\normalfont \texttt{SAD}} is 45\%.
From the evidence of fact landmarks in observations {\normalfont \pred{(unstack E A)}} and {\normalfont \pred{(stack E D)}}, {\normalfont Figures~\ref{fig:RED-AchievedLandmarks}, \ref{fig:BED-AchievedLandmarks}}, and~{\normalfont \ref{fig:SAD-AchievedLandmarks}} show the achieved fact landmarks for the candidate goals {\normalfont \pred{RED}, \pred{BED},} and {\normalfont \pred{SAD}}. 
Boxes in dark gray denote fact landmarks that have been achieved in observations.	
\end{example}

\begin{figure*}[h!]
	\centering
	\includegraphics[width=0.85\linewidth]{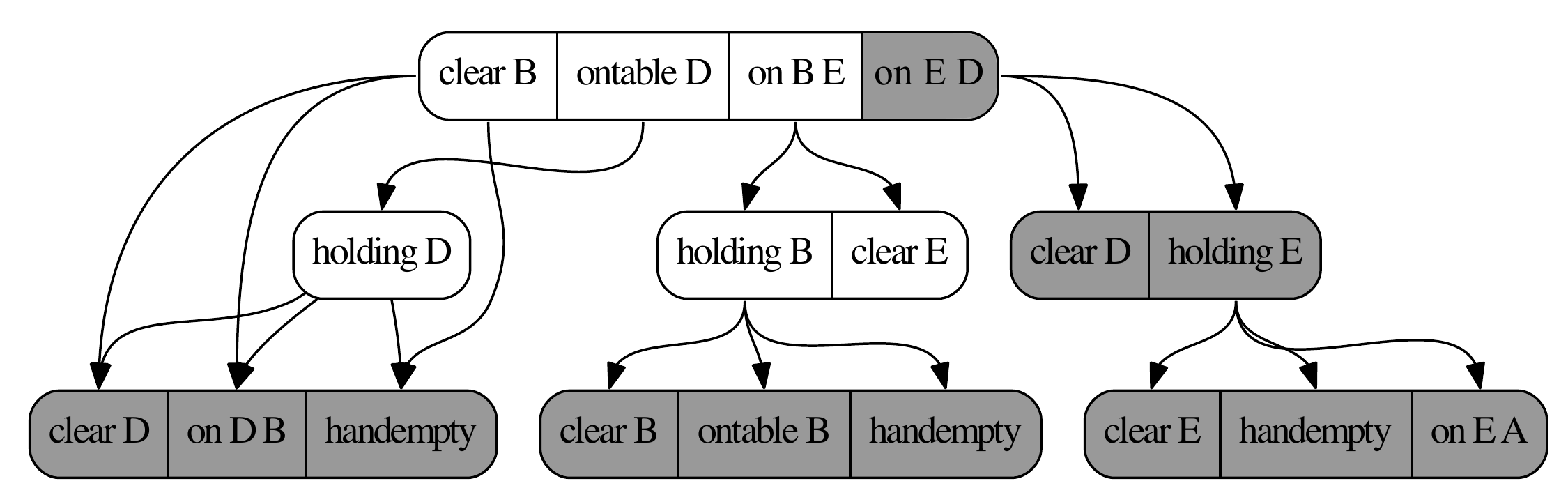}
	\caption{Fact landmarks for the word BED. Boxes in dark gray show achieved fact landmarks from the observed actions \pred{(unstack E A)} and \pred{(stack E D)}.}
	\label{fig:BED-AchievedLandmarks}
\end{figure*}

\begin{figure*}[h!]
	\centering
	\includegraphics[width=1\linewidth]{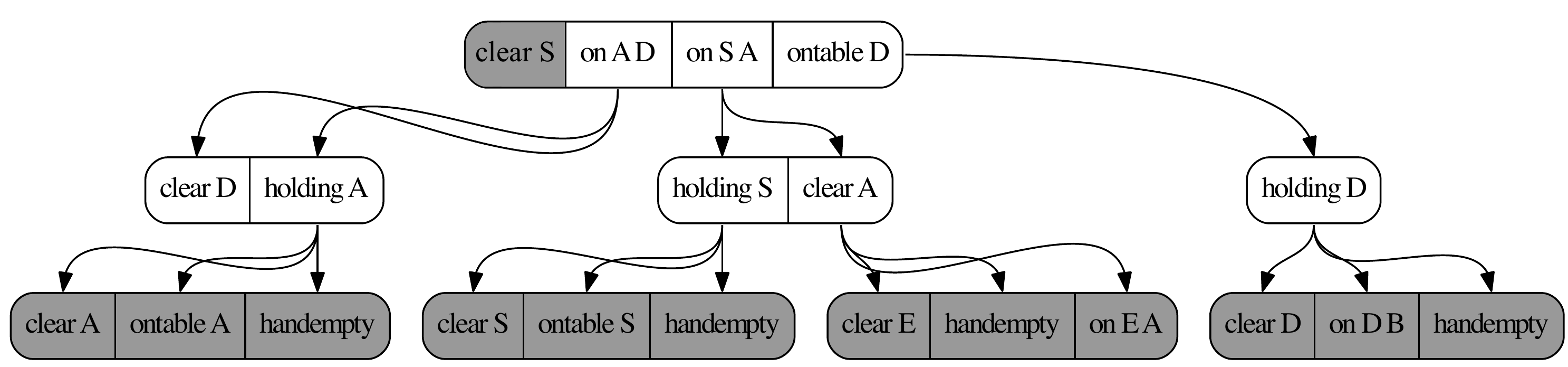}
	\caption{Fact landmarks for the word SAD. Boxes in dark gray show achieved fact landmarks from the observed actions \pred{(unstack E A)} and \pred{(stack E D)}.}
	\label{fig:SAD-AchievedLandmarks}
\end{figure*}

Example~\ref{exemp:filterCandidateGoals} does not show the  real impact of using the set of partition facts (Section~\ref{subsec:factparitioning}) in our filtering method. 
However, we argue that the evidence of such partitions in observations can immediately prune impossible candidate goals, avoiding the computation of achieved landmarks for such goal, improving the recognition time. 
We also show in Section~\ref{section:ExperimentsAndEvaluation} that our filtering method can be used with other planning-based goal and plan recognition approaches~\cite{RamirezG_IJCAI2009,RamirezG_AAAI2010,Sohrabi_IJCAI2016}, improving significantly the recognition time for all domains and problems, by reducing the number of calls to a planner (or heuristic).

The complexity of filtering candidate goals (Algorithm~\ref{alg:filterCandidateGoals}) in the worst case, including the process of extracting landmarks ($EL$) and fact partitions ($FP$) is: $O(|\mathcal{G}|\cdot EL \cdot FP \cdot |O|\cdot|\mathcal{L}_{\mathcal{G}}|)$,
where $\mathcal{G}$ is the set of candidate goals, $O$ is the observation sequence, and $\mathcal{L}_{\mathcal{G}}$ is the extracted landmarks for $\mathcal{G}$. Classifying facts into partitions is a simple iteration over the set of instantiated actions $\mathcal{A}$.

\begin{proposition}[\textbf{Soundness of the Goal Filtering Algorithm}]\label{pro:filter_correctness}
Let $T_{GR} = \langle\Xi, \mathcal{I}, \mathcal{G}, O\rangle$ be a goal recognition problem with candidate goals $\mathcal{G}$, a complete and noiseless observation sequence $O = \langle o_1, o_2, ..., o_n\rangle$. If $G^{*} \in \mathcal{G}$ is the correct hidden goal, then, for any landmark extraction algorithm that generates fact landmarks $\mathcal{L}_{\mathcal{G}}$ and computed landmarks $\Lambda_{\mathcal{G}}$, function \Call{FilterCandidateGoals}{$\Xi, \mathcal{I}, \mathcal{G}, O, \theta$} never filters out $G^{*}$ for any threshold $\theta$.
\end{proposition}

\begin{proof}\label{pf:filter_correctness}
	The proof of this proposition depends on two conditions: first that the reasoning performed over fact partitions never discards $G^{*}$; and second, that the ranking using the percentage of achieved landmarks always ranks $G^{*}$ highest (with possible ties). 
	
	The first property then relies on three conditions, namely that we never discard the true goal reasoning about: \textit{Strictly Activating} facts $F_{sa}$ (from Definition~\ref{def:strictlyActivating}), \textit{Unstable Activating} facts $F_{ua}$ (from Definition~\ref{def:unstableActivating}), and \textit{Strictly Terminal} facts $F_{st}$ (from Definition~\ref{def:strictlyTerminal}). 
	Since we only eliminate goals whose landmarks are \textit{Strictly Activating} ($F_{sa}$) that \textbf{are not} in the initial state $\mathcal{I}$, this condition only eliminates goals for which there is no possible plan from the initial state. 
	By Definition~\ref{def:goalRecognition}, $O$ must correspond to a plan from $\mathcal{I}$ that achieves $G^{*}$, so, if any landmark of $G^{*}$ is \textit{Strictly Activating}, it must be in $\mathcal{I}$. 
	Similarly, we only eliminate goals whose landmarks $\mathcal{L}$ are \textit{Unstable Activating} ($F_{ua}$) and \textit{Strictly Terminal} ($F_{st}$) if they are part of the preconditions or effects of the observations that occur before $l$ is needed (\idest, that they have become false throughout the plan before they were needed). 
	Since, landmarks $\mathcal{L}$ are necessary conditions, then any valid plan from $\mathcal{I}$ to $G^{*}$ must only delete $\mathcal{L}$ after they are needed, and thus only goals $G \in \mathcal{G}$ for which observation $O$ is not a valid plan can be discarded.
	The second property follows from Theorem~\ref{thm:hgc_soundness}.
\end{proof}

As a consequence of Proposition~\ref{pro:filter_correctness}, the filtering mechanism can do no worse than the goal recognition algorithm that uses the results of \textsc{FilterCandidateGoals} as candidate goals in full observability. 
Indeed the empirical results of Section~\ref{subsec:goalrecognitionResults} corroborate this theoretical result, as the accuracy for the filtered version of the Ram{\'{\i}}rez and Geffner~(\citeyear{RamirezG_IJCAI2009}) algorithm is strictly superior to the algorithm alone for full observability. 

\subsection{Landmark-Based Goal Completion Heuristic}\label{subsec:goalCompletionHeuristic}

We now describe a goal recognition heuristic that estimates the percentage of completion of a goal based on the number of landmarks that have been detected, and are required to achieve that goal~\cite{RamonNirMeneguzzi_AAAI2017}. 
This estimate represents the percentage of sub-goals in a goal that have been accomplished based on the evidence of achieved fact landmarks in the observations. 
We note that a candidate goal is composed of sub-goals comprised of the atomic facts that are part of a conjunction of facts in the goal definition.

Our heuristic method estimates the percentage of completion towards a goal by using the set of achieved fact landmarks computed by Algorithm~\ref{alg:ComputeAchievedLandmarks} (\textsc{ComputeAchievedLandmarks}).
More specifically, this heuristic operates by aggregating the percentage of completion of each sub-goal into an overall percentage of completion for all facts of a candidate goal. 
We denote this heuristic as $\mathit{h_{gc}}$, and it is formally defined by Equation~\ref{eq:heuristic}, where $\mathcal{AL}_{g}$ is the number of achieved landmarks from observations of every sub-goal $g$ of the candidate goal $G$ in $\mathcal{AL}_{G}$, and $\mathcal{L}_{g}$ represents the number of necessary landmarks to achieve every sub-goal $g$ of $G$ in $\mathcal{L}_{G}$.

\begin{equation}
\label{eq:heuristic}
h_{gc}(G, \mathcal{AL}_{G}, \mathcal{L}_{G}) = \left(\frac{\sum_{g \in G} \frac{|\mathcal{AL}_{g} \in \mathcal{AL}_{G} |}{|\mathcal{L}_{g} \in \mathcal{L}_{G}|}}{ |G| }\right)
\end{equation}

Thus, heuristic $\mathit{h_{gc}}$ estimates the completion of a goal $G$ by calculating the ratio between the sum of the percentage of completion for every sub-goal $g \in G$, \idest, $\sum_{g \in G} \frac{|\mathcal{AL}_{g} \in \mathcal{AL}_{G} |}{|\mathcal{L}_{g} \in \mathcal{L}_{G}|}$, and the size $|G|$ of the set of sub-goals, that is, the number of sub-goals in $G$. 
Algorithm~\ref{alg:RecognizeHeuristicCompletion} describes how to recognize goals using the $\mathit{h_{gc}}$ heuristic and takes as input a goal recognition problem $T_{GR}$, as well as a threshold value $\theta$. 
The $\theta$ threshold gives us flexibility to avoid eliminating candidate goals whose the percentage of goal completion are close to the highest completion value.
In Line~\ref{alg:Algo2:ExtractLandmarks}, the algorithm uses the \textsc{ExtractLandmarks} function to extract fact landmarks for all candidate goals. 
By taking as input the initial state $\mathcal{I}$, the observations $O$, and the extracted landmarks $\mathcal{L}_{\mathcal{G}}$, in Line~\ref{alg:Algo2:ComputeAchievedLandmarks}, our algorithm first computes the set of achieved landmarks $\Lambda_{\mathcal{G}}$ for every candidate goal using Algorithm~\ref{alg:ComputeAchievedLandmarks}. 
Finally, the algorithm uses the heuristic $\mathit{h_{gc}}$ to estimate goal completion for every candidate $G$ in $\mathcal{G}$, and as output (Line~\ref{alg:Algo2:ReturnHeuristicCompletion}), the algorithm returns those candidate goals with the highest estimated value within the threshold $\theta$. Example~\ref{exemp:goalCompletionHeuristic} shows how heuristic $\mathit{h_{gc}}$ estimates the completion of a candidate goal.

\floatname{algorithm}{Algorithm}
\begin{algorithm}[h!]
    \caption{Recognize Goals using the Goal Completion Heuristic $\mathit{h_{gc}}$.} 
    \textbf{Input:} $\Xi$ \textit{planning domain definition}, $\mathcal{I}$ \textit{initial state}, $\mathcal{G}$ \textit{set of candidate goals}, $O$ \textit{observations}, and $\theta$ \emph{threshold}.
    \\\textbf{Output:} \textit{Recognized goal(s).}
    \label{alg:RecognizeHeuristicCompletion}
    \begin{algorithmic}[1]
        \Function{Recognize}{$\Xi, \mathcal{I}, \mathcal{G}, O, \theta$}
        		\State $\mathcal{L}_{\mathcal{G}} \gets$ \textsc{ExtractLandmarks}($\Xi, \mathcal{I}, \mathcal{G}$)\label{alg:Algo2:ExtractLandmarks}
        		\State $\Lambda_{\mathcal{G}} \gets$ \textsc{ComputeAchievedLandmarks}($\mathcal{I}, \mathcal{G}, O, \mathcal{L}_{\mathcal{G}}$)\label{alg:Algo2:ComputeAchievedLandmarks}
        		\State $\mathit{maxh} \gets \displaystyle \max_{G' \in \mathcal{G}} h_{gc}(G',\Lambda_{\mathcal{G}}(G'), \mathcal{L}_{\mathcal{G}}(G'))$
        		\State \textbf{return} {all $G$ s.t $G \in \mathcal{G}$ and \newline {\phantom{return }} $\mathit{h_{gc}}(G, \Lambda_{\mathcal{G}}(G), \mathcal{L}_{\mathcal{G}}(G)) \geq (\mathit{maxh} - \theta)$}\label{alg:Algo2:ReturnHeuristicCompletion}
        \EndFunction
    \end{algorithmic}
\end{algorithm}

\newpage
\begin{example}\label{exemp:goalCompletionHeuristic}
As an example of how heuristic $\mathit{h_{gc}}$ estimates goal completion of a candidate goal, recall the \textsc{Blocks-World} example from {\normalfont Figure~\ref{fig:blocksExample}}. Consider that among these candidate goals {\normalfont (\pred{RED}, \pred{BED},} and {\normalfont \pred{SAD})} the correct hidden goal is {\normalfont \pred{RED}}, and we observe the following partial sequence of actions: {\normalfont \pred{(unstack E A)}} and {\normalfont \pred{(stack E D)}}. Thus, based on the achieved landmarks $\mathcal{AL}_{\texttt{RED}}$ computed using {\normalfont Algorithm~\ref{alg:ComputeAchievedLandmarks}} {\normalfont (Figure~\ref{fig:RED-AchievedLandmarks})}, our heuristic $\mathit{h_{gc}}$ estimates that the percentage of completion for the goal {\normalfont \texttt{RED}} is 0.66: {\normalfont \pred{(clear R)} = $\frac{1}{1}$ $+$ \pred{(on E D)} = $\frac{3}{3}$ $+$ \pred{(on R E)} = $\frac{1}{3}$ $+$ \pred{(ontable D)} = $\frac{1}{3}$}, and hence, $\frac{2.66}{4}$ = 0.66. For the words {\normalfont \pred{BED}} and {\normalfont \pred{SAD}} our heuristic $\mathit{h_{gc}}$ estimates respectively, 0.54 and 0.58.	
\end{example}

Besides extracting landmarks for every candidate goal ($EL$), our landmark-based goal completion approach iterates over the set of candidate goals $\mathcal{G}$, the observations sequence $O$, and the extracted landmarks $\mathcal{L}_{\mathcal{G}}$. 
The heuristic computation of $\mathit{h_{gc}}$ ($HC$) is linear on the number of fact landmarks. Thus, the complexity of this approach is: $O(EL + |\mathcal{G}|\cdot|O|\cdot|\mathcal{L}_{\mathcal{G}}| + HC)$. 
Finally, the goal ranking based on $\mathit{h_{gc}}$ always ensures (in full observability) that the correct goal ranks highest (\idest, it is sound), with possible ties, as stated in Theorem~\ref{thm:hgc_soundness}. 

\begin{theorem}[\textbf{Soundness of the $h_{gc}$ Goal Recognition Heuristic}]\label{thm:hgc_soundness}
Let $T_{GR} = \langle\Xi,\mathcal{I} ,\mathcal{G}, O\rangle$ be a goal recognition problem with candidate goals $\mathcal{G}$ such that $\forall{G_1}\forall{G_2}\in\mathcal{G} (G_1 \not\subset G_2)$, a complete and noiseless observation sequence $O = \langle o_1, o_2, ..., o_n\rangle$. 
If $G^{*} \in \mathcal{G}$ is the correct hidden goal, then, for any landmark extraction algorithm that generates fact landmarks $\mathcal{L}_{\mathcal{G}}$ and computed landmarks $\Lambda_{\mathcal{G}}$, the estimated value of $\mathit{h_{gc}}$ will always be highest for the correct hidden goal $G^{*}$, \idest, $\forall G \in \mathcal{G}$ it is the case that $\mathit{h_{gc}}(G^{*}, \Lambda_{\mathcal{G}}(G^{*}), \mathcal{L}_{\mathcal{G}}(G^{*})) \geq \mathit{h_{gc}}(G, \Lambda_{\mathcal{G}}(G), \mathcal{L}_{\mathcal{G}}(G)) $.
\end{theorem}

\begin{proof}\label{pf:hgc_soundness}
The proof is straightforward from the definition of fact landmarks ensuring they are necessary conditions to achieve a goal $G$ and that all facts $g \in G$ are necessary. 
Let us first assume that any pair of goals $G_1,G_2 \in \mathcal{G}$ are different, \idest, $G_1 \cap G_2 \neq \emptyset$, and that no action $a$ in the domain $\Xi$ achieves facts that are in any pair of goals simultaneously. 
Since any landmark extraction algorithm includes all facts $g \in G_1$ as landmarks for a goal $G_1$, then, for every other goal $G_2$, there exists at least one fact $g$ such that $g \in G_1 \land g \not\in G_2$ that sets it apart from  $G_2$. 
Under these circumstances, an observation sequence $O$ for the correct goal $G^{*}$ will have achieved a set of landmarks $\Lambda_{\mathcal{G}}(G^{*})$ that is exactly the same as the complete computed set of landmarks $\mathcal{L}_{\mathcal{G}}(G^{*})$ for $G^{*}$. 
Hence $\mathit{h_{gc}}(G^{*}, \Lambda_{\mathcal{G}}(G^{*}), \mathcal{L}_{\mathcal{G}}(G^{*})) = 1$, and $\mathit{h_{gc}}(G, \Lambda_{\mathcal{G}}(G), \mathcal{L}_{\mathcal{G}}(G)) < 1$ for any other goal $G \in \mathcal{G}$, since the numerator of the $\mathit{h_{gc}}$ computation will be missing fact $g$ for $G$ as $g$ is not a landmark of $G$. 
If we drop the assumption about the actions not achieving facts simultaneously in any pair of goals or that goals are identical, it is possible that $\mathit{h_{gc}}(G^{*}, \Lambda_{\mathcal{G}}(G^{*}), \mathcal{L}_{\mathcal{G}}(G^{*})) = \mathit{h_{gc}}(G, \Lambda_{\mathcal{G}}(G), \mathcal{L}_{\mathcal{G}}(G)) = 1$, which still ensures that the under $\mathit{h_{gc}}$, $G^{*}$ always ranks at the top, possibly tied with other goals.
\end{proof}

Thus, our goal completion heuristic is sound under full observability in the sense that it can never rank the wrong goal higher than the correct goal when we observe the landmarks. 
We note that there is one specific case when our landmark approach can provide wrong rankings, but which we explicitly exclude from the theorem, which is when the set of candidate goals contains two goals such that one is a sub-goal of the other (\idest, $G_1, G_2 \in\mathcal{G} (G_1 \subset G_2)$). 
In this case, any kind of ``distance'' to goal metric will report $G_1$ as being more likely than $G_2$ until the observations take the observed agent past $G_1$ and closer to $G_2$ than $G_1$. 
We close this section by commenting on the effect of landmark orderings in the accuracy of the heuristic. 
Specifically, although we do use the landmark order to infer the achievement of necessary prior landmarks that were not observed in partially observable environments, our heuristic itself does not consider the actual ordering of the heuristics. 
We infer prior landmarks to obtain more landmarks when we deal with partial observability.
Nevertheless, we have experimented with different scoring mechanisms to account for landmarks having been observed in the expected order or not, and these showed almost no advantage over the current heuristic. 
Consequently, although there are various different algorithms that generate better landmark orderings~\cite{Hoffmann2004_OrderedLandmarks}, the way in which we use the landmarks does not seem to be affected by more or less accurate landmark orderings. 

There are two additional properties provable for our $h_{gc}$ heuristic, first, given how our heuristic accounts for landmarks, the value of this heuristic is strictly increasing. 

\begin{proposition}[\bf Monotonicity of $h_{gc}$] The value of $h_{gc}$ is monotonically (non-strictly) increasing in the observation sequence.
\end{proposition}
\begin{proof}
By definition, $\mathcal{AL}_G$ is monotonically increasing, while all other values in $h_{gc}$ remain constant. 
Therefore from Equation \ref{eq:heuristic}, it is clear that $h_{gc}$ must increase.
\end{proof}

Further, a corollary of Theorem~\ref{thm:hgc_soundness} is that, under full observation, only the correct goal can reach a heuristic value of $1$. 
This also illustrates why we restrict our theorems to settings where candidate goals are not subgoals of each other. 
Consider a goal to be at position $d$, and another to be at position $g$, with landmarks $a,b,c,d,e,f,g$, since $d$ itself is a landmark of $g$, $d$ is implicitly a subgoal of $g$. 
If we observe all landmarks in an observation, then $h_{gc}(d)=\frac{4}{4}=1$, and $h_{gc}(g)=\frac{7}{7}=1$, which leads to Corollary~\ref{cor:thereCanBeOnly1}. 

\begin{corollary}\label{cor:thereCanBeOnly1} If the goal being monitored has no subgoals being monitored under full observability, then $h_{gc}=1$ iff the goal the heuristic is monitoring has been achieved.
\end{corollary}

\begin{proof}
$h_{gc}=1$ when $\frac{\sum_g \frac{|\mathcal{AL}_G|}{\mathcal{L}}}{|G|}=1$, which can only occur when $|\mathcal{AL}_g|=|\mathcal{L}_g|$ for all $g \in G$. This clearly occurs when the goal being monitored is achieved. However, if the heuristic is also monitoring a subgoal, then this condition can be satisfied for the subgoal, hence the exception in the proposition.
\end{proof}

\subsection{Landmark-Based Uniqueness Heuristic}
\label{subsec:uniquenessHeuristic}

Many goal recognition problems contain multiple candidate goals that share common fact landmarks, generating ambiguity for our previous approaches. 
Clearly, landmarks that are common to multiple candidate goals are less useful for recognizing a goal than landmarks that exist for only a single goal. 
As a consequence, computing how unique (and thus informative) each landmark is can help disambiguate similar goals for a set of candidate goals. 
We now develop a second goal recognition heuristic based on this intuition. 
To develop this heuristic, we introduce the concept of \textit{landmark uniqueness}, which is the inverse frequency of a landmark among the landmarks found in a set of candidate goals~\cite{RamonNirMeneguzzi_AAAI2017}. 
For example, consider a landmark $L$ that occurs only for a single goal within a set of candidate goals; the uniqueness value for such a landmark is intuitively the maximum value of 1. 
Equation~\ref{eq:LandmarksUniqueness} formalizes this intuition, describing how the \textit{landmark uniqueness value} is computed for a landmark $L$ and a set of landmarks for goals $\mathcal{L}_{\mathcal{G}}$.

Using the \textit{landmark uniqueness value}, we estimate which candidate goal is the intended one by summing the uniqueness values of the landmarks achieved in the observations. 
Unlike our previous heuristic, which estimates progress towards goal completion by analyzing sub-goals and their achieved landmarks, the landmark-based uniqueness heuristic estimates the goal completion of a candidate goal $G$ by calculating the ratio between the sum of the uniqueness value of the achieved landmarks of $G$ and the sum of the uniqueness value of all landmarks of $G$. 
This algorithm effectively weighs the completion value by the informational value of a landmark so that unique landmarks have the highest weight. 
To estimate goal completion using the landmark uniqueness value, we calculate the uniqueness value for every extracted landmark in the set of landmarks of the candidate goals using Equation~\ref{eq:LandmarksUniqueness}.
This computes the landmark uniqueness value of every landmark $L$ of $\mathcal{L}_{\mathcal{G}}$ and store it into $\Upsilon_{uv}$. 
This heuristic is denoted as $\mathit{h_{uniq}}$ and formally defined in Equation~\ref{eq:HeuristicLandmarksUniqueness}.

\begin{equation}
\label{eq:LandmarksUniqueness}
L_{\mathit{Uniq}}(L, \mathcal{L}_{\mathcal{G}}) = \left(\frac{1}{\displaystyle\sum_{\mathcal{L} \in \mathcal{L_G}} |\{L |L \in \mathcal{L}\}|}\right)
\end{equation}
\vspace{2mm}
\begin{equation}
\label{eq:HeuristicLandmarksUniqueness}
h_{\mathit{uniq}}(G, \mathcal{AL}_{G}, \mathcal{L}_{G}, \Upsilon_{uv}) = \left(
\frac
{\displaystyle\sum_{\mathcal{A}_{L} \in \mathcal{AL}_{G}}\Upsilon_{uv}(\mathcal{A}_{L})}
{\displaystyle\sum_{L \in \mathcal{L}_{G}}\Upsilon_{uv}(L)}\right)
\end{equation}

Algorithm~\ref{alg:RecognizeHeuristicUniqueness} formalizes a goal recognition function that uses the $\mathit{h_{uniq}}$ heuristic. 
This algorithm takes as input the same parameters as the previous approach: a goal recognition problem and a threshold $\theta$. 
Like Algorithm~\ref{alg:ComputeAchievedLandmarks}, this algorithm extracts the set of landmarks for all candidate goals from the initial state $\mathcal{I}$, stores them in $\mathcal{L}_{\mathcal{G}}$ (Line~\ref{alg:Algo3:ExtractLandmarks}), and computes the set of achieved landmarks based on the observations, storing these in $\Lambda_{\mathcal{G}}$. 
Unlike Algorithm~\ref{alg:RecognizeHeuristicCompletion}, in Line~\ref{alg:Algo3:ComputeLandmarkUniquenessValue} this algorithm computes the landmark uniqueness value for every landmark $L$ in $\mathcal{L}_{\mathcal{G}}$ and stores it into $\Upsilon_{uv}$. 
Finally, using these computed structures, the algorithm recognizes which candidate goal is being pursued from observations using the heuristic $\mathit{h_{uniq}}$, returning those candidate goals with the highest estimated value within the $\theta$ threshold. Example~\ref{exemp:uniquenessHeuristic} shows how heuristic $\mathit{h_{uniq}}$ uses the concept of landmark uniqueness value to goal recognition.

\floatname{algorithm}{Algorithm}
\begin{algorithm}[h!]
    \caption{Recognize Goals using the Landmark Uniqueness Heuristic $\mathit{h_{uniq}}$.} 
    \textbf{Input:} $\Xi$ \textit{planning domain definition}, $\mathcal{I}$ \textit{initial state}, $\mathcal{G}$ \textit{set of candidate goals}, $O$ \textit{observations}, and $\theta$ \emph{threshold}.
    \\\textbf{Output:} \textit{Recognized goal(s).}
    \label{alg:RecognizeHeuristicUniqueness}
    \begin{algorithmic}[1]
        \Function{Recognize}{$\Xi, \mathcal{I}, \mathcal{G}, O, \theta$}
        		\State $\mathcal{L}_{\mathcal{G}} \gets$ \textsc{ExtractLandmarks}($\Xi, \mathcal{I}, \mathcal{G}$)\label{alg:Algo3:ExtractLandmarks}
				\State $\Lambda_{\mathcal{G}} \gets$ \textsc{ComputeAchievedLandmarks}($\mathcal{I},\mathcal{G},O,\mathcal{L}_{\mathcal{G}}$)\label{alg:Algo3:ComputeAchievedLAndmarks}
        		\State $\Upsilon_{uv} \gets \langle\rangle$ \Comment{\textit{Map of landmarks to their uniqueness value}.}
        		\For{\textbf{each} fact landmark $L$ in $\mathcal{L}_{\mathcal{G}}$}
        			\State $\Upsilon_{uv}(L) \gets L_{Uniq}(L,\mathcal{L}_{\mathcal{G}})$ \label{alg:Algo3:ComputeLandmarkUniquenessValue}
        		\EndFor
			\State $\mathit{maxh} \gets \displaystyle \max_{G' \in \mathcal{G}} h_{uniq}(G',\Lambda_{\mathcal{G}}(G'), \mathcal{L}_{\mathcal{G}}(G'), \Upsilon_{uv})$
			\State \textbf{return} {all $G$ s.t $G \in \mathcal{G}$ and \newline {\phantom{return }} $h_{uniq}(G, \Lambda_{\mathcal{G}}(G), \mathcal{L}_{\mathcal{G}}(G),\Upsilon_{uv}) \geq (\mathit{maxh} - \theta)$}
        \EndFunction
    \end{algorithmic}
\end{algorithm}

\begin{example}\label{exemp:uniquenessHeuristic}
Recall the \textsc{Blocks-World} example from {\normalfont Figure~\ref{fig:blocksExample}} consider the following observed actions: {\normalfont \pred{(unstack E A)}} and {\normalfont \pred{(stack E D)}}. 
{\normalfont Listing~\ref{lst:FactLandmarksUniquenessValue}} shows the set of extracted fact landmarks for the candidate goals in the \textsc{Blocks-World} example and their respective uniqueness value. 
Based on the set of achieved landmarks (shown in bold in {\normalfont Listing~\ref{lst:FactLandmarksUniquenessValue}}), our heuristic $\mathit{h_{uniq}}$ estimates the following percentage for each candidate goal: $\mathit{h_{uniq}}${\normalfont (\pred{RED})} = $\frac{3.66}{6.33}$ = 0.58; $\mathit{h_{uniq}}${\normalfont (\pred{BED})} = $\frac{2.66}{6.33}$ = 0.42; and $\mathit{h_{uniq}}${\normalfont (\pred{SAD})} = $\frac{3.66}{8.33}$ = 0.44. 
In this case, {\normalfont Algorithm~\ref{alg:RecognizeHeuristicUniqueness}} correctly estimates {\normalfont \pred{RED}} to be the intended goal since it has the highest heuristic value.	
\end{example}

\begin{lstlisting}[float=!t,caption={Extracted fact landmarks for the \textsc{Blocks-World} example in Figure~\ref{fig:blocksExample} and their respective uniqueness value.},label={lst:FactLandmarksUniquenessValue},basicstyle=\ttfamily\footnotesize]
- (and (clear B) (on B E) (on E D) (ontable D)) = 6.33
  (*\bfseries[(on E D)] = 0.5*), (*\bfseries[(clear D) (holding E)] = 0.5*),
  (*\bfseries[(on E A) (clear E) (handempty)] = 0.33*), [(ontable D)] = 0.33,
  (*\bfseries[(on D B) (clear D) (handempty)] = 0.33*), [(holding D)] = 0.33,
  (*\bfseries[(clear B) (ontable B) (handempty)] = 1.0*), [(on B E)] = 1.0,
  [(clear B)] = 1.0, [(clear E) (holding B)] = 1.0

- (and (clear S) (on S A) (on A D) (ontable D)) = 8.33
  (*\bfseries[(clear S)] = 1.0*), [(on A D)] = 1.0, [(on S A)] = 1.0,
  (*\bfseries[(clear A) (ontable A) (handempty)] = 1.0*), [(ontable D)] = 0.33,
  (*\bfseries[(clear S) (ontable S) (handempty)] = 1.0*), [(holding D)] = 0.33,
  (*\bfseries[(on E A) (clear E) (handempty)] = 0.33*),
  (*\bfseries[(on D B) (clear D) (handempty)] = 0.33*),
  [(clear A) (holding S)] = 1.0, [(clear D) (holding A)] = 1.0

- (and (clear R) (on R E) (on E D) (ontable D)) = 6.33
  (*\bfseries[(clear R)] = 1.0*), (*\bfseries[(clear R) (ontable R) (handempty)] = 1.0*),
  (*\bfseries[(clear D) (holding E)] = 0.5*), (*\bfseries[(on E D)] = 0.5*),
  (*\bfseries[(on E A) (clear E) (handempty)] = 0.33*), [(ontable D)] = 0.33,
  (*\bfseries[(on D B) (clear D) (handempty)] = 0.33*), [(holding D)] = 0.33,
  [(on R E)] = 1.0, [(clear E) (holding R)] = 1.0
\end{lstlisting}

Similar to our landmark-based goal completion approach, this approach iterates over the set of candidate goals $\mathcal{G}$, the observations sequence $O$, and the extracted landmarks $\mathcal{L}_{\mathcal{G}}$. 
However, for this approach we compute the uniqueness value ($CLUniq$) for every extracted landmarks, which is linear on the number of landmarks.  
The heuristic computation of $\mathit{h_{uniq}}$ ($HC$) is also linear on the number of fact landmarks. 
Thus, the complexity of this approach is: $O(EL + |\mathcal{G}|\cdot|O|\cdot|\mathcal{L}_{\mathcal{G}}| + CLUniq + HC)$.
Finally, since this is just a weighted version of the $\mathit{h_{gc}}$ heuristic, it follows trivially from Theorem~\ref{thm:hgc_soundness} that, for full observations, $\mathit{h_{uniq}}$ always ranks the correct goal $G^{*}$ highest.

\begin{corollary}[\textbf{Correctness of $\mathit{h_{uniq}}$ Goal Recognition Heuristic}]\label{cor:huniq_correct}
Let $T_{GR} = \langle\Xi,\mathcal{I} ,\mathcal{G}, O\rangle$ be a goal recognition problem with candidate goals $G \in \mathcal{G}$, a complete and noiseless observation sequence $O = \langle o_1, o_2, ..., o_n\rangle$. If $G^{*} \in \mathcal{G}$ is the correct goal, then, for any landmark extraction algorithm that generates fact landmarks $\mathcal{L}_{\mathcal{G}}$ and computed landmarks $\Lambda_{\mathcal{G}}$, the estimated value of $\mathit{h_{uniq}}$ will always be highest for the correct goal $G^{*}$, \idest, $\forall G \in \mathcal{G} \left(\mathit{h_{uniq}}(G^{*}, \Lambda_{\mathcal{G}}(G^{*}), \mathcal{L}_{\mathcal{G}}(G^{*})) \geq \mathit{h_{uniq}}(G, \Lambda_{\mathcal{G}}(G), \mathcal{L}_{\mathcal{G}}(G)) \right)$.
\end{corollary}

\section{Experiments and Evaluation}\label{section:ExperimentsAndEvaluation}

In this section, we describe the experiments and evaluation we carried out on our goal recognition approaches. 
We start with a description of the planning domains and the datasets, as well as the metrics we use to evaluate our approaches in Section~\ref{subsec:metrics}. 
Section~\ref{subsec:goalrecognitionResults} then details the experiments and evaluation results of our goal recognition approaches.

\subsection{Domains, Datasets, and Metrics}\label{subsec:domains}

We empirically evaluated our approaches using 15 domains from the planning literature\footnote{\scriptsize \url{http://ipc.icaps-conference.org}}.
Six of these domains are also used in the evaluation of other goal and plan recognition approaches~\cite{RamirezG_IJCAI2009,RamirezG_AAAI2010}\footnote{\scriptsize \url{https://sites.google.com/site/prasplanning/}}. We summarize these domains as follows. 

\begin{itemize}
	\item \textsc{Blocks-World} is a domain that consists of a set of blocks, a table, and a robot hand. Blocks can be stacked on top of other blocks or on the table. A block that has nothing on it is clear. The robot hand can hold one block or be empty. The goal is to find a sequence of actions that achieves a final configuration of blocks;
	
	\item \textsc{Campus} is a domain that consists of finding what activity is being performed by a student from his observations on a campus environment;

	\item \textsc{Depots} is a domain that combines transportation and stacking. For transportation, packages can be moved between depots by loading them on trucks. For stacking, hoists can stack packages on palettes or other packages. The goal is to move and stack packages by using trucks and hoists between depots;	
	
	\item \textsc{Driver-Log} is a domain that consists of drivers that can walk between locations and trucks that can drive between locations. Walking between locations requires traversal of different paths. Trucks can be loaded with or unloaded of packages. Goals in this domain consists of transporting packages between locations;
	
	\item \textsc{Dock-Worker-Robots (DWR)} is a domain that involves a number of cranes, locations, robots, containers, and piles, in which goals involve transporting containers to a final destination according to a desired order;
	
	\item \textsc{IPC-Grid} domain is a domain consists of an agent that moves in a grid from connected cells to others by transporting keys in order to open locked locations;
	
	\item \textsc{Ferry} is a domain that consists of set of cars that must be moved to desired locations using a ferry that can carry only one car at a time;
	
	\item \textsc{Intrusion-Detection} represents a domain where a hacker tries to access, vandalize, steal information, or  perform a combination of these attacks on a set of servers;
	
	\item \textsc{Kitchen} is a domain that consists of home-activities, in which the goals can be preparing dinner, breakfast, among others;
	
	\item \textsc{Logistics} is a domain which models cities, and each city contains locations. These locations are airports. For transporting packages between locations, there are two types o vehicles: trucks and airplanes. Trucks can drive between cities. Airplanes can fly between airports. The goal is to get and transport packages from locations to other locations;
	
	\item \textsc{Miconic} is a domain that involves transporting a number of passengers using an elevator to reach destination floors;
	
	\item \textsc{Rovers} is a domain that consists of a set of rovers that navigate on a planet surface in order to find samples and communicate experiments;
	
	\item \textsc{Satellite} is a domain that involves using one or more satellites to make observations, by collecting data and down-linking the data to a desired ground station;

	\item \textsc{Sokoban} is a domain that involves an agent whose goal is to push a set of boxes into specified goal locations in a grid with walls; and

	\item \textsc{Zeno-Travel} is a domain where passengers can embark and disembark onto aircraft that can fly at two alternative speeds between locations.
\end{itemize}

We formalize planning domains and problems using the STRIPS~\cite{STRIPSFikes1971} fragment of PDDL~\cite{PDDLMcdermott1998}. 
Based on the datasets provided by Ram{\'i}rez and Geffner~(\citeyear{RamirezG_IJCAI2009,RamirezG_AAAI2010}), which contain hundreds of goal recognition problems for 6 domains, we added non-trivial\footnote{A non-trivial planning problem contains a large search space (in terms of search branching factor and depth), and therefore, modern planners such as \textsc{Fast-Downward} takes up to 5-minutes to solve it. In our datasets, the number of instantiated (grounded) actions is  between 158 and 3258, and plan length is between 5 and 83.} and larger planning problems in their datasets and generated new datasets\footnote{\scriptsize \url{https://github.com/pucrs-automated-planning/goal_plan-recognition-dataset}} from the remaining 9 planning domains using open-source planners, such as \textsc{Fast-Downward}~\cite{HelmertFastDownward_2011}, \textsc{Fast-Forward}~\cite{FFHoffmann_2001}, and \textsc{LAMA}~\shortcite{RichterLPG_2010}, each of which is based on planning problems containing both optimal and sub-optimal plans of various sizes, including large problems to test the scalability of the approaches~\cite{Pereira2017_dataset}. 
We also generated datasets for 4 domains (\textsc{Campus}, \textsc{Intrusion}, \textsc{IPC-Grid}, and \textsc{Kitchen}) with missing, full, and noisy observations, which are the same domains that Sohrabi~\etal~use in~(\citeyear{Sohrabi_IJCAI2016}). 
The dataset for \textsc{Campus} domain with missing and noisy observations comes from Sohrabi~\etal~(\citeyear{Sohrabi_IJCAI2016})\footnote{\scriptsize \url{https://github.com/shirin888/planrecogasplanning-ijcai16-benchmarks}}. 
Thus, we evaluate our goal recognition approaches against the state-of-the-art not only using datasets with missing and full observations, but also using datasets with noisy observations in the same way as Sohrabi~\etal~(\citeyear{Sohrabi_IJCAI2016}).

\subsection{Evaluation Metrics}\label{subsec:metrics}

For evaluation of our goal recognition approaches against the state-of-the-art, we use the Accuracy metric (Equation~\ref{frac:accuracy}), which represents the fraction of times that the correct goal was among the goals found to be most likely, \idest, how well the correct hidden goal is recognized from a set of possible goals for a given goal recognition problem. 
Most goal recognition approaches~\cite{RamirezG_AAAI2010,NASA_GoalRecognition_IJCAI2015,Sohrabi_IJCAI2016} refer to this metric as quality, also denoted as Q.

\vspace{1mm}
\begin{equation}\label{frac:accuracy}
\textsc{Accuracy} = \cfrac{\sum True \enspace positive + \sum True \enspace negative}{\sum Total \enspace population}
\end{equation}
\vspace{1mm}

Like most goal and plan recognition approaches in the literature, we also evaluate the average number of returned goals, which is called as Spread in $\mathcal{G}$, and recognition time (speed), in seconds, representing the time that a goal recognition approach takes to recognize the most likely goal from a set of possible goals. 

\subsection{Goal Recognition Experimental Results}
\label{subsec:goalrecognitionResults}

Our experiments compare our goal recognition approaches and heuristics ($\mathit{h_{gc}}$ and $\mathit{h_{uniq}}$) to four other approaches. 
First, we use the heuristic estimator approach of Ram{\'{\i}}rez and Geffner~(\citeyear{RamirezG_IJCAI2009})\footnote{Ram{\'{\i}}rez and Geffner~(\citeyear{RamirezG_IJCAI2009}) developed a goal and plan recognition approach which uses a heuristic method to approximate the planning solution by computing a relaxed plan~\cite{Keyder_Heuristics_ECAI2008}. The authors show that this heuristic-based approach is their faster and more accurate goal and plan recognition approach.
}, denoted as R\&G 2009; as well as a combination of their approach and our filtering method with threshold $\theta = $ 10\%, denoted as \textsc{Filter}$_{10\%} +$~R\&G 2009. 
This is their fastest and most accurate algorithm\footnote{\scriptsize \url{https://sites.google.com/site/prasplanning/file-cabinet/plan-recognition.tar.bz2}}. 
Second, we use the probabilistic framework of Ram{\'{\i}}rez and Geffner~(\citeyear{RamirezG_AAAI2010})\footnote{\scriptsize \url{https://sites.google.com/site/prasplanning/file-cabinet/prob-plan-recognition.tar.bz2}} that allows the use of any off-the-shelf automated planner, denoted as R\&G 2010. 
The automated planner we used alongside this approach is \textsc{Fast-Downward}~\cite{HelmertFastDownward_2011} with the LM-Cut heuristic~\cite{helmert2009landmarks}, a planning heuristic that relies on landmarks to estimate the distance from a particular state to a goal state. 
We also use this approach with our filtering method (threshold $\theta = $ 10\%), denoted as \textsc{Filter}$_{10\%} +$~R\&G 2010. 
Third, we use the approach of Sohrabi~\etal~(\citeyear{Sohrabi_IJCAI2016})\footnote{Since the exact code from Sohrabi~\etal~(\citeyear{Sohrabi_IJCAI2016}) is unavailable, we developed our own version of this approach with some advice from the main author and the top-k planner she shared.}, which uses a top-K planner to extract multiple optimal and nearly optimal plans for a particular goal, denoted as IBM 2016\footnote{We ran experiments using a top-k planner rather than a diverse planner under advice from the main author.}. 
The automated planner we used alongside this approach is the most modern top-K planner TK$^{*}$~\cite{katz_etal_icaps18} with the LM-Cut heuristic, and the number of sampled plans parameter K=1000. 
These are exactly the same parameters that Sohrabi~\etal~used in the experiments and evaluation in~\cite{Sohrabi_IJCAI2016}.
Note that we use the LM-Cut heuristic~\cite{helmert2009landmarks} with the approaches from~\cite{RamirezG_AAAI2010} and~\cite{Sohrabi_IJCAI2016} because our goal recognition approaches proposed rely on landmarks.
The use of this landmark-based planning heuristic with the planners \textsc{Fast-Downward}~\cite{HelmertFastDownward_2011} and top-K planner TK$^{*}$~\cite{katz_etal_icaps18} aims to provide a fairer comparison against our landmark-based approaches. 
Finally, we compare our approaches against the approach of E-Mart\'{i}n~\etal~(\citeyear{NASA_GoalRecognition_IJCAI2015}), denoted as FGR 2015. 
Their goal recognition approach also obviates the use of calling a planner multiple times for the recognition process, and instead, uses a planning graph, resulting in fast goal recognition in planning settings. 
As advised by the authors, we use the FGR 2015 recognizer with interaction information equals to 0, which is their technique that yields the best results in terms of recognition time and accuracy\footnote{In an attempt to make a fair comparison, we obtained the code for the algorithms of~\cite{NASA_GoalRecognition_IJCAI2015} directly from the main author. 
Running on our datasets, these algorithms performed worse than the results of~\cite{NASA_GoalRecognition_IJCAI2015}. 
We believe that this could be due to different problem set sizes of our datasets. 
In addition, the code behaved unexpectedly on some domains of our datasets (denoted by a $\dag$ symbol in the tables), returning the same recognition score for all candidate goals.
At the time of submission we are working with the authors to clarify these discrepancies.}. 
In both non-noisy (missing) and noisy domains, the \textsc{IPC-Grid} domain timed out for more than 60\% of the problems in the FGR 2015 approach, so we do not report results for this specific domain as they would not be representative. 

These approaches take as input a goal recognition problem $T_{GR}$ (from Definition~\ref{def:goalRecognition}), \idest, a domain description as well as an initial state, a set of candidate goals $\mathcal{G}$, a correct hidden goal $G^{*}$ in $\mathcal{G}$, and an observation sequence $O$. 
An observation sequence contains actions that represent an optimal plan or sub-optimal plan that achieves a correct hidden goal $G^{*}$, and this observation sequence can be full or partial. 
Full observation sequences contain the entire plan for a correct hidden goal $G^{*}$, \idest, 100\% of the actions having been observed. 
Partial observation sequences represent plans for a correct hidden goal $G^{*}$ with 10\%, 30\%, 50\%, or 70\% of their actions having been observed. 
However, for experiments with noisy observations, the observability of partial observations is quite different because every observation sequence always includes at least two noisy observations, so a partial observation sequence with noisy observations represents a plan with 25\%, 50\%, or 75\% of its actions having been observed.

Our evaluation uses three metrics: accuracy of goal recognition (Equation~\ref{frac:accuracy}), the average number of goals in $\mathcal{G}$ that have been found to be the most likely, and recognition time (speed). 
Note that in many domains, all algorithms return more than one candidate goal. 
In the case of Ram{\'{\i}}rez and Geffner~(\citeyear{RamirezG_IJCAI2009}), \idest, R\&G 2009, this may occur when goals have the same distance from their estimated state. 
Alternatively, for our goal recognition heuristics, this may occur when there are ties between the heuristic value of candidate goals within the threshold margin. 
Thus, like most goal recognition approaches, we also evaluate the average number of returned goals for a given goal recognition problem, \idest, the Spread in $\mathcal{G}$. 
We ran all experiments using a single core of a 12 core Intel(R) Xeon(R) CPU E5-2620 v3 @ 2.40GHz with 16GB of RAM, set a maximum memory usage limit of 4GB, and set a 20-minute timeout for each recognition process. 

\subsubsection{Experimental Results with Missing and Full Observations}

Our first set of experiments consists of running the various goal recognition algorithms in datasets containing thousands of problems for 15 domains with missing and full observations. 
In what follows, each table shows the total number of goal recognition problems used under each domain name (first column). 
Each row in the tables express averages for the number of candidate goals $|\mathcal{G}|$; the percentage of the plan that is actually observed \% Obs; the average number of observations (actions) per problem $|O|$; and for each approach, the time in seconds to recognize the goal given the observations (Time); the Accuracy with which the approaches correctly infer the goal; and Spread in $\mathcal{G}$ represents the average number of returned goals. 
For our goal recognition heuristics $\mathit{h_{gc}}$ and $\mathit{h_{uniq}}$, we show their results under different thresholds: 0\%, 10\%, and 20\%. 
If the threshold value is $\theta=0$, our approaches do not give any flexibility estimating candidate goals, returning only the goals with the highest estimated value. 
Tables~\ref{tab:goalRecognitionResults1_1} and~\ref{tab:goalRecognitionResults1_2} show comparative results of our heuristics and previous approaches for the first set of domains (\textsc{Blocks-World} to \textsc{Intrusion}). 
Table~\ref{tab:goalRecognitionResults1_1} shows the results of our goal recognition heuristics $\mathit{h_{gc}}$ and $\mathit{h_{uniq}}$, against R\&G 2009~\cite{RamirezG_IJCAI2009} as well as this approach enhanced with our filtering method with threshold $\theta = 10\%$ (Filter$_{10\%}$). 
Similarly, Table~\ref{tab:goalRecognitionResults1_2} shows the results of R\&G 2010~\cite{RamirezG_AAAI2010}, FGR 2015~\cite{NASA_GoalRecognition_IJCAI2015}, and IBM 2016~\cite{Sohrabi_IJCAI2016}. We show both approaches of R\&G 2010~\cite{RamirezG_AAAI2010} and IBM 2016~\cite{Sohrabi_IJCAI2016} individually as well as enhanced with our filtering method (again with $\theta = 10\%$). 
Tables~\ref{tab:goalRecognitionResults2_1}, and~\ref{tab:goalRecognitionResults2_2} show comparative results of our heuristics and previous approaches for the second set of domains (\textsc{Kitchen} to \textsc{Zeno-Travel}). 
From these tables, it is possible to see that our landmark-based approaches are both faster and more accurate than R\&G 2009, R\&G 2010, FGR 2015, and IBM 2016, and, when we combine their algorithms with our filtering method, the resulting approaches get a substantial speedup and often accuracy improvements. 
As we increase the threshold, our heuristic approaches quickly surpass the other approaches in all domains tested. 
Note that we report the accuracy averaged over all of the problems for each observability. 
For example, in Table~\ref{tab:goalRecognitionResults1_2}, for the \textsc{Campus} domain, there are $15$ problems with $50\%$ observability (totaling $75$ for the entire domain), and the IBM 2016~\cite{Sohrabi_IJCAI2016} includes this goal in its output for $6$ out of $15$, resulting in $40\%$ accuracy. 

The \textit{Receiver Operating Characteristic (ROC) curve} allows us to provide a summary of the discriminatory performance of inferences such as goal recognition over diverse datasets. 
The ROC curve shows graphically the performance of classifier systems by evaluating true positive rate against the false positive rate at various threshold settings. We adapt the notion of the ROC curve into points over the ROC space to compare not only true positive predictions (\idest, Accuracy), but also to compare the false positive ratio of the experimented goals recognition approaches. 
Each prediction result of a goal recognition approach represents one point in the space.
In ROC space, the diagonal line represents a random guess to recognize a goal from observations. 
This diagonal line divides the ROC space, in which points above the diagonal represent good classification results (better than random), whereas points below the line represent poor results (worse than random). 
The best possible (perfect) prediction for recognizing goals must be a point in the upper left corner (\idest, coordinate x = 0 and y = 100) in the ROC space. 
Thus, the closer a goal recognition approach (point) gets to the upper left corner, the better it is for recognizing goals. 
To visualize the comparative performance of the multiple approaches, we adapt the notation of the ROC space, and, rather than plotting a single point per goal recognition problem, we aggregate multiple problems for all domains and plot these results in ROC space.

Figure~\ref{fig:rocspace_alldomains} shows the trade-off between true positive results and false positive results in ROC space for the evaluated goal and plan recognition approaches. 
Recall that the closer a goal recognition approach (point) is to the upper left corner, the better it is for recognizing goals and plans. 
To compare the recognition results of our approaches against the others in the ROC space, we select the results of our heuristics using the threshold $\theta$ = 20\%. 
For each approach, we plot its recognition results for all domains into a cloud of points, which represents (in general) how well each approach recognizes the correct hidden goal from missing and full observations. 
Thus, the points in ROC space show that our heuristics are not only competitive against the four other approaches (R\&G 2009, R\&G 2010, FGR 2015, and IBM 2016) for all variations of observability, but also surpasses the approaches in a substantial number of domains. 

\begin{figure}[ht!]
  \centering
  \includegraphics[width=0.7\linewidth]{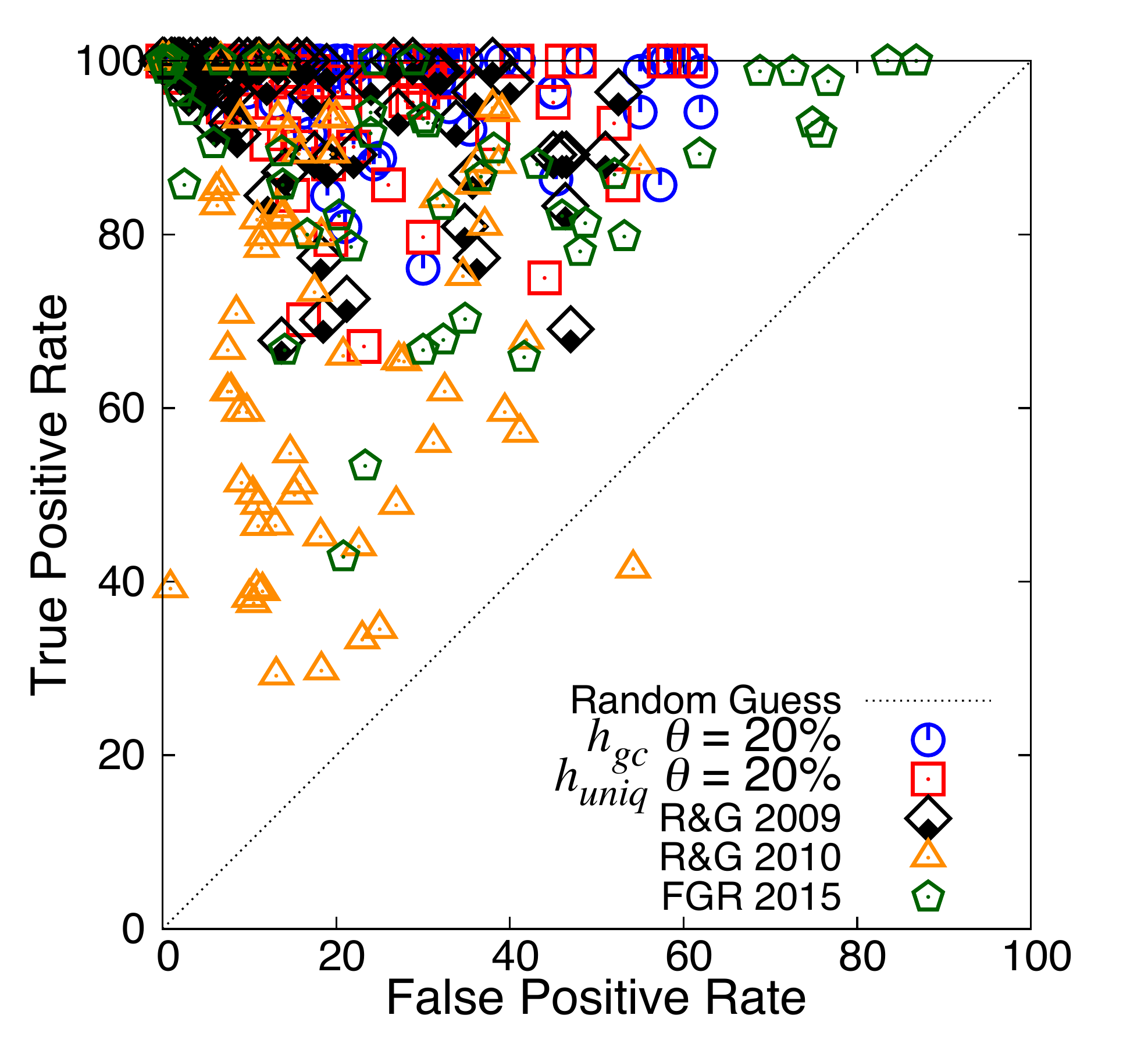}
  \caption{ROC space for all domains with missing and full observations for our landmark-based heuristics ($\mathit{h_{gc}}$ and $\mathit{h_{uniq}}$) against R\&G 2009~\cite{RamirezG_IJCAI2009}, R\&G 2010~\cite{RamirezG_AAAI2010}, and FGR 2015~\cite{NASA_GoalRecognition_IJCAI2015}. The results obtained for \textsc{Campus} domain from IBM 2016~\cite{Sohrabi_IJCAI2016} are not included in the ROC space.}
  \label{fig:rocspace_alldomains}
\end{figure}

With respect to recognition time, we compare the time that each approach takes to recognize the correct hidden goal for different sizes of the observation sequence.
Figures~\ref{fig:recognition_time-missing_1} and~\ref{fig:recognition_time-missing_2} show the runtime as a function of the average length of the observation sequences ($|O|$ column in the Tables) for all of the approaches we evaluated (apart from IBM 2016, which timed out for almost all domains), as reported in the time column of Tables~\ref{tab:goalRecognitionResults1_1}, \ref{tab:goalRecognitionResults1_2}, \ref{tab:goalRecognitionResults2_1}, and~\ref{tab:goalRecognitionResults2_2}). 
Figure~\ref{fig:recognition_time-missing_1} shows the runtime for our heuristic approaches in comparison with R\&G 2009 and this approach alongside our filtering method, whereas Figure~\ref{fig:recognition_time-missing_2} shows the runtime of R\&G 2010 with and without the filtering method, and FGR 2015. 
We used separate graphs for these techniques, given the widely different magnitude of the time taken to recognize a goal.
When measuring recognition time for our heuristics and the filtering method, we also include the time to extract the set of landmarks, so that landmark extraction is performed online, \idest, during the goal recognition process.
Curves in the graph represent the average runtime when observation sizes were the same smoothed over the resulting points. 
The graph shows the scalability of the 4 evaluated approaches. 
Our goal recognition heuristics never take more than 1 second ($\approx 0.7$ seconds) to compute the correct hidden goal in the set of candidate goals, while the other approaches seem to grow super-linearly (for R\&G 2009), and exponentially (for R\&G 2010). 
The approaches of Ram{\'{\i}}rez and Geffner~(\citeyear{RamirezG_IJCAI2009,RamirezG_AAAI2010}), R\&G 2009 and R\&G 2010, took at most $\approx$ $25$ seconds and $\approx 1200$ seconds, respectively. 
Apart from the \textsc{Campus} domain, the approach of IBM 2016~\cite{Sohrabi_IJCAI2016} timed out for all goal recognition problems (all domains) in the datasets we used, probably because the top-K planner~\cite{katz_etal_icaps18} (even the latest top-K planning algorithm) does not scale very well when dealing with non-trivial planning problems, especially when the planner has to sample $1000$ plans for a (transformed) planning problem. 
In this case, even the use of our filtering method (which reduces the number of candidate goals) did not improve the recognition time of the approach from IBM 2016~\cite{Sohrabi_IJCAI2016}. 
While our filtering method significantly improves the recognition time of the approaches R\&G 2009 and R\&G 2010, it sometimes causes a loss of accuracy due to it ruling out the correct hidden from the set of candidate goals. FGR 2015 took at most $\approx$ $355$ seconds over all evaluated domains (and timed out for most problems of the \textsc{IPC-Grid} domain). 
FGR 2015 is much faster than R\&G 2010 and R\&G 2010 with our filtering method, though not as fast as our recognition heuristics and R\&G 2009. 
Finally, the evaluation of the domains \textsc{DWR} and \textsc{Sokoban} shows that larger plan lengths lead R\&G 2009 and R\&G 2010 to rapidly lose accuracy, whereas our approaches show improved accuracy without affecting the recognition time. 

As Tables~\ref{tab:goalRecognitionResults1_1}, \ref{tab:goalRecognitionResults1_2}, \ref{tab:goalRecognitionResults2_1}, and~\ref{tab:goalRecognitionResults2_2} show, our goal recognition heuristics are not only competitive (using thresholds between $10\%$ and $20\%$) against the other approaches with superior accuracy, but also at least an order of magnitude faster (for all evaluated domains), for example, $\approx$ 2900 times faster than R\&G 2010 in \textsc{DWR} domain. 
Comparison with two other state-of-the-art recent techniques, we can also see that IBM 2016 is substantially slower, even compared to R\&G, whereas FGR 2015, while consistently much faster than R\&G 2010, is also slower than our heuristics techniques (up to an order of magnitude) across the board. 
When comparing with our heuristics, the results show that the goal completion heuristic $\mathit{h_{gc}}$ is often more accurate than the uniqueness heuristic $\mathit{h_{uniq}}$. 
However, $\mathit{h_{uniq}}$ returns fewer candidate goals (Spread in $\mathcal{G}$) than the goal completion heuristic $\mathit{h_{gc}}$ as a result of the \textit{landmark uniqueness value}, which weights landmark information among all landmarks for all goals, making $\mathit{h_{uniq}}$ more precise (but sometimes less accurate) than the goal completion heuristic $\mathit{h_{gc}}$. 
We use the threshold value to provide flexibility when the heuristic approaches fail to observe landmarks. 
While our approach is more accurate than virtually all other approaches with a recognition threshold value of $20\%$ of optimal (sometimes with larger spread), the comparison becomes more complex for other thresholds. 
The only domain in which the FGR 2015 approach is more accurate than ours is DWR with with observability under $70\%$, however, the spread in $\mathcal{G}$ is nearly twice as large as ours, meaning that FGR 2015 is worse at disambiguating goals. 
Apart from the \textsc{Campus} and \textsc{Kitchen} domains, our approaches have similar or worse accuracy at very low ($30\%$ or less) observability. 
This loss of accuracy happens for low observability problems because the number of landmarks that happen to be observed is much lower (as the likelihood of observing a landmark goes down) creating a challenge to disambiguate and recognize the correct hidden goal. 
The results for \textsc{Campus} and \textsc{Kitchen} are explained by the reduced number of goal hypotheses in each domain and the informativeness of the actions, which yield landmarks that favor our approaches. 
For domains such as \textsc{DWR}, \textsc{Depots}, \textsc{Sokoban}, \textsc{Zeno-Travel}, which are considered more complex because traditional planning heuristics are not very informative for them, our results are mixed. 
Sometimes, we are able to achieve high accuracy with low observability (albeit with high spread in \textsc{DWR} and \textsc{Depots}), whereas sometimes we achieve lower accuracy with low spread for \textsc{Sokoban} and \textsc{Zeno-Travel}. 
In this particular setting, \textsc{Sokoban} is known to be a particularly difficult domain for planning heuristics~\cite{PereiraAndre_AIJ2015}, and yields a small number of landmarks per goal. Nevertheless, when our heuristic approaches deal with more than $30\%$ of observability the results are very good both in Accuracy and Spread in $\mathcal{G}$ for all domains.

\begin{figure}[h!]
  \centering
  \includegraphics[width=0.73\linewidth]{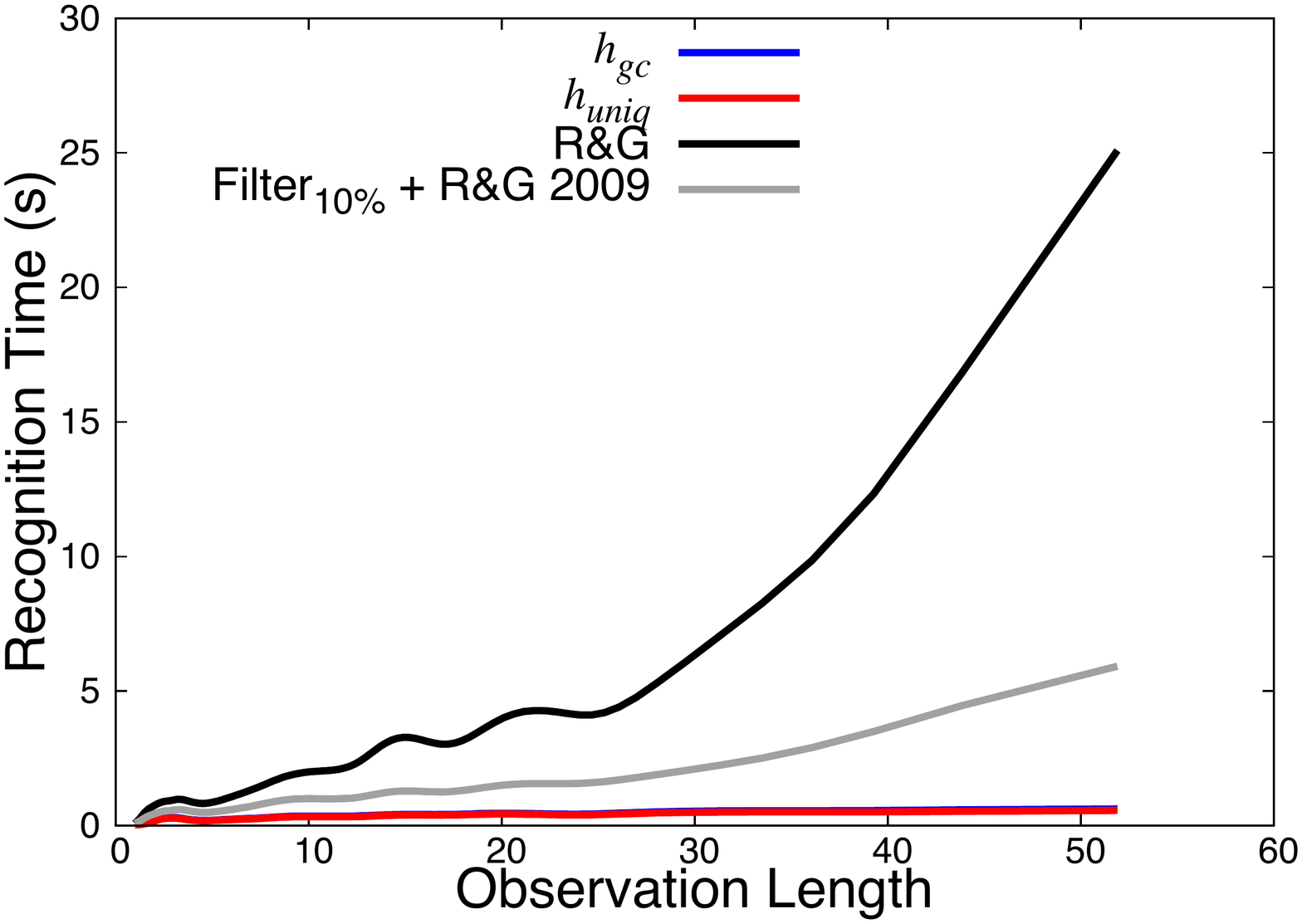}
  \caption{Recognition time comparison for missing and full observations for our landmark-based heuristics ($\mathit{h_{gc}}$ and $\mathit{h_{uniq}}$) against R\&G 2009~\cite{RamirezG_IJCAI2009}, and R\&G 2009 using our filtering method with 10\% of threshold.}
  \label{fig:recognition_time-missing_1}
\end{figure}

\begin{figure}[h!]
  \centering
  \includegraphics[width=0.73\linewidth]{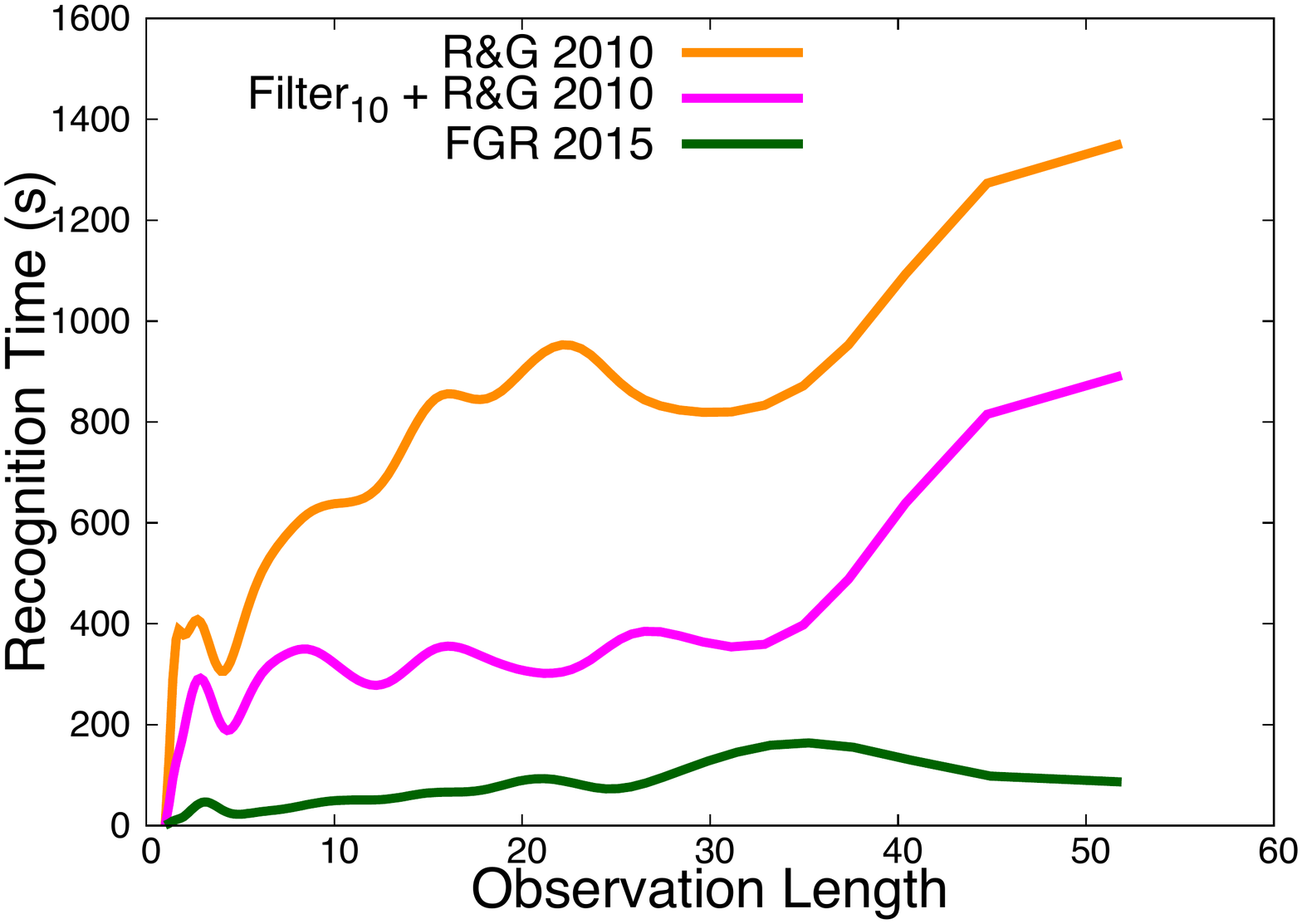}
  \caption{Recognition time comparison for missing and full observations for R\&G 2010 using \textsc{Fast-Downward} with LM-Cut heuristic~\cite{RamirezG_AAAI2010}, R\&G 2010 using our filtering method with 10\% of threshold, and FGR 2015~\cite{NASA_GoalRecognition_IJCAI2015}.}
  \label{fig:recognition_time-missing_2}
\end{figure}

\afterpage{
\begin{landscape}
\begin{table*}[]
\centering
\fontsize{5}{8}\selectfont
\setlength\tabcolsep{3pt}
\begin{tabular}{|c|c|cc|ccc|ccc|ccc|}
\hline
                   &                     & \multicolumn{2}{c|}{} & \multicolumn{3}{c|}{{\footnotesize $\mathit{h_{gc}}$}}                                   & \multicolumn{3}{c|}{{\footnotesize $\mathit{h_{uniq}}$}}                                 & \multicolumn{3}{c|}{\scriptsize R\&G 2009 / Filter$_{10\%}$ + R\&G 2009}    \\ \hline
\#                 & $|\mathcal{G}|$                   & \textbf{\% Obs}        & $|O|$        & \textbf{\begin{tabular}[c]{@{}c@{}}Time\\$\theta$ (0 / 10 / 20)\end{tabular}}                  & \textbf{\begin{tabular}[c]{@{}c@{}}Accuracy\\$\theta$ (0 / 10 / 20)\end{tabular}}              & \textbf{\begin{tabular}[c]{@{}c@{}}Spread in $\mathcal{G}$\\$\theta$ (0 / 10 / 20)\end{tabular}}  & \textbf{\begin{tabular}[c]{@{}c@{}}Time\\$\theta$ (0 / 10 / 20)\end{tabular}}                  & \textbf{\begin{tabular}[c]{@{}c@{}}Accuracy\\$\theta$ (0 / 10 / 20)\end{tabular}}              & \textbf{\begin{tabular}[c]{@{}c@{}}Spread in $\mathcal{G}$\\$\theta$ (0 / 10 / 20)\end{tabular}}  & \textbf{Time}          & \textbf{Accuracy}      & \textbf{Spread in $\mathcal{G}$} \\ \hline

\multirow{5}{*}{\rotatebox[origin=c]{90}{\textsc{Blocks-World}} \rotatebox[origin=c]{90}{(1076)}} 
& \multirow{5}{*}{20} 
										 & 10         & 1.8
											& 0.137 / 0.143 / 0.151 & 39.9\% / 59.2\% / 86.4\% & 1.05 / 4.62 / 8.92 
											
											& 0.131 / 0.139 / 0.146 & 31.6\% / 53.1\% / 67.1\% & 1.03 / 2.77 / 6.06 
											
											& 1.235 / 0.631 & 86.8\% / 60.9\% & 7.84 / 3.34     \\
                   &                     & 30         & 4.9        
											& 0.152 / 0.160 / 0.169 & 50.6\% / 79.4\% / 92.1\% & 1.09 / 3.96 / 7.76
											
											& 0.144 / 0.153 / 0.164 & 51.4\% / 67.1\% / 79.4\% & 1.06 / 2.51 / 5.18 
											
											& 1.698 / 0.819 & 87.2\% / 75.7\% & 3.56 / 1.79     \\
                   &                     & 50         & 7.6        
											& 0.179 / 0.187 / 0.196    & 65.1\% / 86.1\% / 98.7\% & 1.09 / 3.13 / 6.11
											& 0.168 / 0.174 / 0.185     & 60.1\% / 77.7\% / 90.1\% & 1.08 / 2.24 / 4.48 
											& 2.497 / 1.008 & 97.9\% / 87.2\% & 2.63 / 1.42     \\
                   &                     & 70         & 11.1        
											& 0.192 / 0.201 / 0.214    & 84.7\% / 95.8\% / 100\% & 1.12 / 2.51 / 3.79 
											
											& 0.184 / 0.193 / 0.207    & 79.1\% / 90.5\% / 97.9\% & 1.13 / 2.02 / 3.44 
											
											& 3.704 / 1.225 & 97.5\% / 94.6\% & 1.83 / 1.18     \\
                   &                     & 100        & 14.5       
											& 0.245 / 0.254 / 0.261   & 100\% / 100\% / 100\% & 1.09 / 1.78 / 2.51 
											
											& 0.238 / 0.246 / 0.253   & 100\% / 100\% / 100\% & 1.09 / 1.61 / 2.51 
											
											& 6.123 / 1.567 & 100\% / 100\% & 1.46 / 1.06     \\ \hline

\multirow{5}{*}{\rotatebox[origin=c]{90}{\textsc{Campus}} \rotatebox[origin=c]{90}{(75)}} 
& \multirow{5}{*}{2} 
										 & 10         & 1
											& 0.031 / 0.032 / 0.034 & 93.3\% / 100\% / 100\% & 1.0 / 1.33 / 1.46 
											
											& 0.027 / 0.029 / 0.030 & 100\% / 100\% / 100\% & 1.13 / 1.46 / 1.46 
											
											& 0.084 / 0.082 & 100\% / 100\% & 1.46 / 1.46     \\
                   &                     & 30         & 2        
											& 0.045 / 0.048 / 0.049 & 100\% / 100\% / 100\% & 1.0 / 1.33 / 1.46
											
											& 0.042 / 0.044 / 0.046 & 100\% / 100\% / 100\% & 1.13 / 1.46 / 1.46 
											
											& 0.097 / 0.096 & 100\% / 100\% & 1.33 / 1.33     \\
                   &                     & 50         & 3       
											& 0.057 / 0.060 / 0.064    & 93.3\% / 100\% / 100\% & 1.0 / 1.33 / 1.46
											
											& 0.055 / 0.055 / 0.057     & 93.3\% / 100\% / 100\% & 1.13 / 1.46 / 1.46
											
											& 0.104 / 0.102 & 100\% / 100\% & 1.33 / 1.33     \\
                   &                     & 70         & 4.4        
											& 0.061 / 0.063 / 0.068    & 100\% / 100\% / 100\% & 1.0 / 1.33 / 1.33 
											
											& 0.058 / 0.059 / 0.062    & 100\% / 100\% / 100\% & 1.0 / 1.33 / 1.33 
											
											& 0.115 / 0.113 & 100\% / 100\% & 1.26 / 1.26     \\
                   &                     & 100        & 5.5       
											& 0.066 / 0.067 / 0.070   & 100\% / 100\% / 100\% & 1.0 / 1.0 / 1.0 
											
											& 0.061 / 0.063 / 0.064   & 100\% / 100\% / 100\% & 1.0 / 1.0 / 1.0
											
											& 0.128 / 0.129 & 100\% / 100\% & 1.13 / 1.13     \\ \hline
											
\multirow{5}{*}{\rotatebox[origin=c]{90}{\textsc{Depots}} \rotatebox[origin=c]{90}{(364)}} 
& \multirow{5}{*}{8.5} 
										 & 10         & 3.1
											& 0.348 / 0.364 / 0.375 & 35.7\% / 66.6\% / 85.7\% & 1.17 / 3.42 / 4.97 
											
											& 0.331 / 0.342 / 0.357 & 32.1\% / 48.8\% / 75\% & 1.09 / 2.70 / 3.86 
											
											& 1.485 / 0.721 & 77.3\% / 66.6\% & 3.98 / 2.89     \\
                   &                     & 30         & 8.6        
											& 0.372 / 0.401 / 0.428 & 58.3\% / 82.1\% / 96.4\% & 1.05 / 2.53 / 3.73
											
											& 0.356 / 0.389 / 0.410 & 47.6\% / 71.4\% / 91.6\% & 1.07 / 2.57 / 2.84 
											
											& 2.307 / 1.214 & 77.3\% / 76.1\% & 2.39 / 1.72     \\
                   &                     & 50         & 14.1       
											& 0.439 / 0.482 / 0.514    & 76.1\% / 94.1\% / 97.6\% & 1.05 / 1.77 / 2.53
											
											& 0.415 / 0.447 / 0.470     & 71.4\% / 86.9\% / 96.4\% & 1.02 / 2.10 / 1.75
											
											& 3.433 / 1.532 & 84.5\% / 89.2\% & 1.91 / 1.33     \\
                   &                     & 70         & 19.7        
											& 0.502 / 0.555 / 0.593    & 89.2\% / 94.1\% / 97.6\% & 1.01 / 1.35 / 1.91 
											
											& 0.481 / 0.528 / 0.562    & 84.5\% / 96.4\% / 96.4\% & 1.01 / 1.09 / 1.46 
											
											& 5.149 / 1.866 & 91.6\% / 94.1\% & 1.67 / 1.13     \\
                   &                     & 100        & 24.4       
											& 0.613 / 0.642 / 0.677   & 100\% / 100\% / 100\% & 1.03 / 1.05 / 1.65 
											
											& 0.575 / 0.601 / 0.633   & 100\% / 100\% / 100\% & 1.03 / 1.09 / 1.46 
											
											& 7.094 / 2.408 & 92.8\% / 100\% & 1.46 / 1.07     \\ \hline

\multirow{5}{*}{\rotatebox[origin=c]{90}{\textsc{Driver-Log}} \rotatebox[origin=c]{90}{(364)}} 
& \multirow{5}{*}{10.5} 
										 & 10         & 2.6
											& 0.295 / 0.302 / 0.310 & 41.6\% / 52.3\% / 76.1\% & 1.03 / 1.88 / 2.60 
											
											& 0.284 / 0.292 / 0.298 & 35.7\% / 54.7\% / 79.7\% & 1.10 / 1.84 / 2.88
											
											& 1.192 / 0.499 & 96.4\% / 61.9\% & 4.71 / 2.40     \\
                   &                     & 30         & 6.9        
											& 0.303 / 0.310 / 0.316 & 54.7\% / 72.6\% / 90.1\% & 1.13 / 1.90 / 2.77
											
											& 0.291 / 0.299 / 0.305 & 47.6\% / 69.1\% / 85.7\% & 1.09 / 2.03 / 2.40 
											
											& 1.444 / 0.605 & 92.8\% / 72.6\% & 3.34 / 1.66     \\
                   &                     & 50         & 11.1       
											& 0.312 / 0.320 / 0.325    & 72.6\% / 85.7\% / 96.4\% & 1.16 / 1.88 / 2.60
											
											& 0.290 / 0.297 / 0.301     & 64.2\% / 82.1\% / 92.8\% & 1.14 / 1.84 / 2.35
											
											& 1.608 / 0.757 & 94.1\% / 86.9\% & 2.88 / 1.41     \\
                   &                     & 70         & 15.6        
											& 0.322 / 0.330 / 0.342    & 90.4\% / 94.1\% / 100\% & 1.14 / 1.58 / 2.15 
											
											& 0.298 / 0.304 / 0.309    & 90.4\% / 97.6\% / 100\% & 1.14 / 1.47 / 2.05 
											
											& 1.925 / 0.846 & 89.2\% / 98.8\% & 2.46 / 1.35     \\
                   &                     & 100        & 21.7       
											& 0.331 / 0.339 / 0.344   & 100\% / 100\% / 100\% & 1.21 / 1.32 / 1.92 
											
											& 0.305 / 0.311 / 0.318   & 100\% / 100\% / 100\% & 1.17 / 1.25 / 1.46 
											
											& 2.809 / 1.213 & 89.2\% / 96.4\% & 2.14 / 1.14     \\ \hline

\multirow{5}{*}{\rotatebox[origin=c]{90}{\textsc{DWR}} \rotatebox[origin=c]{90}{(364)}} 
& \multirow{5}{*}{7.25} 
										 & 10         & 5.7
											& 0.523 / 0.534 / 0.541 & 36.9\% / 85.7\% / 94.1\% & 1.09 / 4.32 / 5.83
											
											& 0.491 / 0.499 / 0.511 & 33.3\% / 70.2\% / 85.7\% & 1.05 / 3.07 / 4.32
											
											& 1.634 / 0.921 & 83.3\% / 79.7\% & 4.21 / 2.29     \\
                   &                     & 30         & 16        
											& 0.535 / 0.546 / 0.557 & 60.7\% / 95.2\% / 98.8\% & 1.03 / 3.61 / 4.84
											
											& 0.518 / 0.527 / 0.538 & 51.1\% / 80.1\% / 95.2\% & 1.05 / 2.41 / 3.61 
											
											& 2.977 / 1.204 & 80.9\% / 83.3\% & 3.34 / 2.89     \\
                   &                     & 50         & 26.2        
											& 0.560 / 0.572 / 0.581 & 66.6\% / 97.6\% / 100\% & 1.0 / 3.19 / 3.92
											
											& 0.533 / 0.541 / 0.552 & 61.9\% / 88.1\% / 97.6\% & 1.04 / 2.16 / 3.19
											
											& 4.485 / 2.121 & 72.6\% / 85.7\% & 2.27 / 1.64     \\
                   &                     & 70         & 36.8        
											& 0.601 / 0.608 / 0.599 & 89.2\% / 98.8\% / 100\% & 1.0 / 2.50 / 3.07
											
											& 0.540 / 0.547 / 0.555 & 78.5\% / 98.8\% / 98.8\% & 1.03 / 2 / 2.50
											
											& 10.432 / 3.555 & 70.2\% / 89.9\% & 2.04 / 1.48     \\
                   &                     & 100        & 51.9       
											& 0.613 / 0.620 / 0.626 & 100\% / 100\% / 100\% & 1.0 / 1.67 / 2.50
											
											& 0.559 / 0.551 / 0.564   & 100\% / 100\% / 100\% & 1.01 / 1.60 / 1.67 
											
											& 25.091 / 5.921 & 67.8\% / 92.8\% & 1.67 / 1.14     \\ \hline

\multirow{5}{*}{\rotatebox[origin=c]{90}{\textsc{IPC-Grid}} \rotatebox[origin=c]{90}{(673)}} 
& \multirow{5}{*}{9} 
										 & 10         & 2.9
											& 0.243 / 0.255 / 0.262 & 66.6\% / 86.2\% / 94.1\% & 2.57 / 3.28 / 4.41
											
											& 0.220 / 0.229 / 0.237 & 62.7\% / 82.3\% / 92.8\% & 2.34 / 3.13 / 4.09
											
											& 1.084 / 0.708 & 96.1\% / 85.6\% & 2.45 / 2.11     \\
                   &                     & 30         & 7.8        
											& 0.251 / 0.264 / 0.273 & 81.6\% / 87.5\% / 88.8\% & 1.64 / 2.32 / 3.28
											
											& 0.234 / 0.241 / 0.251 & 83.6\% / 89.5\% / 90.1\% & 1.66 / 2.34 / 3.28
											
											& 1.475 / 0.960 & 97.3\% / 87.5\% & 1.42 / 1.25     \\
                   &                     & 50         & 12.7        
											& 0.260 / 0.269 / 0.276 & 90.8\% / 93.4\% / 93.4\% & 1.18 / 1.26 / 2.32
											
											& 0.245 / 0.252 / 0.259 & 90.1\% / 94.7\% / 94.7\% & 1.18 / 1.48 / 2.57
											
											& 1.932 / 1.125 & 100\% / 93.4\% & 1.15 / 1.04     \\
                   &                     & 70         & 17.9        
											& 0.272 / 0.278 / 0.285 & 97.3\% / 97.3\% / 98.1\% & 1.07 / 1.15 / 1.44 
											
											& 0.253 / 0.260 / 0.267 & 97.3\% / 97.3\% / 98.1\% & 1.11 / 1.16 / 1.48
											
											& 2.556 / 1.211 & 100\% / 97.3\% & 1.05 / 1.0     \\
                   &                     & 100        & 24.8       
											& 0.286 / 0.289 / 0.291 & 100\% / 100\% / 100\% & 1.0 / 1.0 / 1.0
											
											& 0.261 / 0.268 / 0.279 & 100\% / 100\% / 100\% & 1.0 / 1.0 / 1.0
											
											& 3.868 / 1.304 & 100\% / 100\% & 1.0 / 1.0     \\ \hline
																						
\multirow{5}{*}{\rotatebox[origin=c]{90}{\textsc{Ferry}} \rotatebox[origin=c]{90}{(364)}} 
& \multirow{5}{*}{7.5} 
										 & 10         & 2.9
											& 0.077 / 0.083 / 0.093 & 58.3\% / 85.7\% / 98.8\% & 1.26 / 3.19 / 4.76
											
											& 0.068 / 0.083 / 0.087 & 58.3\% / 89.2\% / 100\% & 1.17 / 3.14 / 3.45
											
											& 0.511 / 0.302 & 98.8\% / 90.4\% & 3.36 / 2.35     \\
                   &                     & 30         & 7.6        
											& 0.084 / 0.092 / 0.099 & 85.7\% / 97.6\% / 100\% & 1.11 / 2.13 / 3.25
											
											& 0.073 / 0.081 / 0.088 & 83.3\% / 95.2\% / 100\% & 1.05 / 1.90 / 2.28
											
											& 0.677 / 0.399 & 100\% / 97.6\% & 1.76 / 1.41     \\
                   &                     & 50         & 12.3        
											& 0.091 / 0.096 / 0.102 & 95.2\% / 98.8\% / 100\% & 1.07 / 1.5 / 1.72
											
											& 0.084 / 0.086 / 0.092 & 91.6\% / 92.8\% / 100\% & 1.01 / 1.38 / 1.40
											
											& 0.794 / 0.410 & 100\% / 98.8\% & 1.41 / 1.16     \\
                   &                     & 70         & 17.3        
											& 0.098 / 0.100 / 0.107 & 100\% / 100\% / 100\% & 1.01 / 1.13 / 1.17
											
											& 0.092 / 0.094 / 0.101 & 100\% / 100\% / 100\% & 1.0 / 1.11 / 1.38
											
											& 1.202 / 0.525 & 98.8\% / 100\% & 1.14 / 1.02     \\
                   &                     & 100        & 24.2       
											& 0.104 / 0.108 / 0.112 & 100\% / 100\% / 100\% & 1.0 / 1.0 / 1.01
											
											& 0.099 / 0.102 / 0.106 & 100\% / 100\% / 100\% & 1.0 / 1.0 / 1.07
											
											& 1.693 / 0.571 & 100\% / 100\% & 1.07 / 1.0     \\ \hline
																																		
\multirow{5}{*}{\rotatebox[origin=c]{90}{\textsc{Intrusion}} \rotatebox[origin=c]{90}{(465)}} 
& \multirow{5}{*}{15} 
										 & 10         & 1.9
											& 0.095 / 0.098 / 0.102 & 62.8\% / 96.1\% / 100\% & 1.14 / 2.56 / 5.12
											
											& 0.077 / 0.084 / 0.090 & 64.7\% / 100\% / 100\% & 1.23 / 2.54 / 7.29
											
											& 0.724 / 0.444 & 100\% / 100\% & 2.53 / 2.53     \\
                   &                     & 30         & 4.5        
											& 0.101 / 0.106 / 0.108 & 94.2\% / 100\% / 100\% & 1.01 / 1.96 / 2.56
											
											& 0.083 / 0.088 / 0.092 & 85.7\% / 100\% / 100\% & 1.02 / 1.96 / 6.12
											
											& 0.803 / 0.486 & 100\% / 100\% & 1.11 / 1.11     \\
                   &                     & 50         & 6.7        
											& 0.109 / 0.111 / 0.114 & 99.1\% / 100\% / 100\% & 1.01 / 1.19 / 1.61
											
											& 0.089 / 0.091 / 0.094 & 94.2\% / 100\% / 100\% & 1.04 / 1.91 / 3.31
											
											& 0.888 / 0.513 & 100\% / 100\% & 1.02 / 1.0     \\
                   &                     & 70         & 9.5        
											& 0.113 / 0.115 / 0.121 & 100\% / 100\% / 100\% & 1.0 / 1.02 / 1.19
											
											& 0.093 / 0.096 / 0.099 & 94.2\% / 100\% / 100\% & 1.0 / 1.67 / 2.46
											
											& 1.012 / 0.539 & 100\% / 100\% & 1.0 / 1.0     \\
                   &                     & 100        & 13.1       
											& 0.120 / 0.126 / 0.129 & 100\% / 100\% / 100\% & 1.0 / 1.02 / 1.02
											
											& 0.098 / 0.100 / 0.102 & 100\% / 100\% / 100\% & 1.0 / 1.60 / 1.88
											
											& 1.257 / 0.550 & 100\% / 100\% & 1.0 / 1.0     \\ \hline
																					
\end{tabular}
\caption{Experiments and evaluation with missing and full observations for $\mathit{h_{gc}}$, $\mathit{h_{uniq}}$, R\&G 2009, and our filtering method (10\% of threshold) with R\&G 2009 (Part 1).}
\label{tab:goalRecognitionResults1_1}
\end{table*}
\end{landscape}
}

\afterpage{
\begin{landscape}
\begin{table*}[]
\centering
\fontsize{5}{8}\selectfont
\setlength\tabcolsep{3pt}
\begin{tabular}{|c|c|cc|ccc|ccc|ccc|}
\hline
                   &                     
				   & \multicolumn{2}{c|}{} & \multicolumn{3}{c|}{\begin{tabular}[c]{@{}c@{}}\scriptsize R\&G 2010 / Filter$_{10\%}$ + R\&G 2010 \\ \scriptsize (\textsc{Fast-Downward} with LM-Cut heuristic)\end{tabular}}

				   & \multicolumn{3}{c|}{\begin{tabular}[c]{@{}c@{}}\scriptsize FGR 2015\end{tabular}}
				   
				   & \multicolumn{3}{c|}{\begin{tabular}[c]{@{}c@{}}\scriptsize IBM 2016 / Filter$_{10\%}$ + IBM 2016 \\ \scriptsize (TK$^*$ with LM-Cut heuristic, top-1000)\end{tabular}}    \\ \hline
				   
\#                 & $|\mathcal{G}|$                   & \textbf{\% Obs}        & $|O|$        

& \textbf{\begin{tabular}[c]{@{}c@{}}Time\end{tabular}}                  
& \textbf{\begin{tabular}[c]{@{}c@{}}Accuracy\end{tabular}}              
& \textbf{\begin{tabular}[c]{@{}c@{}}Spread in $\mathcal{G}$\end{tabular}}  

& \textbf{Time}          
& \textbf{Accuracy}      
& \textbf{Spread in $\mathcal{G}$}

& \textbf{Time}          
& \textbf{Accuracy}      
& \textbf{Spread in $\mathcal{G}$} \\ \hline

\multirow{5}{*}{\rotatebox[origin=c]{90}{\textsc{Blocks-World}} \rotatebox[origin=c]{90}{(1076)}} 
& \multirow{5}{*}{20} 
										 & 10         & 1.8
											& 1271.282 / 398.090 & 41.4\% / 68.2\% & 11.41 / 3.34

											& 36.562 & 65.8\% & 9.11
											
											& \timeout / \timeout & - / - & - / -     \\
                   &                     & 30         & 4.9        
											& 1280.655 / 481.905 & 75.2\% / 47.1\% & 7.76 / 1.84
											
											& 36.648 & 78.1\% & 10.53
											
											& \timeout / \timeout & - / - & - / -     \\
                   &                     & 50         & 7.6        
											& 1284.269 / 445.196 & 84.1\% / 52.4\% & 7.24 / 1.78

											& 34.290 & 81.3\% & 10.68

											& \timeout / \timeout & - / - & - / -     \\
                   &                     & 70         & 11.1        
											& 1296.773 / 395.902 & 88.2\% / 59.3\% & 8.23 / 1.67

											& 37.056 & 89.8\% & 8.63
											
											& \timeout / \timeout & - / - & - / -     \\
                   &                     & 100        & 14.5       
											& 1305.220 / 288.751 & 94.5\% / 65.2\% & 8.66 / 1.71

											& 40.405 & 100.0\% & 1.22
											
											& \timeout / \timeout & - / - & - / -     \\ \hline

\multirow{5}{*}{\rotatebox[origin=c]{90}{\textsc{Campus}} \rotatebox[origin=c]{90}{(75)}} 
& \multirow{5}{*}{2} 
										 & 10         & 1
											& 1.021 / 1.038 & 93.3\% / 93.3\% & 1.33 / 1.33

											& 0.717 & 53.3\% & 1.0
											
											& 45.749 / 44.834 & 53.3\% / 53.3\% & 1.0 / 1.0     \\
                   &                     & 30         & 2        
											& 1.113 / 1.090 & 100.0\% / 100.0\% & 1.0 / 1.0

											& 0.696 & 80.0\% & 1.13
											
											& 48.438 / 48.112 & 53.3\% / 53.3\% & 1.0 / 1.0     \\
                   &                     & 50         & 3       
											& 1.333 / 1.285 & 100.0\% / 100.0\% & 1.0 / 1.0

											& 0.676 & 66.6\% & 1.26
											
											& 54.111 / 53.402 & 40.0\% / 40.0\% & 1.0 / 1.0     \\
                   &                     & 70         & 4.4        
											& 1.725 / 1.706 & 100.0\% / 100.0\% & 1.0 / 1.0

											& 0.668 & 86.6\% & 1.6
											
											& 89.708 / 89.037 & 53.3\% / 53.3\% & 1.0 / 1.0     \\
                   &                     & 100        & 5.5       
											& 1.809 / 1.771 & 100.0\% / 100.0\% & 1.0 / 1.0

											& 0.631 & 93.3\% & 1.53
											
											& 183.123 / 182.910 & 60.0\% / 60.0\% & 1.0 / 1.0     \\ \hline
											
\multirow{5}{*}{\rotatebox[origin=c]{90}{\textsc{Depots}} \rotatebox[origin=c]{90}{(364)}} 
& \multirow{5}{*}{8.5} 
										 & 10         & 3.1
											& 1347.85 / 1166.02 & 59.5\% / 50.0\% & 4.10 / 1.61

											& $\dag$ & $\dag$  & $\dag$   
											
											& \timeout / \timeout & - / - & - / -     \\
                   &                     & 30         & 8.6        
											& 1369.22 / 1037.97 & 44.0\% / 52.3\% & 2.44 / 0.90

											& $\dag$ & $\dag$  & $\dag$   
											
											& \timeout / \timeout & - / - & - / -     \\
                   &                     & 50         & 14.1       
											& 1335.18 / 1034.36 & 48.8\% / 50.0\% & 2.86 / 0.75

											& $\dag$ & $\dag$  & $\dag$   
											
											& \timeout / \timeout & - / - & - / -     \\
                   &                     & 70         & 19.7        
											& 1392.55 / 853.40 & 55.9\% / 61.9\% & 3.32 / 0.74

											& $\dag$ & $\dag$ & $\dag$   
											
											& \timeout / \timeout & - / - & - / -     \\
                   &                     & 100        & 24.4       
											& 1370.81 / 670.99 & 67.8\% / 75.0\% & 4.39 / 0.72

											& $\dag$ & $\dag$ & $\dag$   
											
											& \timeout / \timeout & - / - & - / -     \\ \hline
											
\multirow{5}{*}{\rotatebox[origin=c]{90}{\textsc{Driver-Log}} \rotatebox[origin=c]{90}{(364)}} 
& \multirow{5}{*}{10.5} 
										 & 10         & 2.6
											& 737.530 / 509.306 & 51.2\% / 52.4\% & 1.64 / 0.89

											& 79.487 & 42.8\% & 1.91
											
											& \timeout / \timeout & - / - & - / -     \\
                   &                     & 30         & 6.9        
											& 846.176 / 438.371 & 50.0\% / 60.7\% & 1.58 / 0.74

											& 60.168 & 70.2\% & 3.19
											
											& \timeout / \timeout & - / - & - / -     \\
                   &                     & 50         & 11.1       
											& 851.659 / 379.450 & 48.8\% / 70.2\% & 1.27 / 0.75

											& 64.427 & 79.7\% & 4.59
											
											& \timeout / \timeout & - / - & - / -     \\
                   &                     & 70         & 15.6        
											& 891.158 / 308.775 & 54.8\% / 91.7\% & 1.61 / 0.93

											& 75.084 & 82.1\% & 4.10
											
											& \timeout / \timeout & - / - & - / -     \\
                   &                     & 100        & 21.7       
											& 945.013 / 196.093 & 46.4\% / 96.4\% & 1.39 / 0.96

											& 96.091 & 96.4\% & 1.11
											
											& \timeout / \timeout & - / - & - / -     \\ \hline
\multirow{5}{*}{\rotatebox[origin=c]{90}{\textsc{DWR}} \rotatebox[origin=c]{90}{(364)}} 
& \multirow{5}{*}{7.25} 
										 & 10         & 5.7
											& 1246.506 / 1079.001 & 61.9\% / 56.0\% & 2.99 / 1.63

											& 66.496 & 92.8\% & 6.38
											
											& \timeout / \timeout & - / - & - / -     \\
                   &                     & 30         & 16        
											& 1476.392 / 1129.304 & 29.8\% / 39.3\% & 1.63 / 0.75

											& 54.461 & 97.6\% & 6.56
											
											& \timeout / \timeout & - / - & - / -     \\
                   &                     & 50         & 26.2        
											& 1501.524 / 1065.484 & 33.3\% / 40.5\% & 2.01 / 0.62

											& 56.255 & 98.8\% & 6.27
											
											& \timeout / \timeout & - / - & - / -     \\
                   &                     & 70         & 36.8        
											& 1505.305 / 980.959 & 34.5\% / 48.8\% & 2.17 / 0.63

											& 65.101 & 98.8\% & 6.0
											
											& \timeout / \timeout & - / - & - / -     \\
                   &                     & 100        & 51.9       
											& 1351.309 / 891.953 & 57.1\% / 57.1\% & 3.57 / 0.64

											& 86.459 & 100.0\% & 1.0
											
											& \timeout / \timeout & - / - & - / -     \\ \hline
\multirow{5}{*}{\rotatebox[origin=c]{90}{\textsc{IPC-Grid}} \rotatebox[origin=c]{90}{(673)}} 
& \multirow{5}{*}{9} 
										 & 10         & 2.9
											& 259.349 / 127.138 & 66.0\% / 62.7\% & 2.46 / 1.35

											& \timeout & -  & -   
											
											& \timeout / \timeout & - / - & - / -     \\
                   &                     & 30         & 7.8        
											& 377.482 / 143.669 & 85.6\% / 81.0\% & 1.44 / 0.90

											& \timeout & -  & -   
											
											& \timeout / \timeout & - / - & - / -     \\
                   &                     & 50         & 12.7        
											& 516.035 / 100.172 & 85.0\% / 90.2\% & 1.39 / 0.93

											& \timeout & -  & -   
											
											& \timeout / \timeout & - / - & - / -     \\
                   &                     & 70         & 17.9        
											& 639.157 / 140.685 & 81.7\% / 94.8\% & 1.76 / 0.95

											& \timeout & -  & -   
											
											& \timeout / \timeout & - / - & - / -     \\
                   &                     & 100        & 24.8       
											& 708.007 / 177.43 & 93.4\% / 100.0\% & 2.54 / 1.0

											& \timeout & -  & -   
											
											& \timeout / \timeout & - / - & - / -     \\ \hline
																						
\multirow{5}{*}{\rotatebox[origin=c]{90}{\textsc{Ferry}} \rotatebox[origin=c]{90}{(364)}} 
& \multirow{5}{*}{7.5} 
										 & 10         & 2.9
											& 251.648 / 120.596 & 89.3\% / 90.5\% & 2.08 / 1.44

											& 6.659 & 91.6\% & 6.65
											
											& \timeout / \timeout & - / - & - / -     \\
                   &                     & 30         & 7.6        
											& 425.151 / 91.857 & 83.3\% / 97.6\% & 1.31 / 1.06

											& 6.801 & 100.0\% & 7.57
											
											& \timeout / \timeout & - / - & - / -     \\
                   &                     & 50         & 12.3        
											& 662.567 / 54.453 & 66.7\% / 98.8\% & 1.24 / 1.02

											& 8.296 & 100.0\% & 7.57
											
											& \timeout / \timeout & - / - & - / -     \\
                   &                     & 70         & 17.3        
											& 820.501 / 40.898 & 59.5\% / 100.0\% & 1.26 / 1.0

											& 10.649 & 100.0\% & 7.32
											
											& \timeout / \timeout & - / - & - / -     \\
                   &                     & 100        & 24.2       
											& 1015.216 / 46.148 & 50.0\% / 100.0\% & 1.29 / 1.0

											& 13.625 & 100.0\% & 1.07
											
											& \timeout / \timeout & - / - & - / -     \\ \hline
																																		
\multirow{5}{*}{\rotatebox[origin=c]{90}{\textsc{Intrusion}} \rotatebox[origin=c]{90}{(465)}} 
& \multirow{5}{*}{15} 
										 & 10         & 1.9
											& 5.683 / 2.889 & 73.3\% / 81.0\% & 3.66 / 2.37

											& 0.475 & 89.5\% & 3.18
											
											& \timeout / \timeout & - / - & - / -     \\
                   &                     & 30         & 4.5        
											& 5.908 / 4.348 & 100.0\% / 100.0\% & 1.11 / 1.11

											& 0.476 & 90.5\% & 1.88
											
											& \timeout / \timeout & - / - & - / -     \\
                   &                     & 50         & 6.7        
											& 6.248 / 4.807 & 100.0\% / 100.0\% & 1.02 / 1.02

											& 0.496 & 94.3\% & 1.45
											
											& \timeout / \timeout & - / - & - / -     \\
                   &                     & 70         & 9.5        
											& 6.665 / 5.261 & 100.0\% / 100.0\% & 1.0 / 1.0

											& 0.637 & 99.1\% & 1.05
											
											& \timeout / \timeout & - / - & - / -     \\
                   &                     & 100        & 13.1       
											& 7.372 / 5.815 & 100.0\% / 100.0\% & 1.0 / 1.0

											& 0.828 & 100.0\% & 1.04
											
											& \timeout / \timeout & - / - & - / -     \\ \hline
\end{tabular}
\caption{Experiments and evaluation with missing and full observations for R\&G 2010 using \textsc{Fast-Downward} with LM-Cut heuristic, FGR 2015~\cite{NASA_GoalRecognition_IJCAI2015}, and IBM 2016 using TK$^*$ with LM-Cut heuristic, top-1000 (Part 1).}
\label{tab:goalRecognitionResults1_2}
\end{table*}
\end{landscape}
}

\afterpage{
\begin{landscape}
\begin{table*}[]
\centering
\fontsize{5}{8}\selectfont
\setlength\tabcolsep{3pt}
\begin{tabular}{|c|c|cc|ccc|ccc|ccc|}
\hline
                   &                     & \multicolumn{2}{c|}{} & \multicolumn{3}{c|}{{\footnotesize $\mathit{h_{gc}}$}}                                   & \multicolumn{3}{c|}{{\footnotesize $\mathit{h_{uniq}}$}}                                 & \multicolumn{3}{c|}{\scriptsize R\&G 2009 / Filter$_{10\%}$ + R\&G 2009}    \\ \hline
\#                 & $|\mathcal{G}|$                   & \textbf{\% Obs}        & $|O|$        & \textbf{\begin{tabular}[c]{@{}c@{}}Time\\$\theta$ (0 / 10 / 20)\end{tabular}}                  & \textbf{\begin{tabular}[c]{@{}c@{}}Accuracy\\$\theta$ (0 / 10 / 20)\end{tabular}}              & \textbf{\begin{tabular}[c]{@{}c@{}}Spread in $\mathcal{G}$\\$\theta$ (0 / 10 / 20)\end{tabular}}  & \textbf{\begin{tabular}[c]{@{}c@{}}Time\\$\theta$ (0 / 10 / 20)\end{tabular}}                  & \textbf{\begin{tabular}[c]{@{}c@{}}Accuracy\\$\theta$ (0 / 10 / 20)\end{tabular}}              & \textbf{\begin{tabular}[c]{@{}c@{}}Spread in $\mathcal{G}$\\$\theta$ (0 / 10 / 20)\end{tabular}}  & \textbf{Time}          & \textbf{Accuracy}      & \textbf{Spread in $\mathcal{G}$} \\ \hline

\multirow{5}{*}{\rotatebox[origin=c]{90}{\textsc{Kitchen}} \rotatebox[origin=c]{90}{(75)}} 
& \multirow{5}{*}{3} 
										 & 10         & 1.3
											& 0.003 / 0.003 / 0.004 & 93.3\% / 100\% / 100\% & 1.46 / 2.33 / 3 
											
											& 0.002 / 0.003 / 0.003 & 100\% / 100\% / 100\% & 1.33 / 2.60 / 3
											
											& 0.085 / 0.084 & 100\% / 100\% & 1.86 / 1.86     \\
                   &                     & 30         & 3.5        
											& 0.004 / 0.005 / 0.006 & 93.3\% / 100\% / 100\% & 1.46 / 2.33 / 3 
											
											& 0.003 / 0.004 / 0.005 & 100\% / 100\% / 100\% & 1.33 / 2.60 / 3
											
											& 0.097 / 0.098 & 100\% / 100\% & 1.33 / 1.33     \\
                   &                     & 50         & 4        
											& 0.003 / 0.003 / 0.003 & 93.3\% / 100\% / 100\% & 1.46 / 2.33 / 3
											
											& 0.006 / 0.005 / 0.007 & 100\% / 100\% / 100\% & 1.33 / 2.60 / 3
											
											& 0.104 / 0.103 & 100\% / 100\% & 1.46 / 1.46     \\
                   &                     & 70         & 5        
											& 0.008 / 0.008 / 0.009    & 93.3\% / 100\% / 100\% & 1.46 / 2.60 / 2.60 
											
											& 0.006 / 0.007 / 0.007    & 100\% / 100\% / 100\% & 1.46 / 2.33 / 2.60 
											
											& 0.115 / 0.116 & 100\% / 100\% & 1.26 / 1.26     \\
                   &                     & 100        & 7.4       
											& 0.008 / 0.009 / 0.009   & 100\% / 100\% / 100\% & 1.0 / 1.0 / 1.0
											
											& 0.007 / 0.006 / 0.008   & 100\% / 100\% / 100\% & 1.0 / 1.0 / 1.0
											
											& 0.119 / 0.115 & 100\% / 100\% & 1.26 / 1.26     \\ \hline

\multirow{5}{*}{\rotatebox[origin=c]{90}{\textsc{Logistics}} \rotatebox[origin=c]{90}{(673)}} 
& \multirow{5}{*}{10.5} 
										 & 10         & 2.9
											& 0.614 / 0.618 / 0.629 & 55.5\% / 84.3\% / 94.7\% & 1.72 / 3.58 / 4.22
											
											& 0.563 / 0.572 / 0.579 & 55.5\% / 77.1\% / 95.4\% & 1.24 / 2.84 / 3.50
											
											& 1.201 / 0.705 & 99.3\% / 90.8\% & 2.98 / 2.69     \\
                   &                     & 30         & 8.2        
											& 0.632 / 0.637 / 0.644 & 80.3\% / 95.4\% / 99.3\% & 1.20 / 2.14 / 3.58
											
											& 0.571 / 0.580 / 0.584 & 76.4\% / 86.9\% / 98.1\% & 1.20 / 1.96 / 2.84
											
											& 1.798 / 1.166 & 98.6\% / 97.3\% & 1.39 / 1.35     \\
                   &                     & 50         & 13.4        
											& 0.656 / 0.662 / 0.675 & 90.1\% / 99.3\% / 100\% & 1.10 / 1.92 / 2.14
											
											& 0.599 / 0.603 / 0.607 & 86.2\% / 95.4\% / 100\% & 1.10 / 1.50 / 1.77
											
											& 2.545 / 1.212 & 98.6\% / 98.6\% & 1.29 / 1.15     \\
                   &                     & 70         & 18.9        
											& 0.670 / 0.674 / 0.683 & 96.7\% / 98.6\% / 100\% & 1.05 / 1.39 / 1.57
											
											& 0.608 / 0.611 / 0.622 & 96.7\% / 98.6\% / 100\% & 1.05 / 1.47 / 1.50
											
											& 3.460 / 1.503 & 100\% / 100\% & 1.13 / 1.11     \\
                   &                     & 100        & 26.5       
											& 0.681 / 0.685 / 0.692 & 100\% / 100\% / 100\% & 1.0 / 1.0 / 1.27
											
											& 0.615 / 0.624 / 0.631 & 100\% / 100\% / 100\% & 1.0 / 1.0 / 1.08
											
											& 4.887 / 1.691 & 100\% / 100\% & 1.0 / 1.0     \\ \hline
											
\multirow{5}{*}{\rotatebox[origin=c]{90}{\textsc{Miconic}} \rotatebox[origin=c]{90}{(364)}} 
& \multirow{5}{*}{6} 
										 & 10         & 3.9
											& 0.350 / 0.361 / 0.366 & 67.8\% / 98.8\% / 100\% & 1.33 / 3.28 / 4.34
											
											& 0.321 / 0.325 / 0.334 & 54.7\% / 97.6\% / 100\% & 1.26 / 3.40 / 4.05
											
											& 0.838 / 0.521 & 100\% / 98.8\% & 3.26 / 3.15     \\
                   &                     & 30         & 11.1        
											& 0.357 / 0.369 / 0.370 & 96.4\% / 100\% / 100\% & 1.10 / 2.27 / 3.86
											
											& 0.326 / 0.333 / 0.341 & 90.1\% / 100\% / 100\% & 1.08 / 2.53 / 3.86
											
											& 1.196 / 0.707 & 100\% / 100\% & 1.58 / 1.55     \\
                   &                     & 50         & 18.1        
											& 0.368 / 0.373 / 0.375 & 96.4\% / 100\% / 100\% & 1.01 / 1.54 / 2.11
											
											& 0.339 / 0.342 / 0.352 & 96.4\% / 100\% / 100\% & 1.01 / 1.47 / 1.83
											
											& 1.722 / 1.099 & 100\% / 100\% & 1.28 / 1.27    \\
                   &                     & 70         & 25.3        
											& 0.372 / 0.378 / 0.384 & 100\% / 100\% / 100\% & 1.01 / 1.20 / 1.57
											
											& 0.344 / 0.357 / 0.365 & 100\% / 100\% / 100\% & 1.0 / 1.21 / 1.57
											
											& 2.504 / 1.516 & 100\% / 100\% & 1.03 / 1.03     \\
                   &                     & 100        & 35.6       
											& 0.389 / 0.394 / 0.397 & 100\% / 100\% / 100\% & 1.0 / 1.0 / 1.54
											
											& 0.356 / 0.363 / 0.372 & 100\% / 100\% / 100\% & 1.0 / 1.0 / 1.47
											
											& 5.105 / 2.013 & 100\% / 100\% & 1.0 / 1.0     \\ \hline

\multirow{5}{*}{\rotatebox[origin=c]{90}{\textsc{Rovers}} \rotatebox[origin=c]{90}{(364)}} 
& \multirow{5}{*}{6} 
										 & 10         & 3
											& 0.342 / 0.343 / 0.350 & 64.2\% / 91.6\% / 96.4\% & 1.72 / 2.45 / 3.83
											
											& 0.310 / 0.318 / 0.324 & 51.1\% / 79.7\% / 95.2\% & 1.10 / 3.01 / 3.42
											
											& 0.704 / 0.512 & 98.8\% / 95.2\% & 2.85 / 2.29     \\
                   &                     & 30         & 7.9        
											& 0.347 / 0.358 / 0.361 & 83.3\% / 91.1\% / 100\% & 1.23 / 2.14 / 3.72
											
											& 0.323 / 0.322 / 0.335 & 69.1\% / 90.1\% / 97.6\% & 1.07 / 2.19 / 2.46
											
											& 1.029 / 0.787 & 100\% / 97.6\% & 1.66 / 1.41     \\
                   &                     & 50         & 12.7        
											& 0.374 / 0.383 / 0.375 & 92.8\% / 96.4\% / 100\% & 1.08 / 1.72 / 3.01
											
											& 0.331 / 0.338 / 0.344 & 85.7\% / 95.2\% / 97.6\% & 1.01 / 1.57 / 2.19
											
											& 1.355 / 0.841 & 100\% / 98.8\% & 1.29 / 1.17     \\
                   &                     & 70         & 17.9        
											& 0.389 / 0.382 / 0.391 & 98.8\% / 100\% / 100\% & 1.01 / 1.35 / 2.14
											
											& 0.345 / 0.350 / 0.353 & 91.6\% / 98.8\% / 100\% & 1.0 / 1.20 / 1.58
											
											& 1.796 / 1.008 & 100\% / 100\% & 1.07 / 1.05     \\
                   &                     & 100        & 24.9       
											& 0.392 / 0.394 / 0.396 & 100\% / 100\% / 100\% & 1.0 / 1.07 / 1.35
											
											& 0.356 / 0.361 / 0.365 & 100\% / 100\% / 100\% & 1.0 / 1.03 / 1.25
											
											& 2.292 / 1.314 & 100\% / 100\% & 1.07 / 1.0     \\ \hline

\multirow{5}{*}{\rotatebox[origin=c]{90}{\textsc{Satellite}} \rotatebox[origin=c]{90}{(364)}} 
& \multirow{5}{*}{6.5} 
										 & 10         & 2.1
											& 0.458 / 0.466 / 0.471 & 57.1\% / 85/7\% / 100\% & 1.55 / 2.33 / 2.88
											
											& 0.431 / 0.445 / 0.456 & 47.6\% / 79.7\% / 96.4\% & 1.21 / 2.09 / 2.33
											
											& 1.049 / 0.599 & 97.6\% / 90.4\% & 3.41 / 3.02     \\
                   &                     & 30         & 5.4        
											& 0.465 / 0.474 / 0.482 & 76.1\% / 94.1\% / 100\% & 1.31 / 1.92 / 2.13 
											
											& 0.442 / 0.454 / 0.460 & 69.1\% / 89.2\% / 97.6\% & 1.14 / 1.91 / 2.09
											
											& 1.182 / 0.723 & 97.6\% / 96.4\% & 2.40 / 1.94     \\
                   &                     & 50         & 8.7        
											& 0.472 / 0.485 / 0.494 & 85.7\% / 98.8\% / 100\% & 1.09 / 1.48 / 1.91
											
											& 0.458 / 0.463 / 0.477 & 80.9\% / 91.6\% / 98.8\% & 1.10 / 1.63 / 1.91
											
											& 1.398 / 0.901 & 97.6\% / 97.6\% & 1.69 / 1.47     \\
                   &                     & 70         & 12.2        
											& 0.489 / 0.490 / 0.498 & 97.6\% / 100\% / 100\% & 1.07 / 1.48 / 1.75
											
											& 0.460 / 0.471 / 0.486 & 94.1\% / 98.8\% / 100\% & 1.03 / 1.34 / 1.63
											
											& 1.884 / 1.076 & 96.4\% / 100\% & 1.52 / 1.25     \\
                   &                     & 100        & 16.8       
											& 0.491 / 0.499 / 0.512 & 100\% / 100\% / 100\% & 1.02 / 1.21 / 1.63
											
											& 0.475 / 0.482 / 0.490 & 100\% / 100\% / 100\% & 1.07 / 1.21 / 1.50
											
											& 2.107 / 1.224 & 96.4\% / 100\% & 1.33 / 1.10     \\ \hline

\multirow{5}{*}{\rotatebox[origin=c]{90}{\textsc{Sokoban}} \rotatebox[origin=c]{90}{(364)}} 
& \multirow{5}{*}{7.25} 
										 & 10         & 3.1
											& 0.549 / 0.552 / 0.554 & 53.5\% / 86.9\% / 88.1\% & 2.05 / 2.89 / 3.78
											
											& 0.523 / 0.530 / 0.539 & 51.1\% / 67.8\% / 88.1\% & 1.85 / 2.78 / 3.04
											
											& 3.025 / 1.857 & 69.1\% / 71.4\% & 4.02 / 2.51     \\
                   &                     & 30         & 8.7        
											& 0.555 / 0.560 / 0.562 & 57.1\% / 77.3\% / 84.5\% & 1.36 / 1.81 / 2.69
											
											& 0.531 / 0.538 / 0.543 & 55.9\% / 69.1\% / 84.5\% & 1.21 / 1.77 / 2.69
											
											& 4.429 / 2.081 & 89.2\% / 76.1\% & 4.10 / 1.67     \\
                   &                     & 50         & 14.1        
											& 0.568 / 0.571 / 0.573 & 71.4\% / 88.1\% / 94.1\% & 1.32 / 1.80 / 2.02
											
											& 0.540 / 0.544 / 0.551 & 69.1\% / 83.3\% / 91.6\% & 1.20 / 1.80 / 1.82
											
											& 7.553 / 2.409 & 89.2\% / 85.7\% & 4.16 / 1.63     \\
                   &                     & 70         & 19.8        
											& 0.577 / 0.580 / 0.585 & 83.3\% / 91.6\% / 96.4\% & 1.04 / 1.31 / 1.80
											
											& 0.554 / 0.556 / 0.558 & 86.9\% / 92.8\% / 95.2\% & 1.08 / 1.60 / 1.77
											
											& 9.112 / 2.572 & 89.2\% / 86.9\% & 4.17 / 1.19     \\
                   &                     & 100        & 35.5       
											& 0.586 / 0.591 / 0.598 & 100\% / 100\% / 100\% & 1.0 / 1.0 / 1.0
											
											& 0.562 / 0.572 / 0.574 & 100\% / 100\% / 100\% & 1.0 / 1.03 / 1.28
											
											& 12.008 / 2.610 & 89.2\% / 100\% & 4.53 / 1.03     \\ \hline
																						
\multirow{5}{*}{\rotatebox[origin=c]{90}{\textsc{Zeno-Travel}} \rotatebox[origin=c]{90}{(364)}} 
& \multirow{5}{*}{7.5} 
										 & 10         & 2.6
											& 0.502 / 0.511 / 0.528 & 39.2\% / 55.9\% / 80.9\% & 1.15 / 1.92 / 3.04 
											
											& 0.491 / 0.502 / 0.509 & 36.9\% / 48.8\% / 70.2\% & 1.04 / 1.92 / 2.13
											
											& 1.834 / 1.207 & 96.4\% / 66.6\% & 3.41 / 1.58     \\
                   &                     & 30         & 6.7        
											& 0.517 / 0.523 / 0.536 & 70.2\% / 78.5\% / 91.6\% & 1.10 / 1.73 / 2.26
											
											& 0.504 / 0.515 / 0.520 & 60.7\% / 79.7\% / 90.4\% & 1.02 / 1.71 / 1.73
											
											& 2.528 / 1.396 & 88.1\% / 79.7\% & 2.11 / 1.29     \\
                   &                     & 50         & 10.8        
											& 0.521 / 0.534 / 0.544 & 78.5\% / 86.9\% / 95.2\% & 1.07 / 1.40 / 1.71
											
											& 0.516 / 0.521 / 0.528 & 76.1\% / 88.1\% / 95.2\% & 1.0 / 1.57 / 1.61
											
											& 3.071 / 1.513 & 92.8\% / 90.4\% & 1.41 / 1.14     \\
                   &                     & 70         & 15.2        
											& 0.535 / 0.542 / 0.550 & 97.6\% / 97.6\% / 100\% & 1.04 / 1.14 / 1.40
											
											& 0.522 / 0.533 / 0.539 & 90.4\% / 95.2\% / 100\% & 1.0 / 1.29 / 1.57
											
											& 3.986 / 1.605 & 96.4\% / 100\% & 1.13 / 1.0     \\
                   &                     & 100        & 21.1       
											& 0.548 / 0.555 / 0.564 & 100\% / 100\% / 100\% & 1.0 / 1.0 / 1.07
											
											& 0.530 / 0.541 / 0.552 & 100\% / 100\% / 100\% & 1.0 / 1.07 / 1.10
											
											& 4.815 / 1.722 & 100\% / 100\% & 1.07 / 1.0     \\ \hline
																					
\end{tabular}
\caption{Experiments and evaluation with missing and full observations for $\mathit{h_{gc}}$, $\mathit{h_{uniq}}$, R\&G 2009, and our filtering method (10\% of threshold) with R\&G 2009 (Part 2).}
\label{tab:goalRecognitionResults2_1}
\end{table*}
\end{landscape}
}

\afterpage{
\begin{landscape}
\begin{table*}[]
\centering
\fontsize{5}{8}\selectfont
\setlength\tabcolsep{3pt}
\begin{tabular}{|c|c|cc|ccc|ccc|ccc|}
\hline
                   &                     
				   & \multicolumn{2}{c|}{} & \multicolumn{3}{c|}{\begin{tabular}[c]{@{}c@{}}\scriptsize R\&G 2010 / Filter$_{10\%}$ + R\&G 2010 \\ \scriptsize (\textsc{Fast-Downward} with LM-Cut heuristic)\end{tabular}}                        
				   
				   & \multicolumn{3}{c|}{\begin{tabular}[c]{@{}c@{}}\scriptsize FGR 2015\end{tabular}}				   
				     
				   & \multicolumn{3}{c|}{\begin{tabular}[c]{@{}c@{}}\scriptsize IBM 2016 / Filter$_{10\%}$ + IBM 2016 \\ \scriptsize (TK$^*$ with LM-Cut heuristic, top-1000)\end{tabular}}    \\ \hline
\#                 & $|\mathcal{G}|$                   & \textbf{\% Obs}        & $|O|$        

& \textbf{\begin{tabular}[c]{@{}c@{}}Time\end{tabular}}                  
& \textbf{\begin{tabular}[c]{@{}c@{}}Accuracy\end{tabular}}              
& \textbf{\begin{tabular}[c]{@{}c@{}}Spread in $\mathcal{G}$\end{tabular}}

& \textbf{Time}          
& \textbf{Accuracy}      
& \textbf{Spread in $\mathcal{G}$}

& \textbf{Time}          
& \textbf{Accuracy}      
& \textbf{Spread in $\mathcal{G}$} \\ \hline

\multirow{5}{*}{\rotatebox[origin=c]{90}{\textsc{Kitchen}} \rotatebox[origin=c]{90}{(75)}} 
& \multirow{5}{*}{3} 
										 & 10         & 1.3
											& 1.310 / 1.222 & 93.3\% / 93.3\% & 1.33 / 1.33

											& 0.373 & 100.0\% & 1.86
											
											& \timeout / \timeout & - / - & - / -     \\
                   &                     & 30         & 3.5        
											& 1.365 / 1.238 & 93.3\% / 93.3\% & 1.20 / 1.13

											& 0.360 & 100.0\% & 1.33
											
											& \timeout / \timeout & - / - & - / -     \\
                   &                     & 50         & 4        
											& 1.571 / 1.438 & 100.0\% / 100.0\% & 1.33 / 1.33

											& 0.392 & 100.0\% & 1.33
											
											& \timeout / \timeout & - / - & - / -     \\
                   &                     & 70         & 5        
											& 1.702 / 1.609 & 100.0\% / 100.0\% & 1.20 / 1.20

											& 0.378 & 100.0\% & 1.20
											
											& \timeout / \timeout & - / - & - / -     \\
                   &                     & 100        & 7.4       
											& 2.144 / 1.983 & 100.0\% / 100.0\% & 1.40 / 1.40

											& 0.483 & 100.0\% & 1.40
											
											& \timeout / \timeout & - / - & - / -     \\ \hline

\multirow{5}{*}{\rotatebox[origin=c]{90}{\textsc{Logistics}} \rotatebox[origin=c]{90}{(673)}} 
& \multirow{5}{*}{10.5} 
										 & 10         & 2.9
											& 334.315 / 288.357 & 65.4\% / 69.3\% & 3.57 / 1.52

											& $\dag$ & $\dag$ & $\dag$   
											
											& \timeout / \timeout & - / - & - / -     \\
                   &                     & 30         & 8.2        
											& 411.948 / 321.727 & 81.7\% / 83.0\% & 2.26 / 0.99

											& $\dag$ & $\dag$ & $\dag$
											
											& \timeout / \timeout & - / - & - / -     \\
                   &                     & 50         & 13.4        
											& 431.775 / 218.848 & 78.4\% / 88.9\% & 1.97 / 0.93

											& $\dag$ & $\dag$ & $\dag$
											
											& \timeout / \timeout & - / - & - / -     \\
                   &                     & 70         & 18.9        
											& 409.629 / 181.421 & 83.0\% / 92.2\% & 2.22 / 0.94

											& $\dag$ & $\dag$ & $\dag$
											
											& \timeout / \timeout & - / - & - / -     \\
                   &                     & 100        & 26.5       
											& 310.563 / 109.189 & 91.8\% / 95.1\% & 2.43 / 0.95

											& $\dag$ & $\dag$ & $\dag$
											
											& \timeout / \timeout & - / - & - / -     \\ \hline
											
\multirow{5}{*}{\rotatebox[origin=c]{90}{\textsc{Miconic}} \rotatebox[origin=c]{90}{(364)}} 
& \multirow{5}{*}{6} 
										 & 10         & 3.9
											& 313.117 / 258.977 & 89.3\% / 92.9\% & 2.07 / 1.63

											& $\dag$ & $\dag$ & $\dag$
											
											& \timeout / \timeout & - / - & - / -     \\
                   &                     & 30         & 11.1        
											& 590.095 / 338.774 & 59.5\% / 88.1\% & 1.18 / 1.0

											& $\dag$ & $\dag$ & $\dag$
											
											& \timeout / \timeout & - / - & - / -     \\
                   &                     & 50         & 18.1        
											& 578.263 / 224.679 & 61.9\% / 100.0\% & 1.10 / 1.04

											& $\dag$ & $\dag$ & $\dag$
											
											& \timeout / \timeout & - / - & - / -     \\
                   &                     & 70         & 25.3        
											& 577.087 / 123.984 & 61.9\% / 100.0\% & 1.07 / 1.0

											& $\dag$ & $\dag$ & $\dag$
											
											& \timeout / \timeout & - / - & - / -     \\
                   &                     & 100        & 35.6       
											& 155.422 / 30.114 & 100.0\% / 100.0\% & 1.0 / 1.0

											& $\dag$ & $\dag$ & $\dag$
											
											& \timeout / \timeout & - / - & - / -     \\ \hline

\multirow{5}{*}{\rotatebox[origin=c]{90}{\textsc{Rovers}} \rotatebox[origin=c]{90}{(364)}} 
& \multirow{5}{*}{6} 
										 & 10         & 3
											& 551.746 / 524.997 & 88.1\% / 48.8\% & 4.18 / 0.95

											& $\dag$ & $\dag$ & $\dag$
											
											& \timeout / \timeout & - / - & - / -     \\
                   &                     & 30         & 7.9        
											& 589.213 / 518.466 & 94.0\% / 57.1\% & 3.29 / 0.64

											& $\dag$ & $\dag$ & $\dag$
											
											& \timeout / \timeout & - / - & - / -     \\
                   &                     & 50         & 12.7        
											& 640.802 / 518.649 & 88.1\% / 56.0\% & 3.20 / 0.62

											& $\dag$ & $\dag$ & $\dag$
											
											& \timeout / \timeout & - / - & - / -     \\
                   &                     & 70         & 17.9        
											& 641.881 / 517.961 & 81.0\% / 57.1\% & 3.04 / 0.61

											& $\dag$ & $\dag$ & $\dag$
											
											& \timeout / \timeout & - / - & - / -     \\
                   &                     & 100        & 24.9       
											& 589.669 / 518.319 & 85.7\% / 57.1\% & 3.0 / 0.57

											& $\dag$ & $\dag$ & $\dag$
											
											& \timeout / \timeout & - / - & - / -     \\ \hline

\multirow{5}{*}{\rotatebox[origin=c]{90}{\textsc{Satellite}} \rotatebox[origin=c]{90}{(364)}} 
& \multirow{5}{*}{6.5} 
										 & 10         & 2.1
											& 477.756 / 418.734 & 65.5\% / 81.0\% & 2.40 / 1.74

											& 14.821 & 89.3\% & 4.86
											
											& \timeout / \timeout & - / - & - / -     \\
                   &                     & 30         & 5.4        
											& 488.884 / 278.747 & 79.8\% / 89.3\% & 1.98 / 1.30

											& 32.172 & 86.9\% & 4.21
											
											& \timeout / \timeout & - / - & - / -     \\
                   &                     & 50         & 8.7        
											& 553.301 / 208.331 & 79.8\% / 90.5\% & 1.79 / 1.11

											& 51.567 & 88.1\% & 3.65
											
											& \timeout / \timeout & - / - & - / -     \\
                   &                     & 70         & 12.2        
											& 520.356 / 186.682 & 79.8\% / 94.0\% & 1.54 / 1.06

											& 75.363 & 92.8\% & 2.89
											
											& \timeout / \timeout & - / - & - / -     \\
                   &                     & 100        & 16.8       
											& 455.197 / 172.089 & 82.1\% / 96.4\% & 1.71 / 1.04

											& 113.381 & 100.0\% & 2.57
											
											& \timeout / \timeout & - / - & - / -     \\ \hline

\multirow{5}{*}{\rotatebox[origin=c]{90}{\textsc{Sokoban}} \rotatebox[origin=c]{90}{(364)}} 
& \multirow{5}{*}{7.25} 
										 & 10         & 3.1
											& 637.342 / 377.933 & 70.8\% / 70.8\% & 1.31 / 0.85

											& 461.701 & 67.8\% & 2.98
											
											& \timeout / \timeout & - / - & - / -     \\
                   &                     & 30         & 8.7        
											& 850.272 / 310.082 & 51.4\% / 62.5\% & 1.15 / 0.63

											& 370.412 & 83.3\% & 3.14
											
											& \timeout / \timeout & - / - & - / -     \\
                   &                     & 50         & 14.1        
											& 1029.601 / 310.888 & 38.9\% / 68.1\% & 1.19 / 0.74

											& 358.028 & 82.1\% & 2.27
											
											& \timeout / \timeout & - / - & - / -     \\
                   &                     & 70         & 19.8        
											& 1082.685 / 176.873 & 37.5\% / 86.1\% & 1.11 / 0.88

											& 353.721 & 85.7\% & 1.84
											
											& \timeout / \timeout & - / - & - / -     \\
                   &                     & 100        & 35.5       
											& 1153.979 / 108.782 & 29.2\% / 95.8\% & 1.21 / 0.96

											& 353.183 & 85.7\% & 1.03
											
											& \timeout / \timeout & - / - & - / -     \\ \hline
																						
\multirow{5}{*}{\rotatebox[origin=c]{90}{\textsc{Zeno-Travel}} \rotatebox[origin=c]{90}{(364)}} 
& \multirow{5}{*}{7.5} 
										 & 10         & 2.6
											& 782.17 / 405.126 & 45.2\% / 53.6\% & 1.70 / 0.76

											& 93.917 & 66.6\% & 1.63
											
											& \timeout / \timeout & - / - & - / -     \\
                   &                     & 30         & 6.7        
											& 829.058 / 458.575 & 46.4\% / 66.7\% & 1.27 / 0.76

											& 88.285 & 78.6\% & 2.27
											
											& \timeout / \timeout & - / - & - / -     \\
                   &                     & 50         & 10.8        
											& 884.339 / 382.015 & 39.3\% / 71.4\% & 1.13 / 0.75

											& 105.814 & 91.6\% & 2.56
											
											& \timeout / \timeout & - / - & - / -     \\
                   &                     & 70         & 15.2        
											& 922.641 / 221.105 & 38.1\% / 81.0\% & 1.07 / 0.81

											& 125.652 & 94.1\% & 2.58
											
											& \timeout / \timeout & - / - & - / -     \\
                   &                     & 100        & 21.1       
											& 949.088 / 153.976 & 39.3\% / 89.3\% & 1.07 / 0.89

											& 168.674 & 100.0\% & 1.0
											
											& \timeout / \timeout & - / - & - / -     \\ \hline
																					
\end{tabular}
\caption{Experiments and evaluation with missing and full observations for R\&G 2010 using \textsc{Fast-Downward} with LM-Cut heuristic, FGR 2015~\cite{NASA_GoalRecognition_IJCAI2015}, and IBM 2016 using TK$^*$ with LM-Cut heuristic, top-1000 (Part 2).}
\label{tab:goalRecognitionResults2_2}
\end{table*}
\end{landscape}
}

\newpage
\subsubsection{Experimental Results with Missing, Noisy, and Full Observations}

Our second set of experiments uses datasets containing hundreds of problems for four domains with missing, noisy, and full observations.
Tables~\ref{tab:goalRecognitionResultsWithNoisy1} and \ref{tab:goalRecognitionResultsWithNoisy2} compare results for the experiments with missing, noisy, and full observations for our goal recognition heuristics (using a threshold between 0\% and 10\%) against R\&G 2009~\cite{RamirezG_IJCAI2009}, R\&G 2010~\cite{RamirezG_AAAI2010}, FGR 2015~\cite{NASA_GoalRecognition_IJCAI2015}, and~\citeauthor{Sohrabi_IJCAI2016}~(\citeyear{Sohrabi_IJCAI2016}), denoted as IBM 2016. 
We use our filtering method with 10\% of threshold alongside all these three approaches, denoted as Filter$_{10\%}$.
For this set of experiments, we used the same 4 domains used by~\citeauthor{Sohrabi_IJCAI2016}~(\citeyear{Sohrabi_IJCAI2016}). 
In these experiments, column $|N|$ represents the average number of noisy observations, \idest, extra observations that we added randomly to the observation sequence $O$. 
These two extra observations represent 12\% of noise regarding the total number of observations~\cite{Sohrabi_IJCAI2016}. 
Since \citeauthor{Sohrabi_IJCAI2016}~(\citeyear{Sohrabi_IJCAI2016}) timed out for all recognition problems, we are unable to provide a direct comparison of accuracy and runtime performance against this approach. 
However, given our understanding of the underlying technique, we believe that our approaches are almost certainly more computationally efficient, since they use a top-K planner~\cite{katz_etal_icaps18} during the recognition process (extracting 1000 sampled plans), much like \cite{RamirezG_AAAI2010}. 
The FGR 2015 approach is closer to ours in runtime performance for noisy observations, but still two to ten times slower. 

Under noisy observations, it is clear from the results in Tables~\ref{tab:goalRecognitionResultsWithNoisy1} and~\ref{tab:goalRecognitionResultsWithNoisy2} that the approaches R\&G 2009 and R\&G 2010 are not only much slower but substantially less accurate (with threshold value of 10\%) than our heuristics for virtually all 4 domains, reaching a low of 4.4 and 3.3 percent (respectively) of accuracy in the \textsc{IPC-Grid} domain. 
However, using the recognition threshold = 0\%, the R\&G 2009 and R\&G 2010 approaches are more accurate than our heuristics for two particular domains, more specifically, for \textsc{Intrusion} and \textsc{Campus} (respectively), while the FGR 2015 approach is more accurate than ours for the \textsc{Intrusion} domain, as well as \textsc{Campus} and \textsc{Kitchen} under some conditions (multiple noisy observations). 
Our uniqueness heuristic $\mathit{h_{uniq}}$ performed better (more accurate and faster) than the goal completion heuristic $\mathit{h_{gc}}$ for all 4 domains. 
Regarding the difference in accuracy for $h_{gc}$ and $h_{uniq}$ in the Kitchen domain under low observability, note the number of useful actions actually observed (2.5, 2 of which are known to be noise) at that observability level. 
This means that on average, in this domain, each experiment will have seen mostly noise and possibly one or two actions, or at times, no non-noisy action. 
Under these conditions being able to get the most information out of the observation (and correctly ignoring noise) is key. 
Here, the analysis of propositions that are not landmarks performed by the FGR 2015 approach seems to allow coping with a substantial amount of noisy versus non-noisy observations better than our approach. 
Note that by increasing the threshold parameter, we increase the spread to closer values to FGR 2015 and reach similar levels of accuracy. 

Figure~\ref{fig:rocspace_alldomains-noisy} shows the trade-off between true positive results and false positive results in a ROC space for all 4 domains with missing, noisy, and full observations. 
Figures~\ref{fig:recognition_time-noisy_1} and~\ref{fig:recognition_time-noisy_2} show a comparison of recognition time for our heuristics against the approaches R\&G 2009 and R\&G 2010. We used separate graphs for R\&G 2010, Filter$_{10\%} + $ R\&G 2010, and FGR 2015 given the widely different magnitude of the time taken to recognize a goal.

\begin{figure}[ht!]
  \centering
  \includegraphics[width=0.85\linewidth]{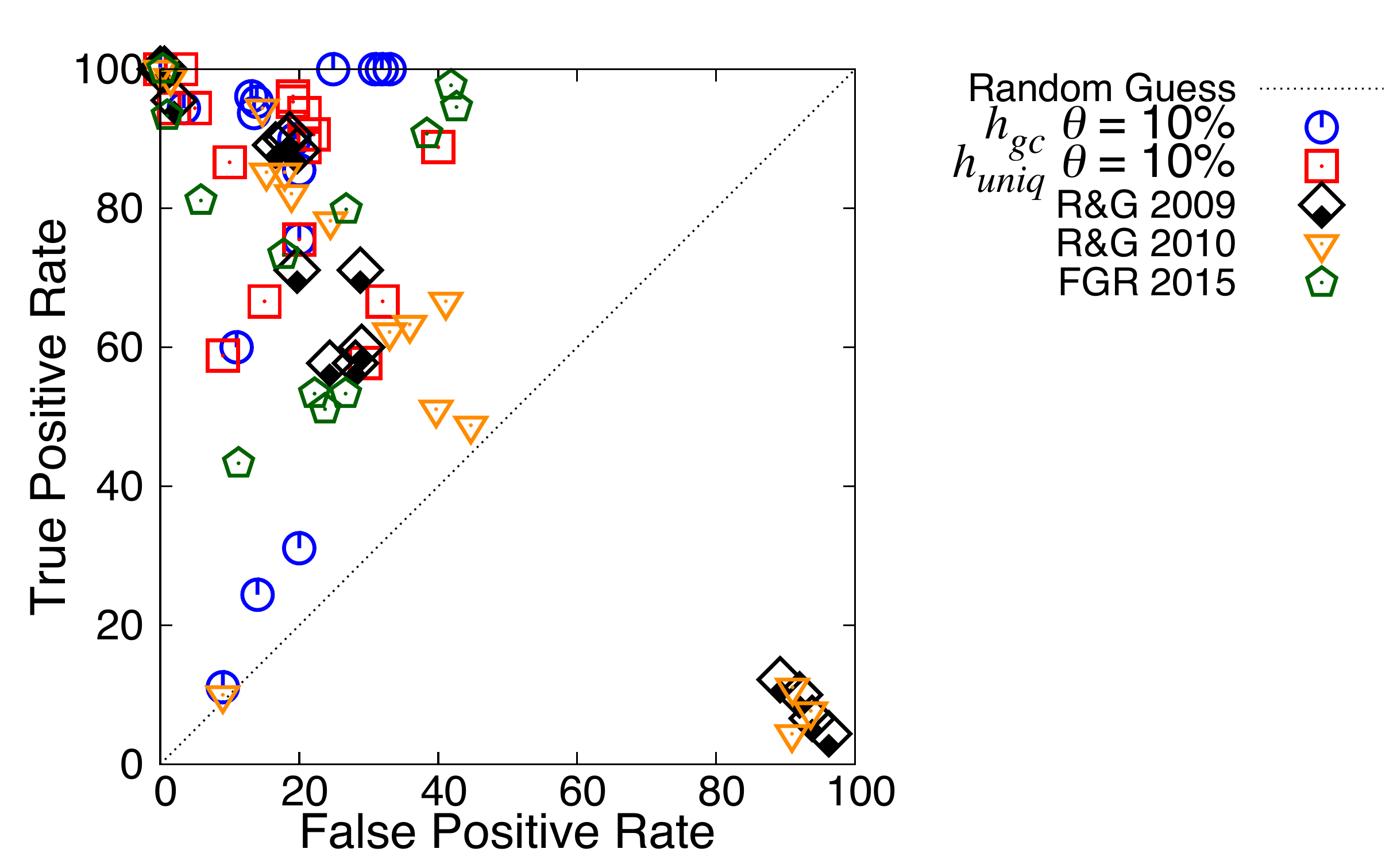}
  \caption{ROC space for all domains with missing, noisy, and full observations for our landmark-based heuristics ($\mathit{h_{gc}}$ and $\mathit{h_{uniq}}$) against R\&G 2009~\cite{RamirezG_IJCAI2009}, R\&G 2010~\cite{RamirezG_AAAI2010}, and FGR 2015~\cite{NASA_GoalRecognition_IJCAI2015}.}
  \label{fig:rocspace_alldomains-noisy}
\end{figure}

\begin{figure}[ht!]
  \centering
  \includegraphics[width=0.72\linewidth]{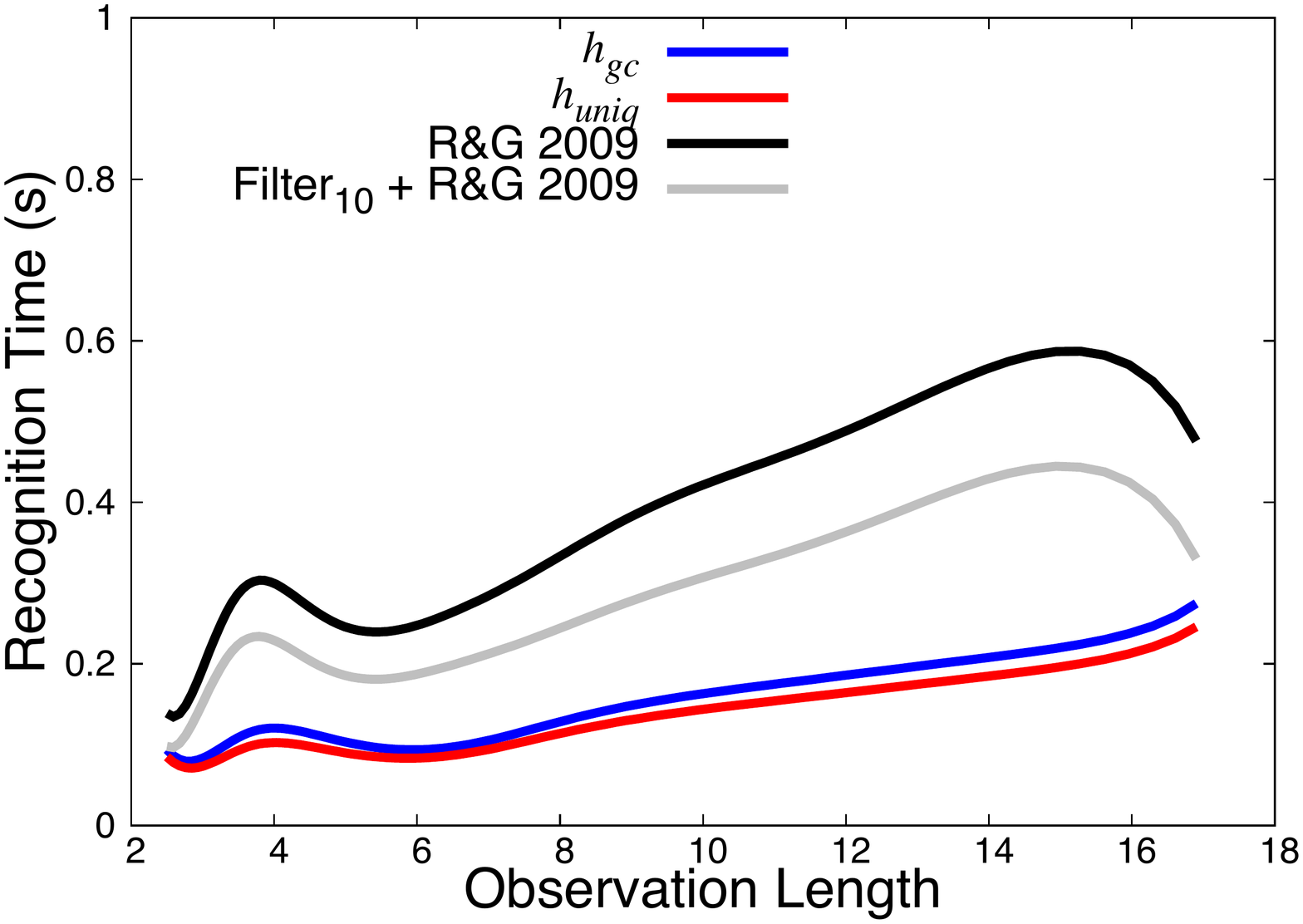}
  \caption{Recognition time comparison for missing, noisy, and full observations for our landmark-based heuristics ($\mathit{h_{gc}}$ and $\mathit{h_{uniq}}$) against R\&G 2009~\cite{RamirezG_IJCAI2009}, and R\&G 2009 using our filtering method with 10\% of threshold.}
  \label{fig:recognition_time-noisy_1}
\end{figure}

\begin{figure}[ht!]
  \centering
  \includegraphics[width=0.72\linewidth]{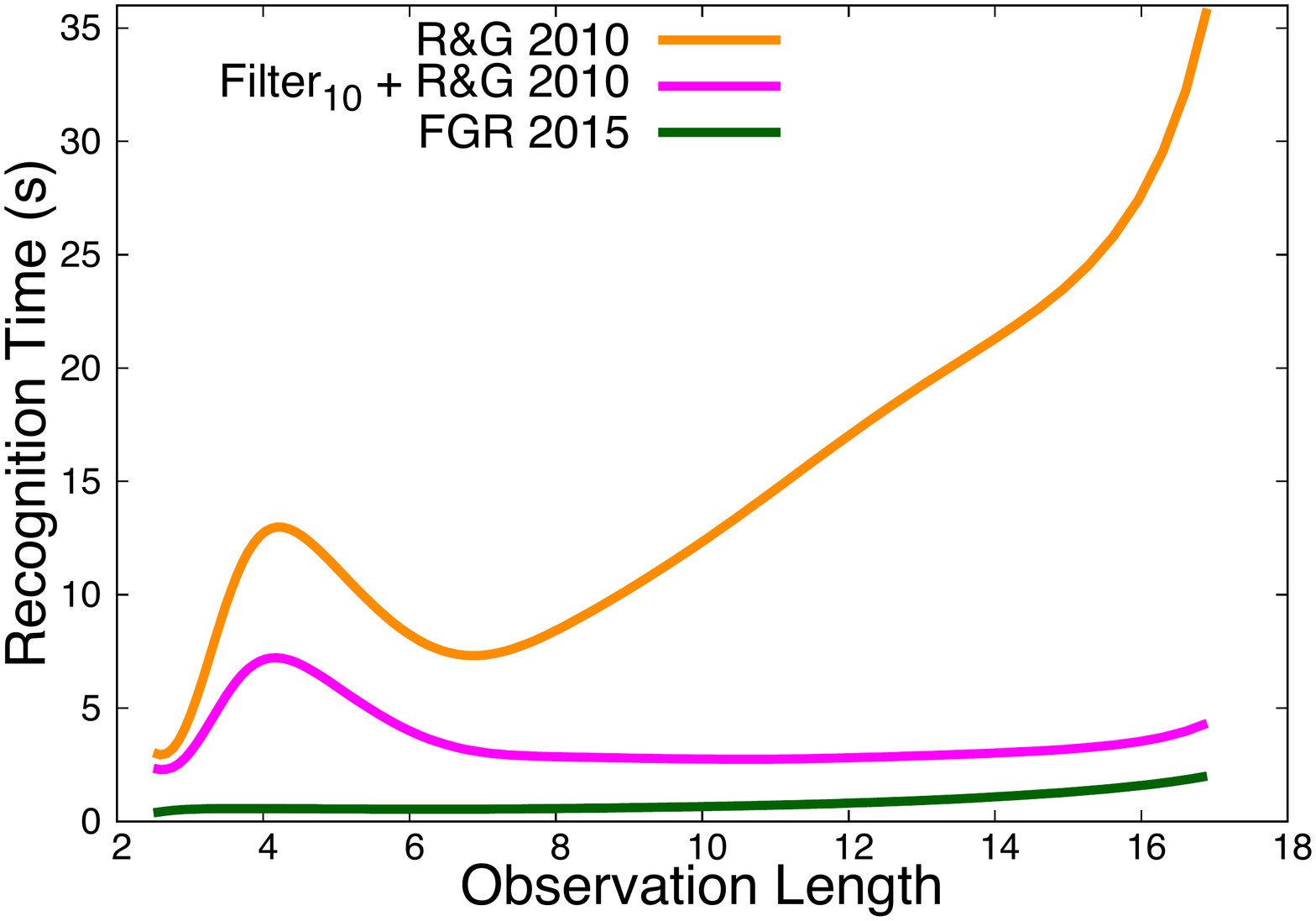}
  \caption{Recognition time comparison for missing, noisy, and full observations for R\&G 2010 using \textsc{Fast-Downward} with LM-Cut heuristic~\cite{RamirezG_AAAI2010}, R\&G 2010 using our filtering method with 10\% of threshold, and FGR 2015~\cite{NASA_GoalRecognition_IJCAI2015}.}
  \label{fig:recognition_time-noisy_2}
\end{figure}

\afterpage{
\begin{landscape}
\begin{table*}[]
\centering
\fontsize{7}{10}\selectfont
\setlength\tabcolsep{4pt}
\begin{tabular}{|c|c|ccc|ccc|ccc|ccc|}
\hline
                   &                    & \multicolumn{3}{c|}{} & \multicolumn{3}{c|}{\scriptsize $\mathit{h_{gc}}$}                                   & \multicolumn{3}{c|}{\scriptsize $\mathit{h_{uniq}}$}                                 & \multicolumn{3}{c|}{\scriptsize R\&G 2009 / Filter$_{10\%}$ + R\&G 2009}    \\ \hline
\#                 & $|\mathcal{G}|$                  & \% \textbf{Obs}   &  $|N|$   & $|O + N|$        & \textbf{\begin{tabular}[c]{@{}c@{}}Time\\$\theta$ (0 / 10)\end{tabular}}                  & \textbf{\begin{tabular}[c]{@{}c@{}}Accuracy\\$\theta$ (0 / 10)\end{tabular}}             & \textbf{\begin{tabular}[c]{@{}c@{}}Spread in $\mathcal{G}$\\$\theta$ (0 / 10)\end{tabular}}  & \textbf{\begin{tabular}[c]{@{}c@{}}Time\\$\theta$ (0 / 10)\end{tabular}}                  & \textbf{\begin{tabular}[c]{@{}c@{}}Accuracy\\$\theta$ (0 / 10)\end{tabular}}             & \textbf{\begin{tabular}[c]{@{}c@{}}Spread in $\mathcal{G}$\\$\theta$ (0 / 10)\end{tabular}}  & \textbf{Time}          & \textbf{Accuracy}      & \textbf{Spread in $\mathcal{G}$ }\\ \hline
\multirow{4}{*}{\rotatebox[origin=c]{90}{\textsc{Campus}} \rotatebox[origin=c]{90}{(516)}} & \multirow{4}{*}{2}
										& 25         
											& 2 & 3.1
												& 0.031 / 0.034 & 68.2\% / 89.9\% & 1.0 / 1.28
												& 0.030 / 0.032 & 82.1\% / 90.6\% & 1.13 / 1.44 
												& 0.073 / 0.060  & 88.3\% / 78.2\%  & 1.27 / 1.13      \\
                   &                    & 50         
				   							& 2 & 4.5
				   								& 0.033 / 0.035 & 75.9\% / 93.7\% & 1.0 / 1.20
												& 0.031 / 0.032 & 78.2\% / 93.7\% & 1.02 / 1.43 
												& 0.076 / 0.068 & 89.9\% / 82.1\% & 1.26 / 1.09     \\
                   &                    & 75         
				   							& 2 & 6.4
				   								& 0.035 / 0.039 & 73.6\% / 96.2\% & 1.0 / 1.22
												& 0.034 / 0.036 & 73.6\% / 96.1\% & 1.0 / 1.42
												& 0.079 / 0.071  & 90.6\% / 85.2\% & 1.27 / 1.10     \\
                   &                    & 100        
				   							& 2 & 7.5
				   								& 0.038 / 0.041 & 72.1\% / 95.3\% & 1.0 / 1.23
												& 0.037 / 0.039 & 72.1\% / 95.3\% & 1.0 / 1.41
												& 0.084 / 0.080 & 89.1\% / 85.2\% & 1.22 / 1.06     \\ \hline
\multirow{4}{*}{\rotatebox[origin=c]{90}{\textsc{Intrusion}} \rotatebox[origin=c]{90}{(300)}} & \multirow{4}{*}{16.6} 
										& 25         
											& 2 & 3.6
												& 0.125 / 0.127 & 33.3\% / 68.8\% & 1.11 / 4.43
												& 0.102 / 0.032 & 30.0\% / 58.8\% & 1.11 / 3.94 
												& 0.537 / 0.456 & 71.1\% / 63.3\% & 2.65 / 2.34     \\
                   &                    & 50         
				   							& 2 & 6.7
				   								& 0.134 / 0.135 & 83.3\% / 93.3\% & 1.06 / 2.04
												& 0.116 / 0.118 & 64.4\% / 88.8\% & 1.03 / 2.68 
												& 0.649 / 0.483 & 95.5\% / 94.4\% & 1.28 / 1.27     \\
                   &                    & 75         
				   							& 2 & 10.2
				   								& 0.146 / 0.150 & 94.4\% / 98.8\% & 1.01 / 1.33
												& 0.124 / 0.130 & 87.7\% / 94.4\% & 1.03 / 1.82
												& 0.712 / 0.524 & 100\% / 98.8\% & 1.01 / 1.01     \\
                   &                    & 100        
				   							& 2 & 15.1
				   								& 0.155 / 0.159 & 100\% / 100\% & 1.0 / 1.10
												& 0.136 / 0.138 & 100\% / 100\% & 1.0 / 1.63
												& 0.805 / 0.659 & 100\% / 100.0\% & 1.0 / 1.0     \\ \hline
\multirow{4}{*}{\rotatebox[origin=c]{90}{\textsc{IPC-Grid}} \rotatebox[origin=c]{90}{(300)}} & \multirow{4}{*}{8.3} 
										& 25         
											& 2 & 4.1
												& 0.253 / 0.260 & 58.8\% / 75.5\% & 1.76 / 2.95
												& 0.208 / 0.211 & 53.3\% / 75.5\% & 1.72 / 2.83 
												& 0.462 / 0.301 & 12.2\% / 11.1\% & 7.55 / 2.81     \\
                   &                    & 50         
				   							& 2 & 7.6
				   								& 0.261 / 0.267 & 85.5\% / 85.5\% & 1.33 / 1.71
												& 0.212 / 0.220 & 83.3\% / 86.6\% & 1.33 / 1.71 
												& 0.469 / 0.312 & 4.4\% / 4.4\% & 8.06 / 1.61     \\
                   &                    & 75         
				   							& 2 & 11.5
				   								& 0.269 / 0.272 & 94.4\% / 94.4\% & 1.08 / 1.23
												& 0.224 / 0.233 & 94.4\% / 94.4\% & 1.08 / 1.15
												& 0.475 / 0.323 & 6.6\% / 7.7\% & 7.88 / 1.10     \\
                   &                    & 100        
				   							& 2 & 16.9
				   								& 0.275 / 0.288 & 100\% / 100\% & 1.0 / 1.0
												& 0.239 / 0.246 & 100\% / 100\% & 1.0 / 1.0
												& 0.476 / 0.330 & 10.0\% / 10.0\% & 7.76 / 1.0    \\ \hline
\multirow{4}{*}{\rotatebox[origin=c]{90}{\textsc{Kitchen}} \rotatebox[origin=c]{90}{(150)}} & \multirow{4}{*}{3} 
										& 25         
											& 2 & 2.5
												& 0.094 / 0.097 & 11.1\% / 11.1\% & 0.22 / 0.22
												& 0.081 / 0.083 & 88.8\% / 88.8\% & 2.55 / 2.55 
												& 0.139 / 0.098 & 71.1\% / 62.2\% & 1.57 / 1.46     \\
                   &                    & 50         
				   							& 2 & 4.8
				   								& 0.095 / 0.099 & 28.8\% / 31.1\% & 0.64 / 0.66
												& 0.084 / 0.088 & 64.4\% / 66.6\% & 1.71 / 1.73 
												& 0.135 / 0.102 & 57.7\% / 51.1\% & 1.42 / 1.20     \\
                   &                    & 75         
				   							& 2 & 7.3
				   								& 0.097 / 0.101 & 24.4\% / 24.4\% & 0.66 / 0.66
												& 0.090 / 0.092 & 57.7\% / 57.7\% & 1.66 / 1.66
												& 0.138 / 0.103 & 57.7\% / 48.8\% & 1.31 / 1.24     \\
                   &                    & 100        
				   							& 2 & 11
				   								& 0.104 / 0.105 & 60.0\% / 60.0\% & 0.93 / 0.93
												& 0.093 / 0.095 & 66.6\% / 66.6\% & 1.13 / 1.13
												& 0.144 / 0.109 & 60.0\% / 66.6\% & 1.46 / 1.13    \\ \hline
\end{tabular}
\caption{Experiments and evaluation with missing, noisy, and full observations for $\mathit{h_{gc}}$, $\mathit{h_{uniq}}$, R\&G 2009, and our filtering method (10\% of threshold) with R\&G 2009 (Part 1).}
\label{tab:goalRecognitionResultsWithNoisy1}
\end{table*}
\end{landscape}
}

\afterpage{
\begin{landscape}
\begin{table*}[]
\centering
\fontsize{7}{10}\selectfont
\setlength\tabcolsep{4pt}
\begin{tabular}{|c|c|ccc|ccc|ccc|ccc|}
\hline
                   &                    
				   & \multicolumn{3}{c|}{} & \multicolumn{3}{c|}{\begin{tabular}[c]{@{}c@{}}\scriptsize R\&G 2010 / Filter$_{10\%}$ + R\&G 2010 \\ \scriptsize (\textsc{Fast-Downward} with LM-Cut heuristic)\end{tabular}}

				   & \multicolumn{3}{c|}{\begin{tabular}[c]{@{}c@{}}\scriptsize FGR 2015\end{tabular}}
				   
				   & \multicolumn{3}{c|}{\begin{tabular}[c]{@{}c@{}}\scriptsize IBM 2016 / Filter$_{10\%}$ + IBM 2016 \\ \scriptsize (TK$^*$ with LM-Cut heuristic, top-1000)\end{tabular}}    \\ \hline
\#                 & $|\mathcal{G}|$                  & \% \textbf{Obs}   &  $|N|$   & $|O + N|$        

& \textbf{\begin{tabular}[c]{@{}c@{}}Time\\$\theta$ (0 / 10)\end{tabular}}                  
& \textbf{\begin{tabular}[c]{@{}c@{}}Accuracy\\$\theta$ (0 / 10)\end{tabular}}             
& \textbf{\begin{tabular}[c]{@{}c@{}}Spread in $\mathcal{G}$\\$\theta$ (0 / 10)\end{tabular}}  

& \textbf{Time}          
& \textbf{Accuracy}      
& \textbf{Spread in $\mathcal{G}$ }

& \textbf{Time}          
& \textbf{Accuracy}      
& \textbf{Spread in $\mathcal{G}$ }\\ \hline
\multirow{4}{*}{\rotatebox[origin=c]{90}{\textsc{Campus}} \rotatebox[origin=c]{90}{(516)}} & \multirow{4}{*}{2}
										& 25         
											& 2 & 3.1
												& 1.958 / 1.862 & 88.3\% / 88.3\% & 1.24 / 1.24

												& 0.713 & 79.8\% & 1.33
											
												& \timeout / \timeout & - / - & - / -     \\
                   &                    & 50         
				   							& 2 & 4.5
				   								& 2.255 / 2.220 & 94.5\% / 94.5\% & 1.13 / 1.13

												& 0.666 & 90.6\% & 1.67
											
												& \timeout / \timeout & - / - & - / -     \\
                   &                    & 75         
				   							& 2 & 6.4
				   								& 2.784 / 2.731 & 99.2\% / 99.2\% & 1.11 / 1.11

												& 0.655 & 94.6\% & 1.79
											
												& \timeout / \timeout & - / - & - / -     \\
                   &                    & 100        
				   							& 2 & 7.5
				   								& 2.808 / 2.790 & 99.2\% / 99.2\%  & 1.10 / 1.10

												& 0.644 & 97.7\% & 1.81
												
												& \timeout / \timeout & - / - & - / -     \\ \hline
\multirow{4}{*}{\rotatebox[origin=c]{90}{\textsc{Intrusion}} \rotatebox[origin=c]{90}{(300)}} & \multirow{4}{*}{16.6} 
										& 25         
											& 2 & 3.6
												& 6.216 / 2.750 & 35.5\% / 38.8\% & 0.78 / 0.81
												
												& 0.494 & 43.3\% & 2.31
												
												& \timeout / \timeout & - / - & - / -     \\
                   &                    & 50         
				   							& 2 & 6.7
				   								& 6.792 / 1.881 & 74.4\% / 78.8\% & 1.10 / 0.93
												
												& 0.511 & 81.1\% & 1.78											
												
												& \timeout / \timeout & - / - & - / -     \\
                   &                    & 75         
				   							& 2 & 10.2
				   								& 8.081 / 1.686 & 91.1\% / 93.3\% & 0.94 / 0.94
												
												& 0.654 & 93.3\% & 1.10
												
												& \timeout / \timeout & - / - & - / -     \\
                   &                    & 100        
				   							& 2 & 15.1
				   								& 8.753 / 1.670 & 100\% / 100\% & 1.0 / 1.10
												
												& 0.885 & 100.0\% & 1.06
												
												& \timeout / \timeout & - / - & - / -     \\ \hline
\multirow{4}{*}{\rotatebox[origin=c]{90}{\textsc{IPC-Grid}} \rotatebox[origin=c]{90}{(300)}} & \multirow{4}{*}{8.3} 
										& 25         
											& 2 & 4.1
												& 36.275 / 20.612 & 8.8\% / 10.0\% & 0.91 / 1.0
												
												& \timeout & -  & -												
												
												& \timeout / \timeout & - / - & - / -     \\
                   &                    & 50         
				   							& 2 & 7.6
				   								& 16.310 / 4.304 & 3.3\% / 3.3\% & 0.98 / 0.98
												
												& \timeout & -  & -
												
												& \timeout / \timeout & - / - & - / -     \\
                   &                    & 75         
				   							& 2 & 11.5
				   								& 33.358 / 2.737 & 7.7\% / 7.7\% & 1.03 / 0.92
												
												& \timeout & -  & -												
												
												& \timeout / \timeout & - / - & - / -     \\
                   &                    & 100        
				   							& 2 & 16.9
				   								& 35.850 / 4.328 & 10.0\% / 10.0\% & 1.0 / 1.0
												
												& \timeout & -  & -
												
												& \timeout / \timeout & - / - & - / -     \\ \hline
\multirow{4}{*}{\rotatebox[origin=c]{90}{\textsc{Kitchen}} \rotatebox[origin=c]{90}{(150)}} & \multirow{4}{*}{3} 
										& 25         
											& 2 & 2.5
												& 3.038 / 2.343 & 53.3\% / 57.4\% & 1.35 / 0.70
												
												& 0.381 & 53.3\% & 1.33
												
												& \timeout / \timeout & - / - & - / -     \\
                   &                    & 50         
				   							& 2 & 4.8
				   								& 13.291 / 5.009 & 48.8\% / 44.4\% & 1.17 / 0.48

												& 0.410 & 51.1\% & 1.22

												& \timeout / \timeout & - / - & - / -     \\
                   &                    & 75         
				   							& 2 & 7.3
				   								& 6.467 / 2.756 & 51.1\% / 44.4\% & 1.22 / 0.53
												
												& 0.426 & 53.3\% & 1.20
												
												& \timeout / \timeout & - / - & - / -     \\
                   &                    & 100        
				   							& 2 & 11
				   								& 5.289 / 1.818 & 73.3\% / 66.6\% & 1.4 / 0.66
												
												& 0.538 & 73.3\% & 1.26
												
												& \timeout / \timeout & - / - & - / -     \\ \hline
\end{tabular}
\caption{Experiments and evaluation with missing, noisy, and full observations for R\&G 2010 using \textsc{Fast-Downward} with LM-Cut heuristic, FGR 2015~\cite{NASA_GoalRecognition_IJCAI2015}, and IBM 2016 using TK$^*$ with LM-Cut heuristic, top-1000 (Part 2).}
\label{tab:goalRecognitionResultsWithNoisy2}
\end{table*}
\end{landscape}
}

\newpage
\section{Related Work}\label{section:RelatedWork}

In this section, we compare our work to some of the most relevant recent work on goal and plan recognition in recent past years. 
We highlight differences and similarities between our goal recognition approaches and the surveyed related work. 

\citeauthor{HongGoalRecognition_2001}~(\citeyear{HongGoalRecognition_2001}) developed one of the first goal recognition approach that extends the concept of planning graph (which they call a goal graph), developing a similar structure that represents every possible path (\exemp, state transitions that connect facts and actions) from an initial state to a goal state. 
Pattison and Long~(\citeyear{PattisonGoalRecognition_2010}) propose AUTOGRAPH (AUTOmatic Goal Recognition with A Planning Heuristic), a probabilistic heuristic-based goal recognition over planning domains. 
AUTOGRAPH uses heuristic estimation and domain analysis to determine which goals an agent is pursuing. 
Ram\'{i}rez and Geffner~(\citeyear{RamirezG_IJCAI2009}) developed planning approaches for plan recognition, and instead of using plan libraries, they model the problem as a planning domain theory with respect to a known set of goals. 
Their work uses a heuristic, an optimal and modified sub-optimal planner to determine the distance to every goal in a set of goals after an observation. 
We compare their most accurate approach directly with ours. 
Follow-up work~(\citeyear{RamirezG_AAAI2010}) extended the idea of plan recognition as planning into a probabilistic approach using off-the-shelf planners that provide a posterior probability distribution over goals, given an observation sequence as evidence. 
E-Mart\'{i}n~\etal~(\citeyear{NASA_GoalRecognition_IJCAI2015}) developed a planning-based goal recognition approach that propagates cost and interaction information in a planning graph, and uses this information to estimate goal probabilities over the set of candidate goals. 
Sohrabi~\etal~(\citeyear{Sohrabi_IJCAI2016}) developed a probabilistic plan recognition approach that deals with unreliable observations (\idest, noisy or missing observations), and recognizes both goals and plans. 
Unlike these last three approaches, which provide a probabilistic interpretation of the recognition problem, we do not deal with probabilities. 
Nevertheless, our heuristic computation is a good proxy for the posterior probability distribution of the goals, given the observations, and thus could be extended to provide a probabilistic interpretation as we intend to do in future work. 
In~\cite{Mor_ACS_16}, \citeauthor{Mor_ACS_16} introduce the concept of mirroring to develop an online goal recognition approach for continuous domains. 
\citeauthor{Masters_IJCAI2017}~(\citeyear{Masters_IJCAI2017}) propose a fast and accurate goal recognition approach for path-planning, providing a new probabilistic framework for goal recognition. In \cite{MOR_IJCAI2017}, Vered and Kaminka develop a heuristic approach for online goal recognition that deals with continuous domains.
\citeauthor{MorEtAl_AAMAS18}~(\citeyear{MorEtAl_AAMAS18}) propose an online goal recognition approach that combines the use of landmarks and goal mirroring, showing this combination can improve not only the recognition time, but also the accuracy for recognizing goals in the online fashion.
Most recently, \citeauthor{GalMor_AAAI2018}~(\citeyear{GalMor_AAAI2018}) develop a novel plan recognition approach that deals with both continuous and discrete domains.

Secondly, there has been substantial recent work on goal and plan recognition \textit{design}, that is optimize the domain design so that goal and plan recognition algorithms can provide inferences with as few observations as possible.  Keren~\etal~(\citeyear{GoalRecognitionDesign_Keren2014,GoalRecognitionDesign_Keren2015,GoalRecognitionDesign_Keren2016}) develop an alternate view of the goal recognition problem, and rather than developing new goal recognition algorithms, they develop a novel approach that modifies the domain description in order to facilitate the goal recognition process. 
Their work could potentially be used alongside our techniques, and the relation between worst case distinctiveness (their measure of how difficult can it be to disambiguate goals) and the information gain from unique landmarks would provide an interesting avenue for further investigation. 

Finally, \cite{freedman2018towards} recently proposed an approach to perform probabilistic plan recognition along the lines of \cite{RamirezG_AAAI2010}, that, instead of running a full-fledged planner for each goal, takes advantage of multiple goal heuristic search~\cite{davidov2006multiple} to search for for all goals simultaneously and avoid repeatedly expanding the same nodes. 
Their approach has not been implemented and evaluated yet and it aims to overcome the limitation of our technique to only be able to account for progress towards goals when we have evidence of landmarks being achieved, while retaining the speed gains we achieve. 
While we do not have empirical evidence about its accuracy and efficiency, we believe this is an exciting direction for goal recognition, and we expect it to approach and overcome the accuracy of \cite{RamirezG_AAAI2010}.


\newpage
\section{Conclusions}\label{section:Conclusions}

We have developed novel goal recognition approaches based on planning techniques that rely on landmarks. 
Landmarks provide key information about what cannot be avoided to achieve a goal, and we have shown that they can be used efficiently, with simple heuristics, to recognize goals from missing and noisy observations. 
Our goal completion heuristic $\mathit{h_{gc}}$ computes the ratio between achieved landmarks and the total number of landmarks for a particular goal, whereas our uniqueness heuristic $\mathit{h_{uniq}}$, uses a \textit{landmark uniqueness value} to represent how informative a landmark is among the known landmarks for all candidate goals. 
These landmark-based heuristics show that it is possible to recognize goals quickly with high accuracy as well as to use them as a filtering  mechanism to refine existing planning-based goal and plan recognition approaches~\cite{RamirezG_IJCAI2009,RamirezG_AAAI2010,NASA_GoalRecognition_IJCAI2015,Sohrabi_IJCAI2016}, such that they can also be made substantially more efficient.

We have proved that our heuristic approaches are sound both as a filtering mechanism and as a goal recognition algorithm on its own, thus showing that, under certain conditions, we are guaranteed to find the correct hidden goal. 
Our experiments show that our goal recognition approaches yield not only superior accuracy results but also substantially faster recognition time for all fifteen planning domains used in evaluating against the state-of-the-art~\cite{RamirezG_IJCAI2009,RamirezG_AAAI2010,NASA_GoalRecognition_IJCAI2015,Sohrabi_IJCAI2016} at varying observation completeness levels, for both missing and noisy observations. 
The main limitation of our approaches lie in conditions of very low observability of 30\% or less. 
Specifically, for problems with very short plans, and thus, where the number of actually observed action consists of one or two actions, the odds of observing one of the problem's landmarks are very low, jeopardizing recognition accuracy. 
Under these conditions, our filtering mechanism still provides a major improvement on the runtime (and often accuracy) of existing goal recognition approaches. 

As future work, we intend to explore multiple avenues to improve our goal recognition approaches. 
First, we aim to use other planning techniques, such as heuristics and symmetries in classical planning~\cite{Heuristics2015}, and traps, invariants, and dead-ends~\cite{Lipovetzky_ICAPS16}. 
Second, we intend to explore other landmark extraction algorithms to obtain additional information from planning domains~\cite{ICAPS03_DC_ZhGi,KeyderRH_ECAI10}. 
Third, we aim to evaluate our landmark-based heuristics for online goal and plan recognition, and we have started work in that direction~\cite{MorEtAl_AAMAS18}.

\newpage
\section*{Acknowledgements}
This article is a revised and extended version of two papers published at AAAI 2017~\cite{RamonNirMeneguzzi_AAAI2017} and ECAI 2016~\cite{PereiraMeneguzzi_ECAI2016}, we are thankful to the anonymous reviewers that helped improve the research in this article. 
The authors thank Shirin Sohrabi for discussing the way in which the algorithms of \cite{Sohrabi_IJCAI2016} should be configured, as well as Yolanda Escudero-Mart{\'{\i}}n for providing the original code to the approach of \cite{NASA_GoalRecognition_IJCAI2015} and engaging with us on how to run it.  
We thank Miquel Ram{\'{\i}}rez for various discussions around the contributions of this paper and Andr{\'{e}} Grahl Pereira for a discussion of properties of our algorithm. 
Finally, Felipe thanks CNPq for partial financial support under its PQ fellowship, grant number 305969/2016-1.

\clearpage
\bibliography{aij-landmark-recognition}
\bibliographystyle{theapa}
\end{document}